\documentclass{amsart}

\usepackage{mathtools}
\usepackage{mathabx}
\usepackage{booktabs}
\usepackage[T1]{fontenc}
\usepackage{amsmath,amsfonts,mathrsfs,amssymb,amsthm}
\usepackage{enumerate}
\usepackage{bm}
\usepackage{dsfont}
\usepackage{stmaryrd}
\usepackage{wrapfig}
\usepackage{tikz}
\usepackage{tikz-3dplot}
\definecolor{myblue}{rgb}{0.21, 0.34, 0.74}
\definecolor{mygrey}{rgb}{0.55, 0.57, 0.67}
\definecolor{myred}{rgb}{0.79, 0.0, 0.09}
\definecolor{mygreen}{rgb}{0.05, 0.5, 0.06}

\usepackage[colorlinks=true, citecolor=myblue, linkcolor=black]{hyperref}

\DeclareMathAlphabet{\mathscrbf}{OMS}{mdugm}{b}{n}
\usepackage[scr=boondoxo]{mathalpha}
\usepackage{cleveref}
\usepackage{subcaption}
\DeclareCaptionLabelFormat{empty}{}
\usepackage{comment}

\usepackage{wrapfig}
\usepackage{upgreek}

\renewcommand{\S}{\textsection}

\newcommand{\baV}{\overline V}

\newcommand{\deV}{\cY}
\newcommand{\mbV}
{\cX}
\newcommand{\mbW}{\mathcal{X}_0}
\newcommand{\de}{\mathrm{d}}
\newcommand{\epa}{\varepsilon_{\phi}}
\newcommand{\deB}{\widetilde B}

\newcommand{\Id}{\mathrm{Id}}
\newcommand{\Cx}{\mathrm{Conv}}
\newcommand{\supp}{\mathrm{supp}}

\newcommand{\J}
{\phi}
\newcommand{\proj}{\mathbf{P}}

\DeclareMathOperator{\dive}{\mathring{\mathrm{d{\imath}v}}}
\DeclareMathOperator{\grade}{\mathring{\nabla}}
\renewcommand{\S}{\mathbb{S}^{d-1}}

\newcommand{\gradW}{\nabla\mkern-10mu\nabla}

\newcommand{\cX}{\mathcal{X}}
\newcommand{\cY}{\mathcal{Y}}

\newcommand{\sE}{\mathsf{E}}

\newcommand*\diff{\mathop{}\!\mathrm{d}}
\newcommand{\R}{{\rm I}\kern-0.18em{\rm R}}
\newcommand{\h}{{\rm I}\kern-0.18em{\rm H}}
\newcommand{\K}{{\rm I}\kern-0.18em{\rm K}}
\newcommand{\p}{{\rm I}\kern-0.18em{\rm P}}
\newcommand{\E}{{\rm I}\kern-0.18em{\rm E}}
\newcommand{\Z}{\mathbb{Z}}
\newcommand{\1}{{\rm 1}\kern-0.24em{\rm I}}
\newcommand{\N}{{\rm I}\kern-0.18em{\rm N}}

\newcommand{\ud}{\mathrm{d}}

\newcommand{\eps}{\varepsilon}

\numberwithin{equation}{section}

\usepackage[font=small,labelfont=bf]{caption}

\theoremstyle{plain}
\newtheorem{theorem}{Theorem}[section]

\newtheorem{proposition}[theorem]{Proposition}
\newtheorem{lemma}[theorem]{Lemma}

\newtheorem{remark}[theorem]{Remark}

\newtheorem{example}[theorem]{Example}

\begin{document}

\title{Quantitative Clustering in Mean-Field Transformer Models}

\author[S. Chen]{Shi Chen}
\address{(SC) Department of Mathematics, Massachusetts Institute of Technology, 77 Massachusetts Ave, 02139 Cambridge MA, USA} 
\email{schen636@mit.edu}

\author[Z. Lin]{Zhengjiang Lin}
\address{(ZL) Department of Mathematics, Massachusetts Institute of Technology, 77 Massachusetts Ave, 02139 Cambridge MA, USA} 
\email{linzj@mit.edu}

\author[Y. Polyanskiy]{Yury Polyanskiy}
\address{(YP) Department of Electrical Engineering and Computer Science, Massachusetts Institute of Technology, 77 Massachusetts Ave, 02139 Cambridge MA, USA}
\email{yp@mit.edu}

\author[P. Rigollet]{Philippe Rigollet}
\address{(PR) Department of Mathematics, Massachusetts Institute of Technology, 77 Massachusetts Ave, 02139 Cambridge MA, USA} 
\email{rigollet@math.mit.edu}

\date{}


\begin{abstract}
The evolution of tokens through deep transformer models can be modeled as an interacting particle system that has been shown to exhibit an asymptotic clustering behavior akin to the synchronization phenomenon in Kuramoto models. In this work, we investigate the long-time clustering of mean-field transformer models. More precisely, under suitable assumptions on the transformer model parameters, we establish that any suitably regular mean-field initialization synchronizes exponentially fast to a Dirac point mass, with explicit quantitative convergence rates.
\end{abstract}
\maketitle

\setcounter{tocdepth}{2}
\makeatletter
\def\l@subsection{\@tocline{2}{0pt}{2.8pc}{5pc}{}}

\tableofcontents


\section{Introduction}

The (self-)attention mechanism, initially introduced by~\cite{BahChoBen15}, forms the foundation of the transformer architecture developed in~\cite{vaswani2017attention}. This revolutionary architecture has become fundamental for large language models (LLMs), catalyzing remarkable advances in artificial intelligence.

Recently,~\cite{geshkovski2023mathematical} proposed to study how a deep stack of attention layers processes information as a mean-field interacting particle system on the sphere $\S$ that exhibits long-time clustering properties; see also~\cite{sander2022sinkformers,geshkovski2023emergence,karagodin2024clustering,koubbi2024impact, shalova2024solutions, criscitiello2024synchronization,geshkovski2024dynamic,bruno2024emergence,abella2024asymptotic,burger2025analysis,castin2025unified}.

This model---called \emph{attention dynamics}---captures the representation of \emph{tokens} as they evolve through the successive layers of a transformer.  In particular, the clustering phenomenon put forward in~\cite{geshkovski2023emergence,geshkovski2023mathematical} is critical to understanding the structure of internal representations for these pervasive models. More specifically, in attention dynamics, $n$ tokens (particles) $x_1, \ldots, x_n \in \S$ evolve as

\begin{equation}
    \label{eq:token_dymamics}
    \dot x_i(t) = \proj_{x_i(t)}\big[\frac1n \sum_{j=1}^n x_j(t) e^{\beta \langle x_i(t), x_j(t)\rangle}\big]\ \quad t\ge 0,\ i=1, \ldots, n\,,
\end{equation}
where $\beta \geq 0$, $\proj_x[ y]:= y -\langle x,y\rangle\, x$ denotes the projection of $y \in \R^d$ onto the hyperplane $T_x\S$ tangent to the sphere at $x \in \S$. We refer to~\cite{geshkovski2023mathematical} for a derivation of this model and its relationship to the attention mechanism and layer normalization. 

Equation~\eqref{eq:token_dymamics} may be regarded as a Gibbsian
deformation of the classical Kuramoto dynamics---The Kuramoto model is widely regarded as the canonical and most extensively studied mean-field model of spontaneous synchronization in systems of coupled oscillators. Indeed, when
$d=2$ and $\beta=0$, all pairwise interactions are weighted equally, and
the system reduces to the well-known identical-frequency Kuramoto
model~\cite{kuramoto1975self,kuramoto1984chemical,acebron2005kuramoto,benedetto2014complete}.
 Varying $\beta$  changes not only the strength of the
coupling but also the geometry of the phase portrait: at high temperature
small differences in alignment are averaged out, while at low temperature
the dynamics becomes increasingly selective, favoring strongly aligned
groups and allowing more intricate energy landscapes with competing
basins of attraction.  The familiar Kuramoto synchronization mechanism is
therefore only the first, high-temperature member of a broader family of
attention-driven alignment models.

The most prominent dynamical behavior is synchronization, or clustering:
under suitable conditions the tokens collapse to a common direction,
\[
    x_i(t)\to x_\infty
    \qquad \text{as } t\to\infty ,
    \qquad i=1,\ldots,n .
\]
We use the terms \emph{synchronization} and \emph{clustering}
interchangeably for this phenomenon.

Note that the system of ODEs~\eqref{eq:token_dymamics} is of the mean-field type. Indeed, token $i$ interacts with all tokens only through their empirical distribution at time $t$. We denote this distribution by $\mu_t$ and recall that
$$
\mu_t\coloneq \frac1n \sum_{i=1}^n \delta_{x_i(t)}\,.
$$
In turn, the evolution of $\mu_t$ is governed by the continuity equation 
    \begin{align}\label{eq:mfad}
            \partial_t \mu_t + \dive\left(\mu_t  \mathcal{X}_{\mu_t,\beta}\right) = 0\,, \qquad \mathcal{X}_{\mu_t,\beta}(x) \coloneq \int_{\S} \proj_x[ y] e^{\beta \langle  x , y\rangle} \diff \mu_t(y) \,,
    \end{align}
    where here and throughout the paper $\dive=\mathrm{div}_{\S}$ denotes the divergence operator on the sphere. 
    
As pointed out in~\cite{geshkovski2023mathematical}, equation~\eqref{eq:mfad} describes a Wasserstein gradient flow that aims to maximize the functional
    \begin{equation} \label{eqn:interaction energy Lyapunov}
        \mu \mapsto \mathsf{E}_{\beta}[\mu] \coloneq \frac{1}{2 \beta}\iint e^{\beta \langle  x ,  y\rangle}  \diff \mu(x) \diff \mu(y)\,,
    \end{equation}
    where both integrals are over $\S$. In particular, we let $\mathsf{E}_{0} = \frac{1}{2 }\iint \langle  x ,  y\rangle  \diff \mu(x) \diff \mu(y)$. It is easy to see that $\mathsf{E}_{\beta}$ is maximized, for all $\beta \geq 0$, at Dirac point masses $\delta_{x_0}$ for some $x_0 \in \S$. This maximum energy state corresponds to a clustering of the tokens into a single point. Thanks to these observations, clustering of $n$ tokens hinges on three classical tools from finite dimensional dynamical systems theory: the dynamics for the $n$-tuple $(x_1(t), \dots, x_n(t)) \in (\S)^n$ can be shown to (i) converge by the \L{}ojasiewicz inequality, and (ii) avoid saddle points from almost every initialization by the center-stable manifold theorem. Moreover, all stationary points are saddle points except for the global maximizers where $x_1=\cdots=x_n$;~\cite{geshkovski2023mathematical,karagodin2024clustering,criscitiello2024synchronization, markdahl2017almost}.

In this work, we investigate clustering properties for a \emph{continuum} of tokens corresponding to $n=\infty$. The mean-field dynamics of the measure $\mu_t$ of tokens is governed by the continuity equation~\eqref{eq:mfad} but we focus on the case where it is initialized at a measure $\mu_0$ that admits a density with respect to the uniform measure on the sphere $\S$. We call\footnote{While the term ``mean-field" technically applies to the Vlasov PDE~\eqref{eq:mfad} with any initialization, including a discrete one, it is common in the literature to use this term to denote such an evolution initialized at the measure that is absolutely continuous with respect to the uniform measure. To facilitate reading, we adopt the same abuse of language and use "mean-field" to indicate such an initialization.}  this setup \emph{mean-field attention dynamics}. Despite recent efforts~\cite{castin2025unified} to study convergence of the finite-particle system as $n \to \infty$, existing results do not imply asymptotic clustering for the mean-field attention dynamics for lack of a convergence that is uniform in time. Our results overcome this limitation by developing the infinite-dimensional tools necessary to studying directly the mean-field dynamics.

More precisely, our contributions are as follows. First, we show that, echoing the finite-dimensional case, stationary points for~\eqref{eq:mfad} are all saddle points for the interaction energy $\sE$ except for global maxima given by point masses. In particular, our proof extends the approach of~\cite{criscitiello2024synchronization} by exhibiting escape directions for continuous measures. However, in the absence of a counterpart to the center-stable manifold theorem in infinite dimensions, this result is not sufficient to conclude to clustering. In fact, while infinite-dimensional versions of the \L{}ojasiewicz inequality have been developed \cite{simon1983asymptotics,colding2014lojasiewicz}, we show in \Cref{example:critical value not discrete} that such inequalities cannot hold in general at critical points of the interaction energy $\sE$.

Nevertheless, we demonstrate that a stronger version of the \L{}ojasiewicz inequality, known as the Polyak-\L{}ojasiewicz (PL) inequality, holds around point masses for measures supported on a spherical cap. From such PL inequalities, it follows readily that the Wasserstein gradient flow~\eqref{eq:mfad} converges exponentially fast to a global maximizer of $\sE$ when initialized on these measures with constrained support.

This PL inequality is employed in our main contribution, Theorem~\ref{thm:simple global convergence}, which establishes exponential rates of convergence for the mean-field attention dynamics~\eqref{eq:mfad} initialized at \emph{any} density $f_0 \in L^2(\S)$ for sufficiently small parameter $\beta<\beta_0$, where $\beta_0>0$ depends on $f_0$. Note that global convergence to point masses cannot hold at arbitrary $\beta$. Indeed, for $\beta=100$ or larger, we exhibit an equilibrium for mean-field attention dynamics that does \emph{not} correspond to a single cluster in Example~\ref{example:large epa not synchronize}. This qualitative behavior is in sharp contrast with the Kuramoto model where $d=2$ and $\beta=0$ and for which it can be proved that any regular initialization converges to to a point mass exponentially fast; see~\cite{morales2022trend}.

Our main results for mean-field attention dynamics are stated in the next section. In fact, these results are corollaries for our general results stated in Section~\ref{section:general J}. These convergence results cover more general dynamics that correspond to less simplified versions of transformer models; see~\cite{geshkovski2023mathematical}.

\section{Clustering in mean-field attention dynamics}

In this section, we present our main clustering results on mean-field attention dynamics~\eqref{eq:mfad}.

Recall from~\cite{geshkovski2023mathematical} that the mean-field attention dynamics form a reverse Wasserstein gradient flow of the interaction energy $\sE_\beta$ defined in~\eqref{eqn:interaction energy Lyapunov}: $\cX_{\mu, \beta}=\gradW \sE_\beta [\mu]$---see~\cite{chewi2024statistical,ambrosio2005gradient} for an introduction to Wasserstein gradient flows. Indeed, along~\eqref{eq:mfad} we have
    \begin{align}\label{eqn:monotone energy}
        \frac{\de}{\de t} \mathsf{E}_{\beta}[\mu_t] = \int_{\S} \left\| \mathcal{X}_{\mu_t,\beta} (x) \right\|_2 ^2 \diff \mu_t(x) \geq 0\,,
    \end{align}
with equality if and only if $\mathcal{X}_{\mu_t,\beta} (x) =0$ for $\mu_t$ almost every $x \in \S$. This equality case characterizes critical points of the energy $\sE_\beta$. The next result shows that the only critical points that are local maxima for $\sE_\beta$ are in fact single point masses. 

  \begin{proposition}\label{thm:simple local max}
     Let $d \geq 3$. For any $\beta > 0$, any local maxima of the interaction energy $\mathsf{E}_{\beta}$ is a global maxima of the form $\mu=\delta_{x_0}$ for some $x_0\in\S$.
 \end{proposition}

\Cref{thm:simple local max} follows from \Cref{thm:global_max}, which applies to more general transformer models, including those with learned parameters; see Section~\ref{section:general J}. Our proof of \Cref{thm:simple local max} rests on a detailed analysis of the second Wasserstein gradient variation of the interaction energy $\mathsf{E}_{\beta}$, as carried out in \Cref{lem:second derivative of J}. This approach is inspired by the ideas developed in \cite{geshkovski2023mathematical,criscitiello2024synchronization}, building on \cite{markdahl2017almost}, which address the case where $\mu\in\mathcal{P}(\S)$ consists of only finitely many tokens. In its simplest form, the second variation at a critical point \eqref{eqn:second variation at critical point} is not manifestly positive in every direction even when $d=2$. We therefore average over all directions and show that this average is strictly positive whenever $\mu$ is not a Dirac mass. The underlying strategy is analogous to the observation that, for a smooth function $f$, even if $\nabla^2f$ fails to be positive definite, the condition $\Delta f>0$ already precludes local maxima. This is why we require $d\geq 3$: higher-dimensional spheres furnish enough directions for the averaging argument to succeed. When $d=2$, the proof in \cite{andrew25} instead refines the strategy of \cite{geshkovski2023mathematical} and relies heavily on the specific geometry of $\mathbb{S}^1$. It is somewhat surprising that, although the $d=2$ case appears simplest at first glance, its technical demands differ markedly from those in higher dimensions. The same contrast reappears in the proof of \Cref{thm:classical Lojasiewicz}, where the $d=2$ setting admits substantial simplifications thanks to the circle’s geometry.

However, as mentioned in the introduction, \Cref{thm:simple local max} is not sufficient to establish global convergence of the mean-field attention dynamics~\eqref{eq:mfad} to a point mass because of the infinite-dimensional nature of the problem. Nevertheless, using the \L{}ojasiewicz structure theorem, one can see that critical points for $\mathsf{E}_{\beta}$ can only be supported on a finite union of submanifolds of $\S$ of dimension at most $d-2$; see for example Lemma E.5 of \cite{bruno2024emergence}. In particular, no stationary points of the mean-field attention dynamics admits a density with respect to the uniform measure other than the uniform measure itself. 
 
 Additionally, even convergence of the mean-field attention dynamics to a single limiting stationary point is unclear because of the infinite-dimensional nature of the problem. Indeed, while it is a Wasserstein gradient flow, $\sE_\beta$ lacks the Wasserstein geodesic convexity/concavity properties to ensure convergence. In finite dimensions, this limitation may be overcome using the {\L}ojasiewicz inequality whenever the objective function, say $f$ on a compact manifold is analytic. Indeed, in this case,  \cite{lojasiewicz1963propriete} proved that for any critical point $x_{\rm crit}$ of $f$, there exists a neighborhood $U$ of $x_{\rm crit}$ and constants $c_1 \in (1,\infty)$, $c_2 >0$, such that for all $x \in U$, 
 \begin{equation}
 \label{eq:lojademo}
|f(x)-f(x_{\rm crit})|\le c_2  \|\nabla f(x)\|_2 ^{c_1} \,.
 \end{equation}
 As a direct corollary, we see that the critical values of $f$ are locally discrete because if $x \in U$ and $\nabla f(x) = 0$, then \eqref{eq:lojademo} implies $f(x)=f(x_{\rm crit})$. This last observation is instrumental in establishing convergence of gradient flows of analytic functions. Unfortunately, this property does not hold in general for the energy functional $\sE_\beta$ as illustrated by the following example.

 \begin{example}[No {\L}ojasiewicz inequality for $\mathsf{E}_{\beta}$]\label{example:critical value not discrete}
   Let $d=2$ and consider the energy function $\sE_\beta$ for measures defined on the unit circle $\mathbb{S}^1 \subseteq \R^2$ identified to $\R/2\pi\Z$. Take the sequence of measures $\mu_{\varepsilon} = (1-\varepsilon)\delta_{\frac{\pi}{2}}+ \varepsilon\delta_{-\frac{\pi}{2}}, \ \eps \in (0,1)$. Observe that $\mu_{\varepsilon}$ forms a sequence of critical points for   $\mathsf{E}_{\beta}$ because $\gradW \sE_\beta [\mu_\eps](\cdot)=0$ $\mu_\eps$ almost everywhere. But  $\mathsf{E}_{\beta}[\mu_\eps ] \neq \mathsf{E}_{\beta}[\mu_0] $ and $W_2(\mu_\eps,\mu_0) \to 0 $ as $\varepsilon \to 0$,  where $W_2$ denotes the $2$-Wasserstein distance. This implies that the critical values of $\mathsf{E}_{\beta}$ are not necessarily locally discrete.  Hence, a Wasserstein version of~\eqref{eq:lojademo} cannot hold for $\mathsf{E}_{\beta}$ on $\mathcal{P}(\S)$ as argued above.
\end{example}

Example~\ref{example:critical value not discrete} reveals a striking discrepancy between the mean-field dynamics studied here and the ones for a finite number of tokens. Indeed, the map:
$$
(x_1, \ldots, x_n) \mapsto \frac1{n^2} \sum_{i,j=1}^n e^{\beta \langle x_i, x_j \rangle}
$$
is analytic on the compact manifold $(\S)^n$ so the \L{}ojasiewicz inequality holds for a finite number $n$ of tokens. This discrepancy stems from the infinite-dimensional nature of the space of probability measures $\mathcal{P}(\S)$. Consequently, to establish global convergence in $\mathcal{P}(\S)$, one must either adapt the classical form of the \L{}ojasiewicz inequality or exploit specific geometric properties of the Wasserstein gradient flow on $\S$.

The following result shows that if we rule out sequences that place mass outside of a spherical cap around $x_0$ then a strong version of the \L{}ojasiewicz inequality, called Polyak-\L{}ojasiewicz (PL) holds.

 \begin{theorem}[Polyak-{\L}ojasiewicz inequality on a spherical cap]\label{thm:classical Lojasiewicz}
    Fix $d \geq 2, \beta \geq 0, \alpha \in [0,\pi/2), u \in \S$ and let  $S_{\alpha} ^+ (u)\subseteq \S$ denote the spherical cap of angle $\alpha$ around $u$ defined by
            \begin{align}\label{eqn:simple spherical cap}
            S_{\alpha} ^+ (u) \coloneq \left\{ x \in \mathbb{S}^{d-1} \ | \ \left\langle x , u \right\rangle \geq \cos \alpha \right\}.
        \end{align}
Let $\mu$ be a probability measure supported on $S_{\alpha} ^+ (u)$. Then if $10(1+\sqrt\beta)\tan \alpha \le 1$, the following PL inequality holds
        \begin{align*}
          \mathsf{E}_{\beta}[\delta_{u}] - \mathsf{E}_{\beta}[\mu] \le 10e^{-\beta}   \int_{\S} \| \mathcal{X}_{\mu,\beta}(x)\|_2 ^2 \diff \mu(x) \,.
        \end{align*}
Also, the sequence of measures $\mu_t, t\ge 0$ initialized at $\mu_0=\mu$ supported on $S_{\alpha} ^+ (u)$ and evolving according to~\eqref{eq:mfad}  converges to a single point mass $\delta_{x_{\infty}}$ with $\langle x_{\infty}, u \rangle \ge \cos \alpha$ at an exponential rate given by
    \begin{align*}
        W_2(\mu_t,\delta_{x_{\infty}}) \leq 20e^{-\beta} e^{-\frac{e^{\beta}}{20} t}  \left( \int_{\S} \| \mathcal{X}_{\mu,\beta}(x)\|_2 ^2 \diff \mu(x) \right)^{\frac{1}{2}}\,.
    \end{align*}
\end{theorem}

Note that $\mathsf{E}_\beta[\delta_u] = \max_{\mu \in \mathcal{P}(\S)} \mathsf{E}_\beta[\mu] = e^\beta$ for any $u \in \S$. When $\mu$ is a discrete measure supported on a hemisphere of $\S$, i.e., $\supp(\mu) \subseteq S^+_{\pi/2}(u)$ for some $u \in \S$, \cite[Lemma~6.4]{geshkovski2023mathematical} establishes an exponential synchronization result whose convergence rate depends implicitly on the initial configuration of the tokens and deteriorates as the number of tokens increases. This hemisphere-support assumption is classical in the Kuramoto model ($d=2$, $\beta=0$); see, e.g., \cite{ha2010complete,choi2012asymptotic,frouvelle2019long,ha2020emergence,abdalla2022expander}. Indeed, such a geometric restriction on $\supp(\mu)$ stems primarily from the following proposition, which is a crucial ingredient in the proof of \Cref{thm:classical Lojasiewicz}.

\begin{proposition}\label{prop:hemisphere_critical_point}
Let $  \mu \in \mathcal{P}(\S)  $ be a critical point of the energy functional $  \mathsf{E}_\beta  $ such that $  \supp(\mu) \subseteq S_\alpha^+(U)  $ for some $  U \in \S  $ and $  \alpha \in (0, \pi/2)  $. Then there exists an $  x_0 \in S_\alpha^+(U)  $ such that $  \mu = \delta_{x_0}  $.
    
\end{proposition}

In \Cref{thm:classical Lojasiewicz}, the condition $10(1+\sqrt{\beta})\tan\alpha\le 1$ is imposed purely for technical reasons. Nevertheless, when $\beta$ is large, the scaling $\alpha\sim\beta^{-1/2}$ accurately captures the \emph{effective range} of the self-attention mechanism. For any two points $x,y\in\S$ separated by a geodesic angle of this order, one has $1-\langle x,y\rangle\sim\alpha^2\sim\beta^{-1}$. Geometrically, in the expression \eqref{eq:mfad} for the vector field $\mathcal{X}_{\mu_t,\beta}(x)$, the exponential weight $e^{\beta\langle x,y\rangle}$ implies that the force at $x$ is dominated by contributions from those $y$ lying within a geodesic angle of order $\alpha\sim\beta^{-1/2}$. This mechanism is further illustrated in our later \Cref{example:large epa not synchronize}. We remark that this notion of effective interaction range is widely used and discussed in the self-attention literature; see, e.g., \cite{bruno2024emergence,geshkovski2024dynamic,chen2025critical,bruno2025multiscale,koubbi2026homogenized}.

In our setting, \Cref{thm:classical Lojasiewicz} serves as the key tool for establishing global convergence to a point mass in \Cref{thm:simple global convergence}, for arbitrary regular initial measures $  \mu_0  $ not necessarily supported on a spherical cap. The main technical difficulty, analogous to that in the proof of \Cref{thm:simple local max}, is that the second variation of $  \mathsf{E}_{\beta}  $ along the Wasserstein gradient flow (see \eqref{eqn:second variation at critical point} or \eqref{eqn:derivative I_t 2}) is not manifestly strictly negative. In particular, when $\beta>0$, the first term in \eqref{eqn:second variation at critical point} or \eqref{eq:defQ} is strictly positive and thus makes the second variation even less negative. For the Kuramoto model ($d=2$ and $\beta = 0$), a much simpler proof of \Cref{thm:classical Lojasiewicz} is possible, since on $  \mathbb{S}^1 \subseteq \R^2 $, the angle between the tangent lines at any two points $  x,y  \in \mathbb{S}^1$ coincides exactly with the angle between $  x  $ and $  y  $ themselves. This property greatly simplifies the expression of the second variation in \eqref{eq:defQ}. In contrast, when $  d\geq 3  $, the richer geometry of $  \S  $ together with the infinite-dimensional nature of $  \mathcal{P}(\S)  $ makes obtaining the desired uniform convergence rate considerably more challenging than in the finite-token case or the Kuramoto model. The detailed proof of \Cref{thm:classical Lojasiewicz}, which covers all cases $  d\geq 2  $ and $\beta \geq 0$, is presented in \Cref{sec: Lojasiewicz}.

We are now in a position to state our main result \Cref{thm:simple global convergence} for initializations that need not be supported on a spherical cap. As observed in~\cite{morales2022trend} for the Kuramoto model ($  d=2  $, $  \beta=0  $), $\mathsf{E}_{0}$ does not satisfy a PL inequality. Instead, the energy $\mathsf{E}_{0}$ satisfies a second-order differential inequality along the flow $\mu_t$ defined in \eqref{eq:mfad} with a vanishing remainder term. Building on this, we derive the following modified inequality for $\mathsf{E}_{\beta}$ for $d \geq 2$ and $\beta \geq 0$: if $\mu_0$ has a density $f_0 \in L^2(\S)$ with respect to the uniform measure,
    \begin{align}\label{eq:modified PL with error}
         \frac{\de^2}{\de t^2} \mathsf{E}_{\beta}[\mu_t] \leq - \frac{\de}{\de t} \mathsf{E}_{\beta}[\mu_t] + C_0 e^{-d C_1 t}, \quad \text{ for}\quad t> T_0
    \end{align}
where $C_0, T_0$ are constants depending on $\mu_0$, $C_1$ is a universal constant. In turn, this inequality enables us to establish the following result for any initial measure that admits a density $f_0 \in L^2(\S)$ with respect to the uniform measure.

 \begin{theorem}\label{thm:simple global convergence}
     Fix $d \geq 2$. Let $\mu_t$ evolve according to the mean-field attention dynamics~\eqref{eq:mfad} initialized at $\mu_0$ with mean such that
     \begin{align}\label{eqn:R0 def}
     R_0\coloneq \big\| \int_{\S} x \ud \mu_0(x) \big\|_2>0.
     \end{align}
     Assume that $\mu_0$ admits a density  $f_0 \in L^2(\S)$ with respect to the uniform measure, then $\mu_t$ also admits a density $f_t \in L^2(\S)$ for all $t>0$. Moreover, there exist $\beta_0, C_0, T_0 >0$, all depending on $\mu_0$, such that if $|\beta| <\beta_0$, there exists an $x_{\infty} \in \S$ for which
        \begin{align}\label{eqn:simple exponential rate}
            W_2 ( \mu_t , \delta_{x_{\infty}} ) \leq C_0 e^{-\frac{t}{100}} \,, \quad \text{ for}\quad t > T_0\,.
        \end{align}
\end{theorem}
Note that the convergence in Wasserstein distance to a point mass means that (i) the variance of $\mu_t$ converges to zero exponentially fast, and (ii) its mean also converges to $x_{\infty}$. 

\Cref{thm:simple global convergence} is a special case of the more general \Cref{thm:main thm}, which accommodates broader attention mechanisms. The proof relies on the approximation $  e^{\beta \langle x, y \rangle} \approx 1  $ when $  \beta  $ is small. Consequently, the result extends to more realistic settings in which $  e^{\beta \langle x, y \rangle}  $ is replaced by $  e^{\langle Q_t x, K_t y \rangle}  $ in the definition of the vector field $  \mathcal{X}_{\mu_t,\beta}(x)  $, provided that $  \|Q_t\|_2  $, $  \|\partial_t Q_t\|_2  $, $  \|K_t\|_2  $, and $  \|\partial_t K_t\|_2  $ remain bounded above by a sufficiently small constant uniformly in time. We refer the reader to \cite{geshkovski2023emergence,geshkovski2023mathematical} for the introduction of the time-dependent parameter matrices $  Q_t  $ and $  K_t  $ arising in real-world transformer; this extension is omitted from the present paper.

 When $\beta=0$, a qualitative mean-field convergence result was proved in~\cite{frouvelle2019long}, and under the finite particle setting and additional symmetric assumptions on the entire flow $\{f_t(x)\}_{t}$, an exponential convergence implicitly depending on the initial condition of the system was also derived when $\beta=0$. For Kuramoto models ($d=2$ and $\beta = 0$), \cite{ha2020emergence,morales2022trend} proved similar mean-field exponential convergence results for small frequency terms. Long-time behaviors of Kuramoto models have also been extensively studied by \cite{ha2010complete,ha2016collective,benedetto2014complete}.

When $\beta >0$, even in the finite-particle case, existing convergence results for general initial conditions rely on soft qualitative arguments that do not provide explicit convergence rates~\cite{markdahl2017almost,andrew25}. To the best of our knowledge,~\Cref{thm:simple global convergence} is the first result to provide quantitative rates of convergence for attention dynamics (mean-field or finite-particle) under a general initial condition like $R_0>0$.

There are also several fundamental differences between the Kuramoto models ($  \beta=0  $, $  d=2  $) and transformer models that render the proof for \Cref{thm:simple global convergence} substantially more involved. A key Lyapunov functional in the Kuramoto setting is the mean-vector length $  R_t=\bigl\|\int_{\S}x\,\mathrm{d}\mu_t(x)\bigr\|_2  $, which is monotone increasing if $  \beta=0  $. In contrast, as illustrated in \Cref{Fig: Rt not monotone}, $  R_t  $ exhibits no monotonicity in the transformer model.
Moreover, when $  \beta  $ is small, the effective-range phenomenon exploited in \Cref{thm:classical Lojasiewicz} is absent: because the support of $  \mu_t  $ is no longer localized to a hemisphere, every pair of points in $\S$ interacts globally. Indeed, in the proof of \Cref{thm:simple global convergence}, we must construct a dynamically evolving effective range and carefully pull almost all of the mass of $  \mu_t  $ into it. We then combine \Cref{thm:classical Lojasiewicz} with the higher-dimensional analogue of \eqref{eq:modified PL with error} to establish the explicit exponential convergence in \Cref{thm:simple global convergence}. This global coupling is considerably more intricate than on the circle $  \mathbb{S}^1  $, owing to the richer geometry of the higher-dimensional sphere $\S$.

The assumption that $  f_0\in L^2(\S)  $ in \Cref{thm:simple global convergence} is imposed for technical reasons. We believe that the techniques developed in this paper can be extended to yield similar mean-field synchronization results for $  f_0\in L^p(\S)  $ with $  p>1  $, as well as for densities $  f_0  $ supported on a connected smooth submanifold of $  \S  $ with sufficiently regular density. However, as illustrated in \Cref{example:delta mass not synchronize}, the long-time behavior of transformer models can differ substantially when initialized from Dirac delta masses versus mean-field densities.

\begin{example}[No synchronization for finite particles]\label{example:delta mass not synchronize}
Fix $d=2$.  We construct an example on the unit circle $\mathbb{S}^1 \subseteq \R^2$ identified to $\R/2\pi\Z$ where particles do not converge to a single cluster despite being initialized at $\mu_0$ that satisfies $R_0>0$; see Figure~\ref{fig:delta_mass_not_sync}. To that end, define
$\mu_0 = \frac{1}{50}\delta_{\frac{\pi}{2}}+ \frac{49}{100}\delta_{-\frac{\pi}{2}-\xi}+ \frac{49}{100}\delta_{-\frac{\pi}{2}+\xi}$ for an $\xi \in (0,\frac{1}{100})$.  

With this initialization, the initial velocity field $\cX_{\mu_0, \beta}$ is given for any $\theta \in [0,2\pi)$ by\footnote{This expression follows from a simple change of variables; see~\cite[Section~7.1]{geshkovski2023mathematical}}
    \begin{align}\label{eqn:simple transformer circle}
       \cX_{\mu_0, \beta}(\theta)= -\int_{0}^{2\pi}  \sin(\theta-\omega) e^{\beta \cos(\theta - \omega)} \diff \mu_0(\omega)\,.
    \end{align}

It is easy to see that $\cX_{\mu_0, \beta}(\pi/2) = 0$ by symmetry. Moreover, 
$\cX_{\mu_0, \beta}(-\frac{\pi}{2}-\xi)=-\cX_{\mu_0, \beta}(-\frac{\pi}{2}+\xi)>0$ so long as $\beta >0$, hence the two particles in the south hemisphere get closer and $\mu_t$ initialized at $\mu_0$ eventually converges to $\mu_{\infty} \coloneq \frac{1}{50}\delta_{\frac{\pi}{2}}+ \frac{49}{50}\delta_{-\frac{\pi}{2}}$ as $t\to \infty$. The point of this example is that, although $\mu_t$ initialized at $\mu_0$ does not converge to a single point mass, \Cref{thm:simple global convergence} implies that when $\beta$ is small, any initial measure with an $L^2(\S)$-density, which may be arbitrary close to $\mu_0$, contracts to a point mass at an exponential rate. In particular, we see that the contraction speed, i.e. $C_0$ and $T_0$ in \Cref{thm:simple global convergence}, must depend on $f_0$.
\end{example}
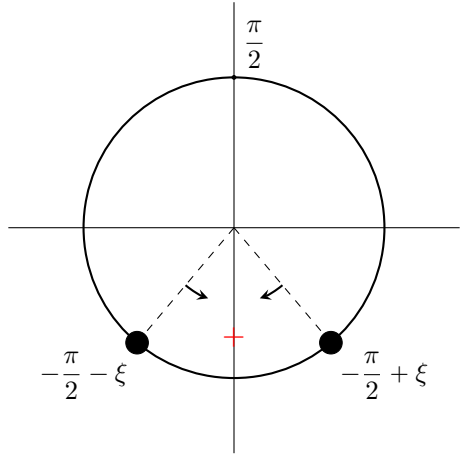
\begin{figure}
  \centering
  \begin{tikzpicture}
  \def\xii{0.7} 
  \def\r{2} 
  \def\s{.5} 
  \pgfmathsetmacro{\yplus}{1/50-49*cos(deg(\xii))/50}
  \pgfmathsetmacro{\yplusscaled}{\yplus*\r} 
  \draw[thick] (0,0) circle (\r);
  \draw[thin, -] (-\r-1,0) -- (\r+1,0) ;
  \draw[thin, -] (0,-\r-1) -- (0,\r+1) ;
  \coordinate (M1) at ({0},{{\r}});
  \coordinate (M2) at ({{\r*cos(deg(3*pi/2-\xii))}},{{\r*sin(deg(3*pi/2-\xii))}});
  \coordinate (M3) at ({{\r*cos(deg(3*pi/2+\xii))}},{{\r*sin(deg(3*pi/2+\xii))}});
  
  \draw[thin, dashed] (0,0) -- (M2);
  \draw[thin, dashed] (0,0) -- (M3);
  
  \coordinate (Mid2) at ({{\r/2*cos(deg(3*pi/2-\xii))}},{{\r/2*sin(deg(3*pi/2-\xii))}});
  \coordinate (Mid3) at ({{\r/2*cos(deg(3*pi/2+\xii))}},{{\r/2*sin(deg(3*pi/2+\xii))}});
  
  \pgfmathsetmacro{\midangleL}{(deg(3*pi/2-\xii)+270)/2}
  \pgfmathsetmacro{\midangleR}{(deg(3*pi/2+\xii)+270)/2}
  \coordinate (ArrowL) at ({{\r/2*cos(\midangleL)}},{{\r/2*sin(\midangleL)}});
  \coordinate (ArrowR) at ({{\r/2*cos(\midangleR)}},{{\r/2*sin(\midangleR)}});
  
  \draw[->, thick, >=stealth] (Mid2) arc (deg(3*pi/2-\xii):\midangleL:\r/2);
  \draw[->, thick, >=stealth] (Mid3) arc (deg(3*pi/2+\xii):\midangleR:\r/2);
  
  \filldraw (M1) circle ({0.05*\s}) node[above right ] {$\displaystyle \frac{\pi}{2}$};
  \filldraw (M2) circle ({0.3*\s}) node[below left] {$\displaystyle -\frac{\pi}{2}-\xi$};
  \filldraw (M3) circle ({0.3*\s}) node[below right] {$\displaystyle -\frac{\pi}{2}+\xi$};
  
  \node[red, scale=1.2] at (0,\yplusscaled) {$+$};
  \end{tikzpicture}
  \caption{Illustration of $\mu_0$ in \Cref{example:delta mass not synchronize} with $\xi=.7$. Circle radii are proportional to mass at each point. The cross indicates the mean of $\mu_0$ with $R_0>.7$ and the arrows indicate velocity fields for initial angles.}
  \label{fig:delta_mass_not_sync}
\end{figure}

We conclude this section by discussing an important limitation of \Cref{thm:simple global convergence}, namely that $\beta$ is required to be small enough. It turns out that this assumption is necessary, as there exists initializations $\mu_0$ for which $R_0>0$ and that admit a density $f_0 \in L^2(\S)$ for which the mean-field attention dynamics do not converge to a single point mass. We describe such an initialization on the circle in the following \Cref{example:large epa not synchronize} and \Cref{fig:large_eps_not_sync}. This example is consistent with our earlier discussion of the effective interaction range in \Cref{thm:classical Lojasiewicz}. When $  \beta  $ is relatively large, each point interacts strongly only with nearby points and effectively neglects distant ones.

\begin{example}[No mean-field synchronization for large $\beta$]\label{example:large epa not synchronize}
    Fix $d=2$.  We construct an example on the unit circle $\mathbb{S}^1 \subseteq \R^2$ identified to $\R/2\pi\Z$ where particles do not converge to a single cluster when $\beta$ is sufficiently large. To that end, let $\beta=100$, and consider the flow \eqref{eq:mfad}  started at $\mu_0$ that admits a density $f_0 \in L^2(\R/2\pi\Z)$ with respect to the Lebesgue measure.  
    
   We construct $f_0$ as follows.  Fix $\eta, \xi \in (0, \frac{1}{100})$ and let  $h_1$ be a positive, even, and smooth function supported on $[-\eta, \eta]$ such that   $h_1$ is strictly increasing on $[- \eta, 0]$. We normalize $h_1$ such that $\int h_1 =1/3$. Similarly, let  $h_2$ be a positive, even, and smooth function supported on $[-\xi,\xi]$ such that   $h_2$ is strictly increasing on $[- \xi, 0]$ and normalized as $\int h_2=2/3$. Finally, let $f_0$ be defined as $f_0(x)= h_1(x)+ h_2(\pi+x)$; see Figure~\ref{fig:large_eps_not_sync} for an illustration.

   With this initialization, akin to \Cref{example:delta mass not synchronize}, the initial velocity field $\cX_{\mu_0, 100}$ is given for any $\theta \in [0,2\pi)$ by
    \begin{align}\label{eqn:simple transformer circle b100}
       \cX_{\mu_0, 100}(\theta)= -\int_{0}^{2\pi}  \sin(\theta-\omega) e^{100 \cos(\theta - \omega)} f_0(\omega)
       \diff \omega\,.
    \end{align}
    Clearly, $\cX_{\mu_0, 100}(0)=\cX_{\mu_0, 100}(\pi) =0 $ by symmetry. Moreover, one can easily see that  $\cX_{\mu_0, 100}$ pulls points $\theta \in [\pi-\xi, \pi+\xi]\setminus \{\pi\}$ towards $\pi$ because the main contribution in $\cX_{\mu_0, 100}$ at those points comes from the integral on $[\pi-\xi , \pi+\xi]$ in~\eqref{eqn:simple transformer circle b100}. Similarly, even though $\int  h_1 =1/3< 2/3=\int h_2$ and points in $[\pi-\xi, \pi + \xi]$ are trying to pull points in $[-\eta, \eta]$ towards $\pi$, their contribution is negligible compared to the pull from antipodal points. Indeed, when $\xi \leq \eta$,
\begin{align*}
   &\cX_{\mu_0, 100}(\eta)=\\
   & = \int_{-\xi}^{\xi}  \sin(\eta-\omega) e^{-100 \cos(\eta - \omega)} h_2(\omega) \diff \omega -\int_{-\eta}^{\eta}  \sin(\eta-\omega) e^{100 \cos(\eta - \omega)} h_1(\omega) \diff \omega\\
   & \le \frac23 \sin(2\eta) e^{-100\cos(2\eta)}  - \int_{-\eta}^{\eta/2}  \sin(\eta-\omega) e^{100 \cos(\eta - \omega)} h_1(\omega) \diff \omega\\
   &  \le \frac23 \sin(2\eta) e^{-100\cos(2\eta)}  - \frac13 \sin \big( \frac{\eta}{2}\big)e^{100 \cos(2\eta)} \lesssim  -\eta \cdot 10^{38} .
\end{align*}
By symmetry $\cX_{\mu_0, 100}(-\eta)=-\cX_{\mu_0, 100}(\eta)$ and we see that the edge of the interval $[-\eta,\eta]$ gets pulled towards $0$. Since trajectories of ODEs cannot cross, all the points in $[-\eta, \eta]$ get pulled towards 0. We can similarly discuss the case when $\xi \geq \eta$. Using a bootstrap argument, it can be shown that $\mu_t$ converges to $\mu_\infty = \frac13 \delta_0 + \frac23\delta_\pi$.
\end{example}

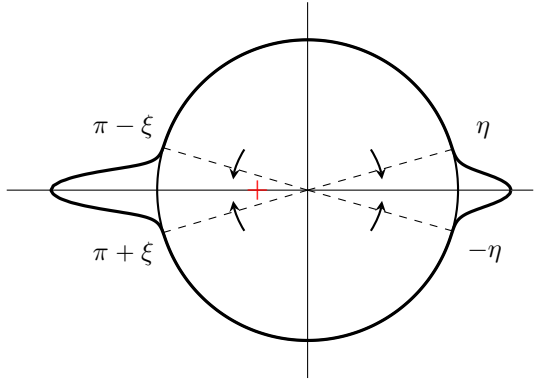
\begin{figure}
      \centering
      \begin{tikzpicture}
    \def\kappa{150} 
    \def\w{0.333} 
    \def\r{2} 
    \def\rb{\r}
    \pgfmathsetmacro{\mass}{\r*(2*\w-1)}
    \pgfmathsetmacro{\xii}{asin(0.75/\kappa)}
    \pgfmathsetmacro{\etaa}{asin(0.75/\kappa)}
    \draw[thick] (0,0) circle (\r);
    \draw[thin, -] (-\r-2,0) -- (\r+1,0) ;
    \draw[thin, -] (0,-\r-.5) -- (0,\r+.5) ;
      \draw[domain=0:2*pi, samples=1000, variable=\t, smooth, very thick, black]
        plot ({(\w*exp(\kappa*cos(deg(\t))-0.995*\kappa) + (1-\w)*exp(-\kappa*cos(deg(\t))-0.995*\kappa) + \rb)*cos(deg(\t))},
              {(\w*exp(\kappa*cos(deg(\t))-0.995*\kappa) + (1-\w)*exp(-\kappa*cos(deg(\t))-0.995*\kappa) + \rb)*sin(deg(\t))});
      \coordinate (M11) at ({{\r*cos(deg(0)-\xii)}},{{\r*sin(deg(0-\xii))}});
      \coordinate (M12) at ({{\r*cos(deg(0)+\xii)}},{{\r*sin(deg(0+\xii))}});
      \coordinate (M21) at ({{\r*cos(deg(pi-\etaa))}},{{\r*sin(deg(pi-\etaa))}});
      \coordinate (M22) at ({{\r*cos(deg(pi+\etaa))}},{{\r*sin(deg(pi+\etaa))}});
      \filldraw (M11) circle ({0}) node[below right] {$-\eta$};
      \filldraw (M12) circle ({0}) node[above right] {$\ \eta$};
      \filldraw (M21) circle ({0}) node[above left] {$\pi-\xi$};
      \filldraw (M22) circle ({0}) node[below left] {$\pi+\xi$};
      \coordinate (Mid1) at ({{\r/2*cos(deg(0-2*\xii))}},{{\r/2*sin(deg(0-2*\xii))}});
      \coordinate (Mid4) at ({{\r/2*cos(deg(0+2*\xii))}},{{\r/2*sin(deg(0+2*\xii))}});
      
      \coordinate (Mid2) at ({{\r/2*cos(deg(pi-2*\etaa))}},{{\r/2*sin(deg(pi-2*\etaa))}});
      \coordinate (Mid3) at ({{\r/2*cos(deg(pi+2*\etaa))}},{{\r/2*sin(deg(pi+2*\etaa))}});
      \pgfmathsetmacro{\midangleL}{(deg(0-\xii)+0)/2}
      \pgfmathsetmacro{\midangleR}{(deg(0+\xii)+0)/2}
      \coordinate (ArrowL) at ({{\r/2*cos(\midangleL)}},{{\r/2*sin(\midangleL)}});
      \coordinate (ArrowR) at ({{\r/2*cos(\midangleR)}},{{\r/2*sin(\midangleR)}});
      \pgfmathsetmacro{\midangleLL}{(deg(pi-\etaa)+180)/2}
      \pgfmathsetmacro{\midangleRR}{(deg(pi+\etaa)+180)/2}
      \coordinate (ArrowLL) at ({{\r/2*cos(\midangleLL)}},{{\r/2*sin(\midangleLL)}});
      \coordinate (ArrowRR) at ({{\r/2*cos(\midangleRR)}},{{\r/2*sin(\midangleRR)}});
      
      \draw[->, thick, >=stealth] (Mid1) arc (deg(0-2*\xii):\midangleL:\r/2);
      \draw[->, thick, >=stealth] (Mid4) arc (deg(0+2*\xii):\midangleR:\r/2);
      \draw[->, thick, >=stealth] (Mid2) arc (deg(pi-2*\etaa):\midangleLL:\r/2);
      \draw[->, thick, >=stealth] (Mid3) arc (deg(pi+2*\etaa):\midangleRR:\r/2);
      \draw[thin, dashed] (0,0) -- (M11);
      \draw[thin, dashed] (0,0) -- (M12);
      \draw[thin, dashed] (0,0) -- (M21);
      \draw[thin, dashed] (0,0) -- (M22);
  \node[red, scale=1.2] at (\mass,0) {$+$};
\end{tikzpicture}
      \caption{Illustration of $f_0$ in \Cref{example:large epa not synchronize}. The cross indicates the mean of $f_0$ with $R_0>0$, and the arrows indicate velocity fields at the boundaries of the support of $f_0$.}
      \label{fig:large_eps_not_sync}
    \end{figure}


\section{Clustering in general transformer models}\label{section:general J}

Despite its simplicity, the previous section shows that mean-field attention dynamics~\eqref{eq:mfad} captures the clustering phenomenon observed in practice.  In practice, the attention mechanism is parameterized by matrices that are learned from data during the training process. In general, the vector field $\cX_{\mu_t, \beta}$ in~\eqref{eq:mfad} becomes 
\begin{equation}
\label{eq:general_vector_field}
    \widetilde \cX_{\mu_t}(x)\coloneq \int_{\S} \proj_x[V_ty]\phi(\langle A_t x,  y\rangle) \diff \mu_t (y)\,,
\end{equation}
where $\{V_t,A_t\}_t$ are learned matrices and $\phi$ is a known nonlinear function; see~\cite{geshkovski2023mathematical}. Such a time inhomogeneous system is difficult to study in full generality but we make progress in this direction by considering the case where\footnote{Employing the same weights across layers has been used to reduce the complexity of transformer models~\cite{lanalbert} and it has been shows that they demonstrate better reasoning properties in certain tasks~\cite{lego}. This is the model initially studied in~\cite{sander2022sinkformers}} $V_t=V$ and $A_t=A$ for all $t$.

\subsection{Critical Points}
\label{sec:crit}

In the case where $V=A$, the vector field $\widetilde \cX_{\mu}$ is a Wasserstein gradient flow for the energy functional
    \begin{equation} \label{eqn:interaction energy}
        \mathsf{E}_{\J}[\mu] \coloneq \frac{1}{2 } \iint \J \left(\langle Ax, y\rangle \right)  \diff \mu(x) \diff \mu(y).
    \end{equation}
{Here and throughout this paper, unless explicitly stated otherwise, integrals are assumed to be taken over the set $\S$.} We leverage this property to characterize the stationary points of the mean-field attention dynamics~\eqref{eq:mfad} with the more general vector field~$\widetilde \cX_{\mu}$ in \Cref{sec:crit} under additional assumptions on the matrix $A$.

Hereafter, we assume that $\J$ is twice differentiable and that $A$ is a $d \times d$ real symmetric matrix with eigenvalues $\lambda_1\geq \lambda_2 \geq \dots \geq \lambda_d$---we allow for some eigenvalues to be negative.  The Wasserstein gradient of the interaction energy $ \mathsf{E}_{\J}$ is given by
    \begin{equation}\label{eqn:wss gradient}
        \widetilde  \cX[\mu](x) \coloneq \gradW \sE_{\J}[\mu](x) = \int_{\S} \proj_{x} [Ay]\J'\left(\langle Ax, y\rangle \right)  \diff \mu(y), \quad x \in \S\,.
    \end{equation}
Consider the general mean-field attention dynamics
\begin{equation}
    \label{eq:general_mfad}
    \partial_t \mu_t + \dive (\mu_t\widetilde  \cX[\mu_t]) =0\,,
\end{equation}
and observe that they collapse to~\eqref{eq:mfad} when $A=I_d$ and $\phi(z)=e^{\beta z}$ up to a time speed up.

The following theorem provides a partial resolution of Conjecture 2 in~\cite{karagodin2024clustering} when adapted to non-causal attention dynamics.

\begin{theorem}\label{thm:global_max}
Fix $d\ge 3$. Assume that the top three eigenvalues of $A$ satisfy $\lambda_1 = \lambda_2 = \lambda_3=\lambda > 0$, and $|\lambda_d|\leq \lambda $. Assume further that $\phi$ is twice differentiable, increasing and convex: $\J'>0$, $\J '' \geq 0$. Then any local maxima of $\mathsf{E}_{\J}[\mu]$ must be a global maximum, that is, a point mass $\delta_{x_0}$ for some $x_0\in\S$ such that $A x_0 = \lambda x_0$.
\end{theorem}
In the rest of \Cref{sec:crit}, we prove \Cref{thm:global_max}. It relies on the first and second variation formulas for $\sE_\phi [\mu]$, which are of independent interest in the study of transformer models. 
We also note that if $\lambda_d$ is significantly smaller than $-\lambda$, then a global maximizer of $\mathsf{E}_{\J}[\cdot]$ need not be a Dirac measure, as shown by a counterexample in Remark 3.5 of \cite{burger2025analysis}.

\subsubsection{First and Second Variation Formulas for $\sE_\phi$}

Let $\mathcal{P} (\S)$ denote the space of probability measures on $\S$ and let $\{\mbV_t(x)\,,\ t \geq 0,\ x \in \S\}$ be a family of vector fields on $\S$, continuously differentiable in $(t, x)$ and such that $\mbV_t(x) \in T_x\S$ for all $(t,x) \in \R_{\ge 0} \times \S$. Note that $\{\partial_t \mbV_t(x)\,,\ t \geq 0,\ x \in \S\}$ is also a family of continuous vector fields in the tangent bundle of $\S$. Let $\{\mu_t\}_{t \geq 0}$ be a curve in $\mathcal{P} (\S)$ starting from $\mu_0 \in \mathcal{P}(\S)$ and evolving according the continuity equation driven by $\{\mbV_t\}_t$:
\begin{align}\label{eqn:measure continuity}
    \partial_t \mu_t + \dive\left(\mu_t \mbV_t \right) = 0, \quad t \geq 0.
\end{align}
The PDE \eqref{eqn:measure continuity} is understood in the distribution sense: for any smooth function $h(x)$ on $\S$, we have
    \begin{equation*}
        \frac{\de}{\de t} \int_{\S} h(x) \diff \mu_t(x) = \int_{\S} \langle \mathring{\nabla}_x h(x) , \mbV_t(x) \rangle \diff \mu_t(x)\,.
    \end{equation*}
 Here and throughout this paper, $\grade_x h(x)$ denotes the Riemannian gradient of $h$ on $\S$. Note that viewing $\S$ as an embedded manifold in $\R^d$ considerably simplifies the Riemannian calculus on the sphere. Indeed, if $H$ is a smooth extension of $h$ to a neighborhood of $\S$ in $\R^d$, we have that $\grade_x h(x)= \proj_x \nabla_x H(x)$. In particular, since $\cX_t(x) \in T_x\S$, we have $\langle \mathring{\nabla}_x h(x) , \mbV_t(x) \rangle = \langle \nabla_x H(x) , \mbV_t(x) \rangle$. Also, we use the notation $\nabla_{\mbV_t(x)} H(x)$ to denote the covariant derivative (directional derivative) along $\mbV_t(x)$.

\begin{lemma}[First Variation Formula for $\mathsf{E}_{\J}$]\label{lem:first derivative of J}
    \begin{align*}
        \frac{\de}{\de t} \mathsf{E}_{\J}[\mu_t] = \iint  \J'(\langle Ax, y \rangle) \langle  Ay, \mbV_t(x) \rangle \diff \mu_t(x) \diff \mu_t(y) .
    \end{align*}
\end{lemma}
\begin{proof}
    Because $A$ is a symmetric matrix, using \eqref{eqn:measure continuity}, we have that
        \begin{align*}
            \frac{\de}{\de t} \mathsf{E}_{\J}[\mu_t] &= \iint \langle \nabla_x \left[ \J(\langle Ax, y \rangle) \right], \mbV_t(x) \rangle \diff \mu_t(x) \diff \mu_t(y)\,,
        \end{align*}
        where we used the fact that $\mbV_t(x)\in T_x\S$. The proof  follows readily by computing the gradient above.
\end{proof}

\begin{lemma}[Second Variation Formula for $\mathsf{E}_{\J}{[}\cdot{]}$]\label{lem:second derivative of J}
    \begin{align}\label{eqn:second derivative of J}
            \frac{\de ^2}{\de t ^2} \mathsf{E}_{\J}[\mu_t]  &= \frac12\iint\J '' \left(\langle Ax, y \rangle \right) \left\| \langle Ay ,\mbV_t(x)\rangle + \langle  Ax, \mbV_t(y) \rangle \right\|_2 ^2 \diff \mu_t(x) \diff \mu_t(y) \\
            &+\iint  \J '(\langle Ax, y \rangle)  \langle A \mbV_t(x), \mbV_t(y) \rangle 
 \diff \mu_t(x) \diff \mu_t(y) \\
    &- \frac12\iint \J '(\langle Ax, y \rangle)  \langle Ax, y \rangle ( \|\mbV_t(x)\|_2 ^2 + \|\mbV_t(y)\|_2 ^2 ) \diff \mu_t(x) \diff \mu_t(y) \\
    &+ \iint \J '(\langle Ax, y \rangle) \left[ 
    \left\langle A y, \partial_t \mbV_t(x) + \mathring{\nabla} _{\mbV_t(x)} \mbV_t(x) \right \rangle \right] \diff \mu_t(x) \diff \mu_t(y) .
    \end{align}

\end{lemma}
\begin{proof}
   Taking the time derivative of the first variation formula, we get
        \begin{align*}
            \begin{split}
                \frac{\de ^2}{\de t ^2} \mathsf{E}_{\J}[\mu_t] 
    &=\underbrace{\iint \left \langle {\nabla} _x \left[ \J '(\langle Ax, y \rangle) \langle  Ay, \mbV_t(x) \rangle \right], \mbV_t(x) \right\rangle \diff \mu_t(x) \diff \mu_t(y)}_{\vartriangleleft} \\
    &+\underbrace{\iint \left \langle {\nabla} _y \left[ \J '(\langle Ax, y \rangle) \langle  Ay, \mbV_t(x) \rangle \right], \mbV_t(y) \right \rangle \diff \mu_t(x) \diff \mu_t(y)}_{\vartriangleright}\\
    &+\iint \J '(\langle Ax, y \rangle) \langle  Ay, \partial_t \mbV_t(x) \rangle \diff \mu_t(x) \diff \mu_t(y).
            \end{split}
        \end{align*}
The rest of the proof follows by direct computations, and we only highlight the key points. 

For $\vartriangleleft$, as $\mbV_t(x) \in T_x \S$, we get that 
    \begin{align*}
        \begin{split}
            &\langle \mathring{\nabla} _x \langle  Ay, \mbV_t(x) \rangle , \mbV_t(x) \rangle = \langle \mathring{\nabla} _{\mbV_t(x)} [\proj_x Ay], \mbV_t(x) \rangle + \langle \proj_x Ay, \mathring{\nabla} _{\mbV_t(x)} \mbV_t(x) \rangle
            \\  &= \langle \nabla_{\mbV_t(x)} [\proj_x Ay], \mbV_t(x) \rangle + \langle  Ay, \mathring{\nabla} _{\mbV_t(x)} \mbV_t(x) \rangle
            \\  &= -\langle Ay ,x \rangle \|\mbV_t(x)\|_2 ^2 + \langle  Ay, \mathring{\nabla} _{\mbV_t(x)} \mbV_t(x) \rangle,
        \end{split}
    \end{align*}
where in the last equality, we used the fact that 
$$
\nabla_x [\proj_x Ay] = -\nabla_x [\langle Ay,x\rangle x]= - (Ay) \otimes x -\langle Ay,x\rangle I_d,
$$
and $\langle x, \mbV_t(x)\rangle = 0$. 

Similarly, for $\vartriangleright$, as $\mbV_t(y) \in T_y \S$, we get that
    \begin{align*}
        \langle \mathring{\nabla} _y \langle  Ay, \mbV_t(x) \rangle , \mbV_t(y) \rangle = \langle \nabla  _y \langle  Ay, \mbV_t(x) \rangle , \mbV_t(y) \rangle = \langle A \mbV_t(x) , \mbV_t(y) \rangle.
    \end{align*}
The final form of the second variation formula can be obtained from the symmetric role of $x$ and $y$.
\end{proof}

Equipped with the first and second variation formulas, we are now in a position to prove~\Cref{thm:global_max}.

\subsubsection{Proof of \Cref{thm:global_max}}
    Let $\mu_0$ be a critical point of $\sE_{\J}$. We show that unless $\mu_0$ is a point mass, there exists an escape direction, that is, a velocity field $\cX_0$ such that if $\mu_t$ evolves according to~\eqref{eqn:measure continuity}, then the value of $\sE_{\J}$ increases.  Since $\mu_0$ is a stationary point, it is sufficient to check that 
        \begin{align*}
            \frac{\de ^2}{\de t ^2} \bigg|_{t=0} \mathsf{E}_{\J}[\mu_t] > 0.
        \end{align*}

Since $\mu_0$ is a critical point, it follows from the first variation formula that
    \begin{align}\label{eqn:critical point condition 2}
             \iint  \J'(\langle Ax, y \rangle) \langle  Ay, \mbV(x) \rangle \diff \mu_0(x) \diff \mu_0(y) = 0, \quad \forall \ \mbV\in C(T\S).
\end{align}
At such critical points, taking $\mbV(x)=\grade_{\mbW(x)}\mbW(x)$ in~\eqref{eqn:critical point condition 2}, the second variation formula simplifies to
        \begin{align}\label{eqn:second variation at critical point}
        \begin{split}
            &\frac{\de ^2}{\de t ^2} \bigg|_{t=0} \mathsf{E}_{\J}[\mu_t]  = \iint \frac{1}{2}\J '' \left(\langle Ax, y \rangle \right) \left\| \langle Ay ,\mbW(x)\rangle + \langle  Ax, \mbW(y) \rangle \right\|_2 ^2 \diff \mu_0(x) \diff \mu_0(y) \\
    &+ \iint \frac{1}{2} \J '(\langle Ax, y \rangle) \left[ 
    2 \langle A \mbW(x), \mbW(y) \rangle - \langle Ax, y \rangle ( \|\mbW(x)\|_2 ^2 + \|\mbW(y)\|_2 ^2 )
    \right] \diff \mu_0(x) \diff \mu_0(y) .
        \end{split}
    \end{align}
Recall that we assumed $\J'' \geq 0$ so we focus on establishing the positivity of the second line in \eqref{eqn:second variation at critical point}. To that end, following~\cite{markdahl2017almost, criscitiello2024synchronization}, define $\mbW(x) = \proj_{x}(w) = w- \langle w , x \rangle x$ where $w \in \S$.  The second line in \eqref{eqn:second variation at critical point} becomes
    \begin{align}\label{eqn:second variation along direction w}
        \begin{split}
            &\iint \frac{1}{2} \J '(\langle Ax, y \rangle) \left[ 
    2 \langle A \mbW(x), \mbW(y) \rangle - \langle Ax, y \rangle ( \|\mbW(x)\|_2 ^2 + \|\mbW(y)\|_2 ^2 )
    \right] \diff \mu_0(x) \diff \mu_0(y)  
    \\  &= \iint \frac{1}{2} \J '(\langle Ax, y \rangle) \big[ 
    2 \left( \langle Aw,w \rangle  - \langle w,x\rangle \langle  x , Aw \rangle- \langle w,y\rangle \langle  y  , Aw \rangle + \langle w , x \rangle \langle w ,y \rangle \langle Ax,y \rangle \right) 
    \\  & \quad - \langle Ax, y \rangle \left( 2- \langle w, x \rangle^2 - \langle w, y \rangle^2 \right)
    \big] \diff \mu_0(x) \diff \mu_0(y) .
        \end{split} 
    \end{align}
Pick $\{e_i\}_{i=1} ^d$ as an orthonormal basis of $\R^d$ such that $Ae_i = \lambda_i e_i$ for  $i =1,\ldots, d$. We also write $x,y$ in the coordinates of $\{e_i\}_{i=1} ^d$, that is,  $x = \sum_{i=1} ^d x_i e_i$ and $y = \sum_{i=1} ^d y_i e_i$. Choosing $w =e_i$ in  \eqref{eqn:second variation along direction w} yields
    \begin{align}\label{eqn: second variation w simplified}
        \begin{split}
            &\iint \frac{1}{2} \J '(\langle Ax, y \rangle) \left[ 
      2\lambda_i(1-x_i ^2 - y_i ^2) 
     - \langle Ax, y \rangle \left( 2- x_i ^2 - y_i^2 - 2x_i y_i\right)
    \right] \diff \mu_0(x) \diff \mu_0(y) .
        \end{split}
    \end{align}
We aim to prove the following inequality:
    \begin{align}\label{eqn: second variation w sum inequality}
        \begin{split}
            &\sum_{i=1} ^3 \iint  \J '(\langle Ax, y \rangle) \left[ 
      2\lambda_i(1-x_i ^2 - y_i ^2) 
     - \langle Ax, y \rangle \left( 2- x_i ^2 - y_i^2 - 2x_i y_i\right)
    \right] \diff \mu_0(x) \diff \mu_0(y) \geq 0,
        \end{split}
    \end{align}
with equality if and only if $\mu_0 = \delta_u$ for some point $u \in \S$ satisfying $Au = \lambda u$. This inequality directly yields the desired conclusion, since \eqref{eqn: second variation w sum inequality} implies that there is an $i \in \{1,2,3\}$ such that \eqref{eqn:second variation along direction w} with $w=e_i$ is strictly positive unless $\mu_0 = \delta_u$ for some point $u \in \S$ with $Au = \lambda u$. 

To prove \eqref{eqn: second variation w sum inequality}, we build up the  pointwise inequality: for any $x,y \in \S$,
    \begin{align}\label{eqn: partial sum positive}
        \sum_{i=1} ^3 2\lambda_i(1-x_i ^2 - y_i ^2) 
     - \langle Ax, y \rangle \left( 2- x_i ^2 - y_i^2 - 2x_i y_i\right) \geq 0,
    \end{align}
with equality if and only if the following conditions hold: for all $j$ with $|\lambda_j| < \lambda$, we have $x_j=y_j=0$; for $j$ with $\lambda_j = \lambda$, we have $x_j=y_j$; for $j$ with $\lambda_j = -\lambda$, we have $x_j=-y_j$. 

We now use \eqref{eqn: partial sum positive} to prove \eqref{eqn: second variation w sum inequality}. Suppose there exists $u \in \supp (\mu_0)$ such that $Au \neq \lambda u$. 
Writing $u = \sum_{i=1} ^d u_i e_i$, this implies that there exists some $i$ such that $\lambda_i < \lambda$ and $u_i \neq 0$. 
Choose a small neighborhood $\mathcal{N}\subset\S$ around $u$, such that $|\langle x, e_i \rangle - u_i | < |u_i|/2$ for any $x \in \mathcal{N}$. Then, for any $x,y \in \mathcal{N}$, the inequality \eqref{eqn: partial sum positive} is strictly positive because $x,y$ don't satisfy the equality condition. Therefore, the integrand in \eqref{eqn: second variation w sum inequality} is strictly positive on $\mathcal{N} \times \mathcal{N}$ (since $\phi' >0$), and nonnegative elsewhere by \eqref{eqn: partial sum positive}. This implies that \eqref{eqn: second variation w sum inequality} is strictly positive. 

In the remaining case, suppose $Au = \lambda u$ for every $u \in \supp (\mu_0)$. Then, the equality conditions in \eqref{eqn: partial sum positive} imply that $\mu_0$ must be a point mass. 
Otherwise, 
there exists $x \neq y$ in the support of $\mu_0$, and \eqref{eqn: partial sum positive} becomes strictly positive, yielding a strictly positive value in \eqref{eqn: second variation w sum inequality}.
Hence, \eqref{eqn: second variation w sum inequality} is strictly positive unless $\mu_0 = \delta_u$ for some $u \in \S $ satisfying $Au = \lambda u$.

The rest of this proof is devoted to the proof of \eqref{eqn: partial sum positive} and its equality conditions. Since $\lambda_1 = \lambda_2 = \lambda_3 =:\lambda >0$, the left hand side of \eqref{eqn: partial sum positive} becomes
    \begin{align}\label{eqn: partial sum positive 2}
        \begin{split}
            &2\lambda\left(3- \sum_{i=1} ^3 (x_i^2+y_i^2) \right) - \sum_{i=1} ^3 
      \langle Ax, y \rangle \left( 2- x_i ^2 - y_i^2 - 2x_i y_i  \right) 
     \\ &=2\left(\lambda - \langle Ax,y \rangle \right)\left(3- \sum_{i=1} ^3 (x_i^2+y_i^2) \right) + \langle Ax, y \rangle \sum_{i=1} ^3( 2x_i y_i -x_i ^2 - y_i ^2).
        \end{split}
    \end{align}
We claim that
    \begin{align}\label{eqn:partial sum claim}
        2\left(\lambda - \langle Ax,y \rangle \right) \geq \lambda \sum_{i=1} ^3 (x_i ^2 + y_i ^2 - 2x_i y_i).
    \end{align}
Indeed $\|x\|_2 ^2=\|y\|_2 ^2=1$ and $\lambda_1 = \lambda_2 = \lambda_3 =\lambda$ so that
\begin{align*}
\begin{split}
     2(\lambda - \langle Ax,y \rangle )&= \lambda\|x\|^2+\lambda \|y\|^2 - 2\sum_{i=1}^d \lambda_i x_iy_i\nonumber\\
     &= \lambda \sum_{i=1} ^3 (x_i ^2 + y_i ^2 - 2x_i y_i) +  \lambda \sum_{i=4} ^d (x_i ^2 + y_i ^2) -  2 \sum_{i=4} ^d \lambda_i x_i y_i\,.
\end{split}
\end{align*}
Hence, \eqref{eqn:partial sum claim} is equivalent to
    \begin{align}\label{eqn:partial sum claim 2}
       \sum_{i=4} ^d   \lambda(x_i ^2 + y_i ^2) \geq 2 \sum_{i=4} ^d \lambda_i x_i y_i,
    \end{align}
which holds since $\lambda \geq |\lambda_i|$ for all $i$. Moreover, note that the equality holds if and only if: $x_i = y_i$ for all $i \geq 4$ with $\lambda_i = \lambda$; $x_i = -y_i$ for all $i \geq 4$ with $\lambda_i = -\lambda$; $x_i=y_i=0$ for all $i \geq 4$ with $|\lambda_i| < \lambda$. Hence, by \eqref{eqn: partial sum positive 2} and \eqref{eqn:partial sum claim}, we have that
    \begin{align}\label{eqn: partial sum positive 3}
         \begin{split}
             &\sum_{i=1} ^3 2\lambda_i(1-x_i ^2 - y_i ^2) 
     - \langle Ax, y \rangle \left( 2- x_i ^2 - y_i^2 - 2x_i y_i \right) 
     \\ &\geq \left( \sum_{i=1} ^3 x_i ^2 + y_i ^2 - 2x_i y_i \right)\left(3\lambda- \lambda\sum_{i=1} ^3 (x_i^2+y_i^2) -\langle Ax, y \rangle \right).
         \end{split}
    \end{align}
Because $|\langle Ax, y \rangle| \leq \lambda \|x\|_2\|y\|_2 = \lambda$, and $\sum_{i=1} ^3 (x_i^2+y_i^2) \leq \sum_{i=1} ^d (x_i^2+y_i^2)=2$, we see that the right hand side of \eqref{eqn: partial sum positive 3} is nonnegative, and it can only be $0$ when $x_1=y_1$, $x_2=y_2$, $x_3=y_3$. Hence, we complete the proof for \eqref{eqn: partial sum positive}. 
By examining the equality conditions in \eqref{eqn: partial sum positive 3} and \eqref{eqn:partial sum claim 2}, we conclude that equality in \eqref{eqn: partial sum positive} holds if and only if the following is satisfied: $x_i = y_i = 0$ for all $i$ with $|\lambda_i| < \lambda$; $x_i = y_i$ for all $i$ with $\lambda_i = \lambda$; and $x_i = -y_i$ for all $i$ with $\lambda_i = -\lambda$.

\subsection{Long Time Behavior}\label{sec:long time synchronization}

 In this section, we consider a transformer model with vector field~\eqref{eq:general_vector_field} where $V_t=I_d$ and $A_t=A$ for all $t\ge 0$. Note that in absence of the preconditioner $V_t= A$, these dynamics may not be a Wasserstein gradient flow. Moreover, we only consider  measures that admit a density with respect to the uniform measure on the sphere. To reflect this, it is convenient to consider the evolution of a density rather than evolution of the measure \eqref{eq:general_mfad}.

 Let $\{\mu_t(x)\}_{t \geq 0}$ be a curve of probability measures on $\S$ satisfying the continuity equation 
    \begin{align}\label{eqn:modified gradient flow}
        \partial_t \mu_t  + \dive \left(\mu_t \deV[\mu_t] \right) = 0,
    \end{align}
where the vector field $\deV[\cdot]$ is defined for any positive measure $\nu$ on $\S$ by
    \begin{align}\label{eqn:modified wss gradient}
        \deV[\nu](x) \coloneq  \int_{\S} \proj_{x} [y] \phi'\left(\langle Ax, y\rangle \right) \diff \nu(y), \quad x \in \S,
    \end{align}
where $A$ is a $d\times d$ real symmetric matrix. 

Here and in the rest of this section, $\phi'$ is a smooth positive function on the interval $[-\|A\|_2,\|A\|_2]$. In particular, we do not require monotonicity for $\phi'$ as in \Cref{sec:crit}.

We often abuse notation and write $\deV_t=\deV[f_t]=\deV[\mu_t]$, and more generally, we liberally switch between $f_t$ and  $\mu_t$ if $\mu_t$ has density $f_t$. Define the $C^1$-norm of a continuously differentiable function $h$ on an interval $S \subseteq \R$ as $\|h\|_{C^1(S)} \coloneq \|h\|_{L^{\infty}(S)} + \|h'\|_{L^{\infty}(S)}$. The following theorem shows that, if $\phi'$ is close to the constant function $1$ in $C^1$-norm on $S \coloneq [-\|A\|_2,\|A\|_2]$, then the flow \eqref{eqn:modified gradient flow} converges to a delta mass exponentially fast. To that end, define
    \begin{align}\label{eq:epa}
        \epa \coloneq (\|A\|_2 +2) \cdot \|\phi'-1\|_{C^1(S)}.
    \end{align}
Note that when $\epa=0$ and $A=I_d$, that is when $\phi'\equiv 1$, one recovers the Kuramoto dynamics on the sphere.
Recall that $R_0$ measures the asymmetry of $f_0$ and is defined in \eqref{eqn:R0 def} and also in \eqref{def:Rt Ut}.
\begin{theorem}\label{thm:main thm}
    Let $f_0 \in L^2(\S)$ be a probability density on $\S$ and let $\{\mu_t(x)\}_{t \geq 0}$ denote the flow of probability measures where $\mu_t$ has density $f_t$ evolving according to \eqref{eqn:modified gradient flow}. There exist universal constants $c_0,c_u>0$, and two computable constants $C_0,T_0$ depending on $R_0, \|f_0\|_{L^2(\S)}$ such that if $\epa \leq c_u R_0 ^6$, then there exists an $x_{\infty} \in \S$ for which
        \begin{align*}
            W_2 \left(\mu_t , \delta_{x_{\infty}} \right) \leq C_0 e^{-c_0 t} , \quad \forall t \geq T_0.
        \end{align*}
\end{theorem}

\subsubsection{Main tools}

We adapt a technique developed in~ \cite{desvillettes2005trend} to obtain quantitative convergence rates for non-convex (and non-concave) gradient flows. For Kuramoto models, that is, when $d=2$ and $\phi' \equiv 1$, this technique was employed to derive a mean-field convergence result in \cite{ha2020emergence,morales2022trend}. Note, however, that this  technique heavily depends on the form of the vector field driving the probability flow as already noted in \cite{desvillettes2005trend}. In particular, the choice~\eqref{eqn:modified wss gradient}---which is not a gradient flow---together with the complexity of dynamics on high-dimensional spheres brings substantial technical difficulties compared to the Kuramoto model on the circle. These difficulties manifest themselves most prominently in the proofs of Theorems~\ref{thm:almost exponential decay I_t} and~\ref{thm:exponential small outside positive cap}.

For any $t \geq 0$, define
    \begin{align}\label{eqn:Kuramoto direction}
        M_t \coloneq \int_{\mathbb{S}^{d-1}} y  \diff \mu_t(y)  \quad \text{and} \quad V_t(x) \coloneq \proj_{x} [M_t] = \int_{\mathbb{S}^{d-1}} \proj_{x} [y] \diff \mu_t(y).
    \end{align}
    Interestingly, $M_t$ has a practical meaning: in encoder-only transformers such as BERT~\cite{devlin2019bert} the average token position $M_t$ corresponds to the vector embedding called \emph{mean-pooled embedding} of an input prompt that is often employed in further downstream tasks (classification, clustering, retrieval, etc.)
    
Moreover, define
\begin{align}\label{def:Rt Ut}
    R_t \coloneq \|M_t\|_2, \quad U_t \coloneq \frac{M_t}{R_t} \in \mathbb{S}^{d-1}\,,
\end{align}
and the following spherical caps with centers $\pm U \in \S$ for $\alpha \in (0, \pi /2)$,
\begin{align}
    S_{\alpha} ^+ (U) \coloneq \left\{ x \in \mathbb{S}^{d-1} \ | \ \left\langle x , U \right\rangle \geq \cos \alpha \right\}, \quad S_{\alpha} ^- (U) \coloneq \left\{ x \in \mathbb{S}^{d-1} \ | \ \left\langle x , -U \right\rangle \geq \cos \alpha \right\} . 
\end{align}
The spherical caps become smaller as $\alpha \to 0$. 
    For simplicity, we also write $S_{\alpha} ^+(t) = S_{\alpha} ^+ (U_t)$ and $S_{\alpha} ^-(t) = S_{\alpha} ^- (U_t)$.

Define $I_t$ as
\begin{align}\label{eqn:definition of I_t}
    I_t \coloneq \int \| \deV_t(y) \|_2 ^2 \diff \mu_t(y).
\end{align}
Direct calculation similar to \Cref{lem:second derivative of J}  gives
    \begin{equation}\label{eqn:derivative I_t 2}
        \partial_t I_t = \iint  Q_{\mu_t}(x,y) \diff \mu_t(x)\diff \mu_t(y)\,,
\end{equation}
where for any positive measure $\nu$ on $\S$,
\begin{align}\label{eq:defQ}
        \begin{split}
    Q_{\nu}(x,y)&=2\left[ \langle \deV[\nu](x),Ay \rangle + \langle \deV[\nu](y), Ax \rangle \right] \langle \deV[\nu](x),y \rangle \cdot \phi'' \left(\langle Ax, y\rangle \right)\\
        & +  \left[ 
            2\langle \deV[\nu](x), \deV[\nu](y) \rangle -   \langle x, y \rangle \left( \|\deV[\nu](x)\|_2 ^2 + 
            \|\deV[\nu](y)\|_2 ^2 \right)
            \right] \cdot \phi'\left(\langle Ax, y\rangle \right) .
                \end{split}
    \end{align}
We now state our two main tools.
\begin{theorem}\label{thm:almost exponential decay I_t}
If $\phi, A$ are such that $\epa \leq 1/100$, where $\epa$ is defined in~\eqref{eq:epa}, then for any $\alpha \in (0, \frac{\pi}{20})$, we have that
    \begin{align}\label{eqn:almost exponential decay I_t}
        \partial_t I_t \leq - I_t + 100\mu_t \left( \S \backslash S_{\alpha} ^+ (U_t) \right).
    \end{align}
\end{theorem}
\Cref{thm:almost exponential decay I_t} holds for any measure along the flow, even for those that do not admit a density with respect to the uniform measure, but \Cref{thm:exponential small outside positive cap} below requires a initial density in $L^2(\S)$.

\begin{theorem}\label{thm:exponential small outside positive cap}
    Fix $\alpha =\pi/{100}$. Assume that $\mu_0$ has density $f_0$ and $f_0 \in L^2(\S)$. There exist two universal constant $c_u, c_1>0$, and two computable constants $C_0,T_0$ depending on $R_0, \|f_0\|_{L^2(\S)}$ such that if $\epa \leq c_u R_0 ^6$, it holds
        \begin{align*}
            \mu_t \left( \S \backslash S_{\alpha} ^+ (U_t) \right) \leq C_0 e^{-(d-1)c_1 t}, \quad \forall t \geq T_0.
        \end{align*}
\end{theorem}

Now, we can combine \Cref{thm:almost exponential decay I_t} and \Cref{thm:exponential small outside positive cap} to prove \Cref{thm:main thm}.

\subsubsection{Proof of \Cref{thm:main thm}}
    By \Cref{thm:almost exponential decay I_t} and \Cref{thm:exponential small outside positive cap}, we see that for any $t \geq T_0$,
        \begin{align}
             I_t + \partial_t I_t \leq   10^2  C_0 e^{-(d-1)c_1 t}.
        \end{align}
    Multiply by $e^{t}$ on both sides and integrate $T_0$ to $t$ to get that for any $t \geq T_0$,
        \begin{align*}
            I_t \leq    I_{T_0} e^{T_0-t} + 10^2 C_0 (t-T_0) e^{ \max \{ -(d-1)c_1 t,-t\}},
        \end{align*}
where we used the fact that for any $t \geq T_0$ and any $\kappa \in \R$,
    \begin{align*}
        \int_{T_0} ^t e^{\kappa s } \diff s \leq (t-T_0) e^{t\cdot \max \{ \kappa  , 0 \}} .
    \end{align*}
 We see that $I_t \to 0$ exponentially fast as $t \to +\infty$. 

Also, recall that $\{\mu_t\}_t$ solves the continuity equation 
    \begin{align*}
        \partial_t \mu_t (x) + \dive \left(\mu_t (x) \deV_t(x) \right) = 0.
    \end{align*}
From~\cite[Theorem 23.9]{villani2009optimal}, 
we have that for any $s \geq T_0$ and almost all $t \geq s$, the  Wasserstein distance between $\mu_t$ and $\mu_s$ satisfies
    \begin{align*}
    \frac{1}{2}\frac{\diff}{\diff t} W^2_2(\mu_t,\mu_s) = -\int_{\S} \langle \mathring{\nabla} \psi_{t\to s} (x), \deV_t(x) \rangle \diff \mu_t(x) ,
    \end{align*}
where $ \psi_{t\to s}(x)$ is a potential function associated with the Wasserstein geodesic connecting $\mu_t,\mu_s$, and $\grade \psi_{t\to s}(x)$ satisfies that
    \begin{align*}
        \int_{\S} \|\mathring{\nabla}\psi_{t\to s}(x)\|_2 ^2   \diff \mu_t(x) = W_2 ^2 (\mu_t,\mu  _s) .
    \end{align*}
By Cauchy-Schwarz, we see that for almost all $t \geq s \geq T_0$,
    \begin{align*}
        \frac{\diff}{\diff t} W_2(\mu_t,\mu_s) \leq I_t ^{\frac{1}{2}},
    \end{align*}
where the right hand side goes to $0^+$ exponentially fast as we proved earlier. Hence, $\{\mu_t\}_{t\geq0}$ is a Cauchy sequence in the Wasserstein space. By completeness of the Wasserstein space, there exists a probability measure $\mu_{\infty}\in\mathcal{P}(\S)$ such that $\mu_t \to \mu_\infty$ in $W_2$, and $\mu_{\infty}$ satisfies that 
\begin{equation*}
    \int_{\mathbb{S}^{d-1}} \| \deV_\infty(y) \|_2 ^2 \diff \mu_\infty(y)=0.
\end{equation*}
By \Cref{thm:exponential small outside positive cap}, $\mu_t \left( \S \backslash S_{\alpha} ^+ (U_t) \right) \to 0$ as $t \to +\infty$, so there is a $U_{\infty} \in \S$ such that $\supp (\mu_{\infty}) \subseteq S_{\alpha} ^+ (U_{\infty})$, where we recall that $S_{\alpha} ^+ (U_{\infty})$ is the spherical cap defined in \eqref{eqn:definition of positive cap}. To conclude that $\mu_{\infty} = \delta_{x_{\infty}}$ for some $x_{\infty} \in \S$, we use the following Lemma.

\begin{lemma}\label{lem:hemisphere_critical_point}
Let $\mu$ be a probability measure on $\S$ with support $\supp(\mu)\subseteq S_{\alpha} ^+ (U)$ for some $ U \in \S, \alpha \in (0, \frac{\pi}{2})$ and such that
\begin{equation}\label{eqn:critical_point_condition}
 \int_{\S} \|\deV[\mu](x)\|_2^2 \diff \mu (x) = 0,
\end{equation}
Then $\mu =\delta_{x_0}$ for some $x_0\in S_{\alpha} ^+ (U)$.
\end{lemma}
\begin{proof}
From~\eqref{eqn:critical_point_condition}, we know that
\begin{align*}
\deV[\mu](x) = \int_{\S} \proj_{x} [y] \phi'\left(\langle Ax, y\rangle \right) \diff \mu (y) = 0, \quad \forall x\in\supp(\mu).
\end{align*}
Multiplying both sides by $U$ we obtain the following
\begin{equation}\label{eqn:critical_point_condition_inner_product}
\int_{\S}  ( \langle y, U \rangle - \langle x, y \rangle \langle x , U \rangle ) \phi'\left(\langle Ax, y\rangle \right) \diff \mu(y) = 0, \quad \forall x\in\supp(\mu) \,.
\end{equation}
Next, take $x=x_0$ to be any minimizer of $z \mapsto\langle z, U \rangle$ on the $\supp(\mu)$, we see that for this $x_0$, $\langle y, U \rangle - \langle x_0, y \rangle \langle x_0 , U \rangle \geq \langle y, U \rangle -  \langle x_0 , U \rangle \geq 0$ for any $y\in\supp(\mu)$. Thus, from~\eqref{eqn:critical_point_condition_inner_product} and $\phi'> 0$, we obtain that
\begin{equation*}
\langle y, U \rangle - \langle x_0, y \rangle \langle x_0 , U \rangle = 0 , \quad \forall y\in\supp(\mu) \,.
\end{equation*}
Since $\supp(\mu) \subseteq S_{\alpha} ^+ (U)$, we know that $\langle x_0, U \rangle >0$, and thus,
\begin{equation}\label{eqn:squeeze}
1\leq \frac{\langle y, U \rangle}{\langle x_0 , U \rangle} = \langle x_0, y \rangle \leq 1 \,,
\end{equation}
where in the first inequality, we use the definition of $x_0$. Hence, the inequalities in~\eqref{eqn:squeeze} are equalities, and then $\langle x_0, y \rangle = 1$ for all $y\in\supp(\mu)$, which implies that $\mu$ is a delta measure supported at $x_0$.
\end{proof}

Finally, we remark that $C_0,T_0$ in \Cref{thm:main thm} and \Cref{thm:exponential small outside positive cap} can be sharpened as follows.
\begin{theorem}\label{thm:main thm improved constants}
 Fix $\alpha =\pi/{100}$. There exists a universal constant $c_u>0$, and a $T_0>0$ such that if $\epa \leq c_u R_0 ^6$, it holds
        \begin{align*}
            \mu_t \left( \S \backslash S_{\alpha} ^+ (U_t) \right) \leq  \|f_0\|_{L^2(\S)}e^{-\frac{d-1}{16} (t-T_0)}\,, \quad \forall\ t \ge T_0\,.
        \end{align*}
Moreover, it is sufficient to take
        \begin{align*}
        T_0 \coloneq \left[ \frac{8}{R_0} \vee (d-1)\right] \cdot  \left[10^{41} (d-1)R_0 ^{-14} + 10^{26} R_0 ^{-6} \log \|f_0\|_{L^2(\S)}^2\right],
        \end{align*}
where $a \vee b \coloneq \max \{ a ,b \}$ for $a,b \in \R$.
\end{theorem}
In \Cref{thm:main thm improved constants}, it is possible to achieve a better dependence on $R_0$, specifically $R_0 ^{-2}$ using more involved arguments. We omit this result for the benefit of space and readability.


\section{{\L}ojasiewicz type inequality: Proof of \Cref{thm:almost exponential decay I_t} and \Cref{thm:classical Lojasiewicz}}\label{sec: Lojasiewicz}

This section is mainly devoted to the proof of \Cref{thm:almost exponential decay I_t}. The same proof together with \Cref{remark: large beta small cap} gives the proof for \Cref{thm:classical Lojasiewicz}.

\begin{lemma}\label{lemma:derivative I_t exponential inequality}
Assume that  $\phi, A$ are such that $\epa \leq 1/100$, where $\epa$ is defined in~\eqref{eq:epa}, and assume that a positive measure $\nu$ on $\S$ is such that there exists  $U \in \S$ and $\alpha \in (0, \frac{\pi}{20})$, such that     
\begin{align}\label{eqn:definition of positive cap}
    \supp (\nu) \subseteq S_{\alpha} ^+ (U) \coloneq \left\{ x \in \mathbb{S}^{d-1} \ | \ \left\langle x , U \right\rangle \geq \cos \alpha \right\}.
    \end{align}
Then 
\begin{equation}\label{eqn:cone_ineqn}
    \iint Q_\nu(x,y) \ud \nu(x) \ud \nu(y) \le - \nu(\S)\int \|\cY[\nu](y)\|^2_2 \ud \nu(y)\,,
\end{equation}
where $Q_\nu$ is defined in~\eqref{eq:defQ}.
In particular, taking $\nu=\mu_t$, this implies that if $\supp(\mu_t) \subseteq S_{\alpha} ^+ (U)$, the following entropy production inequality holds:
    \begin{align}\label{eqn:entropy_production}
        \partial_t I_t \le - I_t\,,
    \end{align}
    where $I_t$ is defined in~\eqref{eqn:definition of I_t}.
\end{lemma}
\begin{remark}\label{remark: large beta small cap}
    To be consistent with the assumptions in \Cref{thm:exponential small outside positive cap}, we eventually choose $\alpha = \frac{\pi}{100}$ and $\epa \leq 1/100$ in our proof for \Cref{thm:main thm}. One can also prove the same result when $\epa > /100$, but one needs to assume that $\alpha$ is less than a function in $\epa$, which goes to $0$ as $\epa$ goes to $+\infty$. For example, for the attention dynamics~\eqref{eq:mfad}  where $A = \beta I_d$ and $\phi'(r) = e^{ r}$ as in   \Cref{thm:classical Lojasiewicz}, the proof extends so long as $\tan\alpha \leq \frac{1}{10(1+\sqrt{\beta})}$, 
    and $\beta$ is any positive number. Also, one can replace the right-hand side of \eqref{eqn:cone_ineqn} with $-\frac{e^{\beta}}{10} \int_{\S} \|\deV[\mu](x) \|_2 ^2  \diff \mu(x)$, which is notably better when $\beta$ is positive and large. Similar proofs and results in \Cref{lemma:derivative I_t exponential inequality} also extend to the case when $\beta <0$ in the dynamics \eqref{eq:mfad}.
\end{remark}

\begin{proof}[Proof of \Cref{lemma:derivative I_t exponential inequality}]
    Because both sides of \eqref{eqn:cone_ineqn} are homogeneous in constant multiplies of $\nu$ of degree $4$, we can assume that $\nu$ is a probability measure, denoted $\mu$ for clarity, on $\S$. Also, in this proof, we simplify our notation to $\deV:=\deV[\mu]$.
    
    Take the standard orthonormal basis of $\R^d$ as $\{e_1,e_2,\dots, e_d\}$. Without loss of generality, we assume that $U=e_d$. We adopt the gnomonic projection to rewrite \eqref{eqn:cone_ineqn}. Note that the gnomonic projection maps any geodesic (great circle) in the upper hemisphere of $\S$  to a geodesic  (straight line) on the hyperplane $\R^{d-1} \times \{1\}$, so that the tangent vectors on $\S$ can be expressed as the difference of two points on $\R^{d-1} \times \{1\}$ under the inverse of the tangent map of the gnomonic projection. In particular, for any $x,y$ in the upper hemisphere of $\S$, $\proj_{x} [y]$ can be characterized by the geodesic connecting $x,y$, which enables us to rewrite $\proj_{x} [y]$ in the definition of $\deV(x)$ in \eqref{eqn:modified wss gradient} in the following linear form \eqref{eqn: linearity after tangent map}, and gives an important equation \eqref{eqn:vanishing integral of X} in this proof for \Cref{lemma:derivative I_t exponential inequality}. Such a property is not satisfied by stereographic projection and orthographic projection.

    \begin{figure}
    \centering

    \tdplotsetmaincoords{260}{-110}
    \begin{tikzpicture}[scale=3,tdplot_main_coords]
        
        \pgfmathsetmacro{\radius}{1} 
        \pgfmathsetmacro{\radiusB}{0.8*\radius} 
        
        \pgfmathsetmacro{\azimuthARC}{-45}
        \pgfmathsetmacro{\altitudeARC}{60}
        \pgfmathsetmacro{\Nz}{cos(\altitudeARC)}
        \pgfmathsetmacro{\Nx}{sin(\altitudeARC)*sin(\azimuthARC)}
        \pgfmathsetmacro{\Ny}{-sin(\altitudeARC)*cos(\azimuthARC)}
        
        \pgfmathsetmacro{\tanazi}{-90}    
        \pgfmathsetmacro{\tanalt}{-atan(1 / ( cos(\tanazi-\azimuthARC) * tan(\altitudeARC)) )} 
        \pgfmathsetmacro{\tanxxpp}{\radius*sin(\tanalt)*cos(\tanazi)} 
        \pgfmathsetmacro{\tanyypp}{\radius*sin(\tanalt)*sin(\tanazi)}
        \pgfmathsetmacro{\tanzzpp}{\radius*cos(\tanalt)}
        \pgfmathsetmacro{\tanr}{sqrt((\tanxxpp)^2+(\tanyypp)^2)/sqrt((\radius)^2-(\tanxxpp)^2-(\tanyypp)^2)} 
        \pgfmathsetmacro{\tanxpp}{-\tanr*cos(\tanazi)} 
        \pgfmathsetmacro{\tanypp}{-\tanr*sin(\tanazi)}
        
        \pgfmathsetmacro{\ttanazi}{-180}    
        \pgfmathsetmacro{\ttanalt}{atan(1 / ( cos(\ttanazi-\azimuthARC) * tan(\altitudeARC)) )} 
        \pgfmathsetmacro{\ttanxxpp}{\radius*sin(\ttanalt)*cos(\ttanazi)} 
        \pgfmathsetmacro{\ttanyypp}{\radius*sin(\ttanalt)*sin(\ttanazi)}
        \pgfmathsetmacro{\ttanzzpp}{\radius*cos(\ttanalt)}
        \pgfmathsetmacro{\ttanr}{sqrt((\ttanxxpp)^2+(\ttanyypp)^2)/sqrt((\radius)^2-(\ttanxxpp)^2-(\ttanyypp)^2)} 
        \pgfmathsetmacro{\ttanxpp}{-\ttanr*cos(\ttanazi)} 
        \pgfmathsetmacro{\ttanypp}{-\ttanr*sin(\ttanazi)}
        
        \pgfmathsetmacro{\xx}{-0.3}
        \pgfmathsetmacro{\xxx}{1.4}
        \pgfmathsetmacro{\kk}{(\ttanypp-\tanypp)/(\ttanxpp-\tanxpp)} 
        \pgfmathsetmacro{\yy}{\kk*(\xx-\tanxpp) + \tanypp} 
        \pgfmathsetmacro{\yyy}{\kk*(\xxx-\tanxpp) + \tanypp} 
        
        \pgfmathsetmacro{\Cx}{\Ny*\tanzzpp-\Nz*\tanyypp}
        \pgfmathsetmacro{\Cy}{\Nz*\tanxxpp-\Nx*\tanzzpp}
        \pgfmathsetmacro{\Cz}{\Nx*\tanyypp-\Ny*\tanxxpp}
        
        \pgfmathsetmacro{\CCx}{\Ny*\ttanzzpp-\Nz*\ttanyypp}
        \pgfmathsetmacro{\CCy}{\Nz*\ttanxxpp-\Nx*\ttanzzpp}
        \pgfmathsetmacro{\CCz}{\Nx*\ttanyypp-\Ny*\ttanxxpp}
    
        \draw[thick,-] (1,0,0) -- (-1,0,0) node[anchor=south]{};
        \draw[thick,-] (0,1,0) -- (0,-1,0) node[anchor=north west]{};
        \draw[thick,-] (0,0,-0.5) -- (0,0,1) node[anchor=south]{};
    
        \draw[thick,dashed] (0,0,0) -- (\tanxpp,\tanypp,1) node[anchor=south]{};
        \draw[thick,dashed] (0,0,0) -- (\ttanxpp,\ttanypp,1) node[anchor=south]{};
        
        \begin{scope}[shift={(0.5,0,1)}, opacity=0.4]
            \filldraw[fill=orange!50] (-1.5,-1.5,0) -- (1.5,-1.5,0) -- (1.5,1.5,0) -- (-1.5,1.5,0) -- cycle;
        \end{scope}
        
        \filldraw[black] (\tanxxpp,\tanyypp,\tanzzpp) circle (0.3pt) node[anchor=north west] {};
        \filldraw[black] (\ttanxxpp,\ttanyypp,\ttanzzpp) circle (0.3pt) node[anchor=south east] {};
        
        \draw[thick,dashed] (\radius,0,0) arc (0:360:\radius);
        \shade[ball color=blue!10!white,opacity=0.5] (1cm,0) arc (0:-180:1cm and 1cm) arc (180:0:1cm and 1.8mm);
        
        \begin{scope}[rotate around z=\azimuthARC, rotate around x=\altitudeARC]
            \draw[thick,dashed,blue] (\radius,0,0) arc (0:180:\radius);
        \end{scope}
        
        \draw[thick,blue,dashed] (\xx,\yy,1) -- (\xxx,\yyy,1);
        
        \filldraw[black] (\tanxpp,\tanypp,1) circle (0.3pt) node[anchor=north west] {};
        \filldraw[black] (\ttanxpp,\ttanypp,1) circle (0.3pt) node[anchor=north east] {};
        
        \draw[thick,->,black] 
            (\tanxpp, \tanypp, 1) -- 
            (\tanxpp + \Cx/2, \tanypp + \kk * \Cx/2, 1) node[anchor=north east]{};
        
        \draw[thick,->,black] 
            (\ttanxpp, \ttanypp, 1) -- 
            (\ttanxpp - \CCx/1.5, \ttanypp - \kk * \CCx/1.5, 1) node[anchor=north west]{};
        
        \filldraw[red] (0,0,1) circle (0.3pt) node[anchor=north west] {$e_d$};
        
        \draw[thick,->,black] 
            (\tanxxpp, \tanyypp, \tanzzpp) -- 
            (\tanxxpp + \Cx/3, \tanyypp + \Cy/3, \tanzzpp + \Cz/3) node[anchor=south west] {};
        
        \draw[thick,->,black] 
            (\ttanxxpp, \ttanyypp, \ttanzzpp) -- 
            (\ttanxxpp - \CCx/3, \ttanyypp - \CCy/3, \ttanzzpp - \CCz/3) node[anchor=north west] {};
        
        \end{tikzpicture}

    \caption{Illustration of gnomonic projection. }
      \label{fig:gnomonic_proj}
    \end{figure}
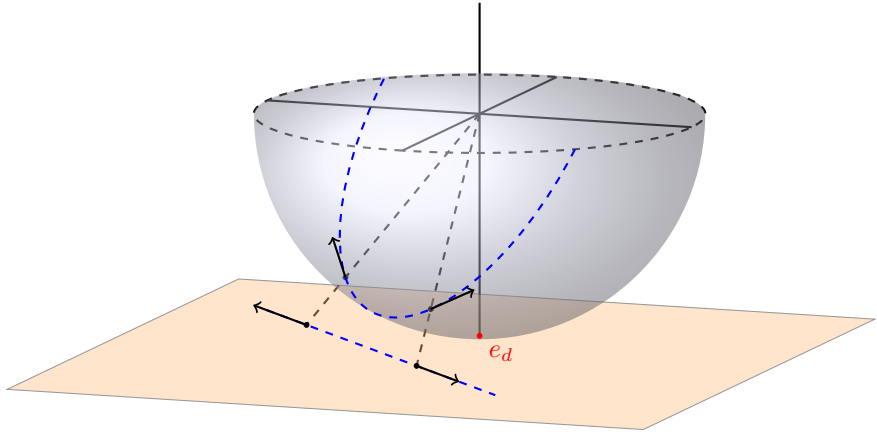

    For an $x=(x_1,\dots,x_d)^\top\in S_{\alpha} ^+ (U) \subseteq \S$, we define 
        \begin{align}
            G(x) \coloneq\left( \frac{x_1}{x_d},\dots,\frac{x_{d-1}}{x_d} \right)^\top .
        \end{align}
    This map $G(x)$ (or the map $G(x)+e_d$, so that its image is in the hyperplane $\R^{d-1} \times \{1\}$), is called the gnomonic projection. $G$ gives a diffeomorphism from $S_{\alpha} ^+ (U) \subseteq \S$ to the Euclidean ball $B_{\alpha} \subseteq \R^{d-1}$ centered at the origin and with radius $\tan \alpha$. Its inverse $F$ is given by
        \begin{align}\label{eqn:inverse gnomonic map}
            F(u) \coloneq \frac{1}{\sqrt{1+\|u\|_2 ^2}} (u+e_d), \quad \forall u \in B_\alpha .
        \end{align}
    Here we identify $u$ with a vector in $\R ^{d-1}\subseteq \R ^{d}$. A direct computation shows that, the tangent map of $F$ at $u$ is given by
        \begin{align}\label{eqn:tangent map}
            \diff F_u(X) = \frac{(1+ \|u\|_2 ^2  )X - \langle X ,u \rangle u -  \langle X ,u \rangle e_d }{\sqrt{( 1 + \|u\|_2^2)^{3}}}  , \  \forall X \in T_u \R^{d-1} \cong \R^{d-1} .
        \end{align}

    For a $u \in B_{\alpha}$, we first find the preimage of $\deV(F(u))$ under $\diff F_u$. By \eqref{eqn:inverse gnomonic map}, one can first verify that, for any $v \in B_{\alpha}$, 
        \begin{align}
            \langle F(u) , F(v) \rangle = \frac{\langle u, v\rangle +1}{\sqrt{1+\|u\|_2 ^2}\sqrt{1+\|v\|_2 ^2}}, 
        \end{align}
    and then by \eqref{eqn:tangent map}
        \begin{align}\label{eqn: linearity after tangent map}
            \proj_{F(u)} [F(v)] = F(v) - \langle F(v) , F(u) \rangle F(u) = \frac{\sqrt{1+\|u\|_2 ^2}}{\sqrt{1+\|v\|_2 ^2}} \diff F_u (v-u) .
        \end{align}
    Hence,
        \begin{align}\label{eqn:pull back deV}
            \deV(F(u)) = \diff F_u (X(u)),
        \end{align}
    with
        \begin{align}\label{eqn:def of X(u)}
            X(u) \coloneq \int_{B_{\alpha}} (v-u) \frac{\sqrt{1+\|u\|_2 ^2}}{\sqrt{1+\|v\|_2 ^2}}  \phi'\left(\langle AF(u), F(v)\rangle \right) \diff G_{\#}\mu(v),
        \end{align}
    a vector in $T_u \R^{d-1} \cong \R^{d-1}$. 
    Here, $G_{\#}\mu$ is the pushforward measure of $\mu$ induced by the gnomonic projection $G$. By symmetry of $u,v \in B_{\alpha}$ and because $A$ is a symmetric matrix, we readily obtain the following important observation:
        \begin{align}\label{eqn:vanishing integral of X}
            \begin{split}
                &\int_{B_\alpha} \frac{X(u)}{1+\|u\|_2 ^2} \diff G_{\#}\mu(u) \\ &= \int_{B_{\alpha}} \int_{B_{\alpha}} (v-u) \frac{\phi'\left(\langle AF(u), F(v)\rangle \right)}{\sqrt{(1+\|u\|_2 ^2)(1+\|v\|_2 ^2)}}  \diff G_{\#}\mu(v) \diff G_{\#}\mu(u) \\    &= 0.
            \end{split} 
        \end{align}

    Next, we rewrite the left hand side of \eqref{eqn:cone_ineqn} (or $Q_\nu(x,y)$) in terms of $X(u)$'s by replacing  $x,y \in S_{\alpha} ^+ (U)$ with $F(u),F(v)$ for $u,v \in B_\alpha$. By \eqref{eqn:tangent map} and \eqref{eqn:pull back deV}, we  obtain the following identities:
        \begin{align}\label{eqn:norm after tangent map}
            \|\deV(F(u))\|_2 ^2 = \frac{\|X(u)\|_2 ^2}{1+ \|u\|_2 ^2 } - \frac{\langle X(u) , u \rangle^2}{(1+\|u\|_2 ^2 )^2},
        \end{align}
    and 
        \begin{align}
            \begin{split}
                \left\langle \deV(F(u)) , \deV(F(v)) \right\rangle &= \frac{\left\langle X(u),X(v) \right \rangle}{\sqrt{(1+ \|u\|_2 ^2)(1+ \|v\|_2 ^2) }} 
                - \frac{\left\langle X(u),u \right \rangle\left\langle X(v),u \right \rangle}{\sqrt{(1+\|u\|_2 ^2)^{3}(1+\|v\|_2 ^2)}}
                \\  &- \frac{\left\langle X(v),v \right \rangle\left\langle X(u),v \right \rangle}{\sqrt{(1+\|v\|_2 ^2)^{3}(1+\|u\|_2 ^2)}} 
                + \frac{\left\langle X(u),u \right \rangle \left\langle X(v),v \right \rangle \left(\langle u,v  \rangle +1\right)}{\sqrt{(1+\|v\|_2 ^2)^3(1+\|u\|_2 ^2)^{3}}}.
            \end{split}
        \end{align}
Before we proceed, let us first explain our main ideas. Recall that $\alpha$ and $\epa$ are small parameters ($\alpha < \pi/20$, $\epa < 1/100$) so terms of the form $\langle X(u) ,v \rangle$ are small when $v \in B_\alpha$. Hence, after the change of variables $(x,y)\mapsto (F(u),F(v))$, the leading term on the left hand side of \eqref{eqn:cone_ineqn} becomes
    \begin{align*}
        \begin{split}
        J_1 \coloneq &\int_{B_{\alpha}} \int_{B_{\alpha}} \phi'\left(\langle AF(u), F(v)\rangle \right) \bigg[ 2\frac{\left\langle X(u),X(v) \right \rangle}{\sqrt{(1+ \|u\|_2 ^2)(1+ \|v\|_2 ^2) }} \\
&- \langle F(u), F(v) \rangle \left(\frac{\|X(u)\|_2 ^2}{1+ \|u\|_2 ^2 } 
 + \frac{\|X(v)\|_2 ^2}{1+ \|v\|_2 ^2 } \right)\bigg]\diff G_{\#}\mu(u) \diff G_{\#}\mu(v) .
        \end{split}
    \end{align*}
We also notice that, when $\alpha$ is suitably small, $\|\deV(F(u))\|_2 ^2 \sim \frac{\|X(u)\|_2 ^2}{1+ \|u\|_2 ^2 }$. 
To simplify the notations in the followings, we assume that $\tan\alpha = \sqrt{\delta}$ for some $\delta\in(0,1)$ to be determined later. In the following estimates, we frequently use the fact that $\|u\|_2^2 \leq \tan^2\alpha =\delta $. We see that $J_1=J_{11} + J_{12}$, 
where
    \begin{align*}
        \begin{split}
    J_{11} &\coloneq -\iint \phi'\left(\langle AF(u), F(v)\rangle \right)  \sqrt{(1+\|u\|_2^2)(1+\|v\|_2^2)} \left\| \frac{X(u)}{1+\|u\|_2^2} -  \frac{X(v)}{1+\|v\|_2^2}\right\|_2^2 \\
    J_{12} &\coloneq  2\iint \phi'\left(\langle AF(u), F(v)\rangle \right)  \frac{\|u\|_2^2-\langle u, v \rangle}{\sqrt{(1+\|u\|_2^2)(1+\|v\|_2^2)}} \frac{\|X(v)\|_2^2}{1+\|v\|_2^2}
        \end{split}
    \end{align*}
    and the double integrals above are over $B_\alpha\times B_\alpha$ and with respect to  $G_{\#}\mu \otimes  G_{\#}\mu$.
Clearly, 
    \begin{align*}
        \begin{split}
            J_{11} &\leq -(1-\epa) \int_{B_\alpha}\int_{B_\alpha}  \left\| \frac{X(u)}{1+\|u\|_2^2} -  \frac{X(v)}{1+\|v\|_2^2}\right\|_2^2 \diff G_{\#}\mu(u) \diff G_{\#}\mu(v)
        \\  &= -2(1-\epa) \int_{B_\alpha}  \frac{\|X(u)\|_2 ^2}{(1+\|u\|_2^2)^2}  \diff G_{\#}\mu(u) 
        \\  &\leq -\frac{2(1-\epa) }{1+\delta}\int_{B_\alpha}  \frac{\|X(u)\|_2 ^2}{1+\|u\|_2 ^2 }  \diff G_{\#}\mu(u) ,
        \end{split}
    \end{align*}
where the equality is by \eqref{eqn:vanishing integral of X}. Also,
    \begin{align*}
        J_{12} \leq 4\delta (1+\epa) \int_{B_\alpha}  \frac{\|X(u)\|_2 ^2}{1+\|u\|_2 ^2 }  \diff G_{\#}\mu(u).
    \end{align*}
By setting $\alpha\in(0,\frac{\pi}{20})$, so that $\delta\in(0,\tan^2\frac{\pi}{20})$, the above two displays imply that 
\begin{equation}
    \label{eq:J1}
    J_1 \le (-1.5 +2.5 \epa)\int \|\cY[\mu]\|^2 \ud \mu
\end{equation}
which gives us a buffer to handle the remaining terms when establishing~\eqref{eqn:cone_ineqn}.

To control these terms, observe that
    \begin{align*}
        \iint Q_{\mu}(x,y) \diff \mu(x) \diff \mu(y) - J_1 = J_2+J_3+J_4,
    \end{align*}
where
    \begin{align*}
        \begin{split}
            J_2 & \coloneq   \int_{\S} \int_{\S} 2\left( \langle \deV(x),Ay \rangle + \langle \deV(y), Ax \rangle \right) \langle \deV(x),y \rangle \cdot \phi'' \left(\langle Ax, y\rangle \right) \diff \mu(x) \diff \mu(y)
            \\ &\leq \|A\|_2 \cdot  \|\phi'-1\|_{C^1(S)} \int_{\S}\int_{\S} 2 (\|\deV(x) \|_2 ^2 + \|\deV(x) \|_2\|\deV(y) \|_2 ) \diff \mu(x) \diff \mu(y)
            \\  &\leq 4\epa \int_{\S} \|\deV(x) \|_2 ^2 \diff \mu(x) , 
        \end{split}
    \end{align*}
and
    \begin{align*}
        \begin{split}
            J_3 &\coloneq \int_{B_{\alpha}} \int_{B_{\alpha}} \phi'\left(\langle AF(u), F(v)\rangle \right) \cdot \langle F(u), F(v) \rangle \\
& \cdot \left( \frac{\langle X(u) , u \rangle^2}{(1+\|u\|_2 ^2 )^2} + \frac{\langle X(v) , v \rangle^2}{(1+\|v\|_2 ^2 )^2} \right) \diff G_{\#}\mu(u) \diff G_{\#}\mu(v) 
        \\  &\leq 2\delta (1+\epa) \int_{B_\alpha}  \frac{\|X(u)\|_2 ^2}{1+\|u\|_2 ^2 }  \diff G_{\#}\mu(u),
        \end{split}
    \end{align*}
and 
    \begin{align*}
        \begin{split}
            J_4 &\coloneq 2 \int_{B_{\alpha}} \int_{B_{\alpha}} \phi'\left(\langle AF(u), F(v)\rangle \right) \bigg( - \frac{\left\langle X(u),u \right \rangle\left\langle X(v),u \right \rangle}{\sqrt{(1+\|u\|_2 ^2)^{3} (1+\|v\|_2 ^2)}}
            \\ & - \frac{\left\langle X(v),v \right \rangle\left\langle X(u),v \right \rangle}{\sqrt{(1+\|v\|_2 ^2)^{3}(1+\|u\|_2 ^2)}} + \frac{\left\langle X(u),u \right \rangle \left\langle X(v),v \right \rangle \left(\langle u,v  \rangle +1\right)}{\sqrt{(1+\|v\|_2 ^2)^{3}(1+\|u\|_2 ^2)^{3}}} \bigg) \diff G_{\#}\mu(u) \diff G_{\#}\mu(v) 
            \\  &\leq 6 \delta (1+\epa) \int_{B_\alpha}  \frac{\|X(u)\|_2 ^2}{1+\|u\|_2 ^2 }  \diff G_{\#}\mu(u).
        \end{split}
    \end{align*}
Together with $\epa < \frac{1}{100},\delta \leq \tan^2 (\frac{\pi}{20})$ and  $\frac{\|X(u)\|_2 ^2}{1+\|u\|_2 ^2 } \geq \|\deV(F(u))\|_2 ^2$ in \eqref{eqn:norm after tangent map}, we get
    \begin{align*}
        \begin{split}
            &\int_{\S} \int_{\S} Q_{\mu}(x,y) \diff \mu(x) \diff \mu(y) 
            \\  &\leq
        -\left( 1.5 - 2.5 \epa - 4\epa - 8 \delta (1+\epa) \right) \int_{B_\alpha}  \frac{\|X(u)\|_2 ^2}{1+\|u\|_2 ^2 }  \diff G_{\#}\mu(u)
        \\ &\leq - \int_{B_\alpha}  \|\deV(F(u))\|_2 ^2  \diff G_{\#}\mu(u)
        = - \int_{\S} \|\deV(x) \|_2 ^2  \diff \mu(x) .
        \end{split}
    \end{align*}
This completes the proof of \eqref{eqn:cone_ineqn}.
\end{proof}

\begin{proof}[Proof of \Cref{thm:almost exponential decay I_t}]
   Fix $t>0$ and define the positive measures $\nu_1, \nu_2$ on $\S$ by
        \begin{align*}
            \nu_1 (\cdot ) = \mu_t(\cdot \cap S_{\alpha} ^+ (t)), \quad \nu_2 (\cdot) = \mu_t (\cdot \, \setminus S_{\alpha} ^+ (t))
        \end{align*}
    and let
        \begin{align}\label{eqn:split mu_t}
            \baV_1(x) = \int_{\S} \proj_{x} [y] \phi'\left(\langle Ax, y\rangle \right) \diff \nu_1 (y), \quad  \baV_2(x) = \int_{\S} \proj_{x} [y] \phi'\left(\langle Ax, y\rangle \right) \diff \nu_2 (y).
        \end{align}
    We see that
        \begin{align*}
            \deV_t(x) = \baV_1(x) + \baV_2(x).
        \end{align*}
    By the explicit formula \eqref{eqn:split mu_t}, we have the estimates that $\|\deV_t(x)\|_2\leq (1+\epa)$, $\|\baV_1(x)\|_2\leq (1+\epa)$, and $\|\deV_t(x)-\baV_1(x)\|_2=\|\baV_2(x)\|_2 \leq  (1+\epa) \mu_t \left( \S \backslash S_{\alpha} ^+ (t)\right)$ for any $x \in \S$. These bounds imply that
        \begin{align}\label{eqn:Yt V1 difference}
             \begin{split}
                 &\left| \langle \deV_t(x) , \deV_t(y) 
 \rangle -  \langle \baV_1(x) , \baV_1(y) \rangle \right| 
 \\ &= \left| \langle \baV_1(x), \baV_2(y) \rangle + \langle \baV_2(x), \baV_1(y) \rangle + \langle \baV_2(x), \baV_2(y) \rangle\right|
 \\ &\leq 3 (1+\epa)^2 \mu_t \left( \S \backslash S_{\alpha} ^+ (t)\right).
             \end{split}
        \end{align}
Hence, by  \eqref{eqn:derivative I_t 2}, we have 
        \begin{align*}
            \begin{split}
                \partial_t I_t &\leq \iint  \bigg[ 2\left( \langle \baV_1(x),Ay \rangle + \langle \baV_1(y), Ax \rangle \right) \langle \baV_1(x),y \rangle \cdot \phi'' \left(\langle Ax, y\rangle \right) \\
        & +  \left( 
            2\langle \baV_1(x), \baV_1(y) \rangle -   \langle x, y \rangle ( \|\baV_1(x)\|_2 ^2 + 
            \|\baV_1(y)\|_2 ^2 )
            \right) \cdot \phi'\left(\langle Ax, y\rangle \right) \bigg]  \diff \mu_t(x) \diff \mu_t(y)
            \\  &+ 24(1+\epa)^3 \mu_t \left( \S \backslash S_{\alpha} ^+ (t)\right).
            \end{split}
        \end{align*}
We can further split the above integral over $\S \times \S$ into integrals over $S_{\alpha} ^+ (t) \times S_{\alpha} ^+ (t)$ and $(\S \times \S) \backslash (S_{\alpha} ^+ (t) \times S_{\alpha} ^+ (t))$, and obtain that
            \begin{align*}
            \begin{split}
                \partial_t I_t &\leq \int_{S_{\alpha} ^+ (t)} \int_{S_{\alpha} ^+ (t)}   \bigg[ 2\left( \langle \baV_1(x),Ay \rangle + \langle \baV_1(y), Ax \rangle \right) \langle \baV_1(x),y \rangle \cdot \phi'' \left(\langle Ax, y\rangle \right) \\
        & +  \left( 
            2\langle \baV_1(x), \baV_1(y) \rangle -   \langle x, y \rangle ( \|\baV_1(x)\|_2 ^2 + 
            \|\baV_1(y)\|_2 ^2 )
            \right) \cdot \phi'\left(\langle Ax, y\rangle \right) \bigg] \diff \mu_t(x) \diff \mu_t(y)
            \\  &+ 48(1+\epa)^3 \mu_t \left( \S \backslash S_{\alpha} ^+ (t)\right).
            \end{split}
        \end{align*}
Together  with \Cref{lemma:derivative I_t exponential inequality}, the above inequality yields
    \begin{align*}
        \begin{split}
            \partial_t I_t &\leq  - \nu_1\left(\S\right)\int_{\S} \|\baV_1(x) \|_2 ^2  \diff \nu_1 (x)  + 48(1+\epa)^3 \mu_t \left( \S \backslash S_{\alpha} ^+ (t)\right)
            \\  &  =  - \left(1-\nu_2\left(\S\right) \right) \int_{\S} \|\baV_1(x) \|_2 ^2  \diff \nu_1 (x)  + 48(1+\epa)^3 \mu_t \left( \S \backslash S_{\alpha} ^+ (t)\right)
            \\ &\leq  - \int_{\S} \|\baV_1(x) \|_2 ^2  \diff \nu_1 (x)  + 49(1+\epa)^3 \mu_t \left( \S \backslash S_{\alpha} ^+ (t)\right).
        \end{split}
    \end{align*}
Note that by \eqref{eqn:Yt V1 difference}, and the estimates that $\|\baV_1(x)\|_2\leq (1+\epa)$ and $\nu_2\left(\S\right) = \mu_t \left( \S \backslash S_{\alpha} ^+ (t)\right)$, 
    \begin{align*}
        \begin{split}
            &\int_{\S} \|\baV_1(x) \|_2 ^2  \diff \nu_1 (x) = \int_{\S} \|\baV_1(x) \|_2 ^2   \diff \mu_t(x)  - \int_{\S} \|\baV_1(x) \|_2 ^2  \diff \nu_2 (x)
            \\  &\geq \int_{\S} \|\deV_t(x)  \|_2 ^2   \diff \mu_t(x)  - 3 (1+\epa)^2 \mu_t \left( \S \backslash S_{\alpha} ^+ (t)\right) - (1+\epa)^2 \mu_t \left( \S \backslash S_{\alpha} ^+ (t)\right).
        \end{split}
    \end{align*}
Hence, because $49(1+\epa)^3 + 4(1+\epa)^2 \leq 100$, we obtain that 
    \begin{align*}
        \partial_t I_t \leq -\int_{\S} \|\deV_t(x)  \|_2 ^2   \diff \mu_t(x)  +100\mu_t \left( \S \backslash S_{\alpha} ^+ (t)\right).
    \end{align*}
This completes the proof for \Cref{thm:almost exponential decay I_t}.

\end{proof}

To conclude this section, we complete the proof of \Cref{thm:classical Lojasiewicz} as a corollary of \Cref{lemma:derivative I_t exponential inequality} and \Cref{remark: large beta small cap}.

\begin{proof}[Proof of \Cref{thm:classical Lojasiewicz}]
  Recall first that $\mathsf{E}_{\beta}[\delta_{u}] = \max_{\mu \in \mathcal{P} (\S)} \mathsf{E}_{\beta}[\mu]$ by \Cref{thm:simple local max} (or \Cref{thm:global_max}). Let $\mu_t$ be the Wasserstein gradient flow initialized at $\mu_0 = \mu$. For any $t_1 \geq 0$, we define the diffeomorphisms $\{\phi_{t_1\to t}(x)\}_{t \geq t_1}$ on $\mathbb{S}^{d-1}$ by solving the ODE
    \begin{align*}
        \partial_t \phi_{t_1\to t}(x) = \mathcal{X}_{\mu_t,\beta}(\phi_{t_1\to t}(x)), \text{ with }  \phi_{t_1\to t_1}(x) = x, \ \forall x \in \mathbb{S}^{d-1}.
    \end{align*}
    We first show that $\supp (\mu_t) \subseteq S_{\alpha} ^+ (u)$ for any $t \geq 0$. Fix an arbitrary $t_1 \geq 0$, and we assume that $\supp (\mu_{t_1}) \subseteq S_{\alpha} ^+ (u)$. Let $x_{t_1} \in \supp (\mu_{t_1})$ achieve $\min_{x \in \supp (\mu_{t_1})} \langle x , u \rangle$, then
    \begin{align*}
         \begin{split}
             &\frac{\de}{\de t} \bigg|_{t=t_1} \langle \phi_{t_1\to t}(x_1) , u \rangle = \langle \mathcal{X}_{\mu_{t_1},\beta}(x_{t_1}) , u \rangle  = \int_{\S} \langle\proj_{x_{t_1}} [y] , u \rangle e^{\beta \langle  x_{t_1} ,  y\rangle}  \diff \mu_{t_1}(y)
             \\ &= \int_{\S} \left( \langle y  , u \right\rangle - \langle x_{t_1} , y\rangle \langle x_{t_1} , u\rangle )
 e^{\beta \langle  x_{t_1} ,  y\rangle}  \diff \mu_{t_1}(y) 
 \\ &\geq \int_{\S} \left(\langle x_{t_1}  , u \rangle - \langle x_{t_1} , y\rangle \langle x_{t_1} , u\rangle \right) 
 e^{\beta \langle  x_{t_1} ,  y\rangle}  \diff \mu_{t_1}(y) \geq 0,
         \end{split}
    \end{align*}
where the last inequality is because $\langle x_{t_1} , u\rangle >0 $ and $1 \geq \langle x_{t_1} , y\rangle >0$ when $y \in \supp (\mu_{t_1}) \subseteq S_{\alpha} ^+ (u)$. Hence, $\min_{x \in \supp (\mu_{t})} \langle x , u \rangle$ is nondecreasing in $t$, and then $\supp (\mu_t) \subseteq S_{\alpha} ^+ (u)$ for any $t \geq 0$. 

Define $I_t = \int_{\S} \| \mathcal{X}_{\mu_t,\beta}(x)\|_2 ^2 \diff \mu_t(x)$ so that  $I_t = \partial_t \mathsf{E}_{\beta}[\mu_t]$. Then, combine \eqref{eqn:derivative I_t 2},  \Cref{lemma:derivative I_t exponential inequality}, and \Cref{remark: large beta small cap}, together giving that $\partial_t I_t \leq - \frac{e^{\beta}}{10}I_t$. As a consequence, we see that $I_t \leq e^{-\frac{e^{\beta}}{10} t} I_0$. Using similar arguments as in the proof of \Cref{thm:main thm} in \Cref{sec:long time synchronization} we can show that there exists $x_{\infty} \in S_{\alpha} ^+ (u)$, such that $W_2( \mu_t,\delta_{x_{\infty}}) \leq \int_{t} ^{+\infty} I_r ^{\frac{1}{2}} \diff r \leq  20e^{-\beta} e^{-\frac{e^{\beta}}{20} t} I_0^{\frac{1}{2}}$, which goes to $0$ exponentially fast. Then, we integrate $\partial_t I_t \leq - \frac{e^{\beta}}{10}I_t$ from $0$ to $+\infty$, and find that
    \begin{align*}
        - I_{0}= I_{\infty} - I_{0} \leq \frac{e^{\beta}}{10} \left( - \mathsf{E}_{\beta}[\delta_{x_{\infty}}] + \mathsf{E}_{\beta}[\mu] \right).
    \end{align*}
 
\end{proof}

\section{Some Basic Derivatives and Estimates for the Proof of \Cref{thm:exponential small outside positive cap}}
If $\phi'\equiv 1$ in \eqref{eqn:modified wss gradient} so that $ \deV_t=V_t$ from~\eqref{eqn:Kuramoto direction}, then~\eqref{eqn:modified gradient flow} coincides with the classical Kuramoto model. Our main strategy is to study $f_t$ as a perturbation of the Wasserstein gradient flow driven by $V_t(x)$. In this section, we gather various perturbative results in this direction. We first define the perturbation
    \begin{align*}
        W_t(x) \coloneq \deV_t(x) - V_t(x).
    \end{align*}
    Recall that the size of this perturbation is controlled by the parameter $\epa$ defined as

\begin{equation}
    \label{eq:def_epa2}
    \epa = (\|A\|_2 +2) \cdot \|\phi'-1\|_{C^1(S)}\,.
\end{equation}
    
Observe that $\deV_t(x), V_t(x)$ can be viewed as vector fields defined on $\R ^d$ although we mainly care about $x \in \S$. The following three kinds of terms appear in our arguments. 

\begin{lemma}\label{lem:Kuramoto perturbation}
    For any $x \in \S$, we have that
    \begin{align*}
        \|W_t(x) \|_2 \leq \epa, \quad \| \nabla W_t(x)\|_2 \leq \epa
    \end{align*}
    where $\nabla$ is the standard gradient on $\R^d$.
    Also, 
    \begin{align*}
        \left | \int_{\S} \langle \partial_t W_t(x) , \deV_t(x) \rangle \diff \mu_t(x) \right | \leq \epa \cdot I_t.
    \end{align*}
\end{lemma}
\begin{proof}
    Because $x,y \in \S$ in \eqref{eqn:wss gradient}, we see that
    \begin{align*}
        \begin{split}
            & \|W_t(x) \|_2  = \big\|\deV_t(x) - V_t(x) \|_2 = \| \int_{\S} \proj_{x} [y] (\phi'\left(\langle Ax, y\rangle \right)-1) \diff \mu_t(y) \big\|_2
            \\  &\leq \|\phi'-1\|_{C^1({S)}}\int_{\S} \diff \mu_t(y) \leq \epa.
        \end{split}
    \end{align*}
    Similarly, we see that
    \begin{align*}
            \begin{split}
                &\|\nabla W_t(x)\|_2 
                \\  &= \big\| \int_{\S} \big[(\phi'\left(\langle Ax, y\rangle \right)-1)  (-y \otimes x - \langle x ,y \rangle \Id) 
                \\  &+ \phi'' \left(\langle Ax, y\rangle \right) \left(A y \otimes \proj_{x} [y]\right) \big] \diff \mu_t(y) \big\|_2
                \\  &\leq \|\phi'-1\|_{C^1({S)}} \left(\|A\|_2 +2  \right) \int_{\S}  \diff \mu_t(y) = \epa.
            \end{split}
        \end{align*}
    Finally, direct computations show that 
        \begin{align*}
            \begin{split}
                \partial_t W_t(x) &= \int_{\S} \nabla_{\deV_t(y)}[\proj_{x} [y] (\phi'\left(\langle Ax, y\rangle \right)-1)] \diff \mu_t(y) 
                \\  &=  \int_{\S} \proj_{x}\left[\left(\phi'\left(\langle Ax, y\rangle \right)-1 \right) \deV_t(y)+ \phi''\left(\langle Ax, y\rangle \right) \langle Ax,\deV_t(y) \rangle y \right] \diff \mu_t(y).
            \end{split}
        \end{align*}
    Hence,
        \begin{align*}
              \begin{split}
                  &\big | \int_{\S} \langle \partial_t W_t(x) , \deV_t(x) \rangle \diff \mu_t(x) \big | 
                  \\    &= \big | \iint \big [ (\phi'\left(\langle Ax, y\rangle \right)-1)  \langle \deV_t(y) , \deV_t(x) \rangle 
                  \\    &+  \phi''\left(\langle Ax, y\rangle \right) \langle Ax,\deV_t(y) \rangle \langle y,\deV_t(x) \rangle \big ] \diff \mu_t(x) \diff \mu_t(y) \big | 
                  \\    &\leq \frac{1}{2} \iint (1+\|A\|_2)\|\phi'-1\|_{C^1({S)}} \left( \|\deV_t(y)\|_2 ^2 + \|\deV_t(x)\|_2 ^2\right) \diff \mu_t(x) \diff \mu_t(y) 
                  \\    &=\epa \int_{\S} \| \deV_t(x) \|_2 ^2 \diff \mu_t(x).
              \end{split}
        \end{align*}
\end{proof}

\begin{lemma}\label{lemma:derivative M_t R_t}
    For the derivatives of $M_t$ and $R_t$, we have the following formulas:
        \begin{align*}
            \partial_t M_t = \int_{\mathbb{S}^{d-1}} \deV_t(y) \diff \mu_t(y) ,
        \end{align*}
    and 
        \begin{align*}
            \partial_t (R_t ^2) = 2 \int_{\mathbb{S}^{d-1}} \left( \| \deV_t (y)\|_2 ^2 - \left\langle \deV_t (y), W_t(y) \right\rangle \right) \diff \mu_t(y).
        \end{align*}
As a corollary, for any $\epa >0$ we see that
    \begin{align*}
        I_t - \epa ^2 \leq \partial_t (R_t ^2) \leq 3 I_t + \epa ^2 .
    \end{align*}
    
\end{lemma}
\begin{proof}
    The equations for $\partial_t M_t$ and $\partial_t (R_t ^2)$ follow from direct computations, and we can apply \Cref{lem:Kuramoto perturbation} to obtain the inequalities for $\partial_t (R_t ^2)$.
\end{proof}

\begin{lemma}\label{lemma:derivative I_t}
    For the derivative of $I_t$, we also have the following formula:
        \begin{align}\label{eqn:derivative I_t}
            \begin{split}
                \partial_t I_t &= \iint 
                \big[ 2 \left \langle \deV_t(x), \deV_t(y) \right \rangle - \left \langle x, y \right \rangle \left( \| \deV_t (x)\|_2 ^2 + \| \deV_t (y)\|_2 ^2 \right)  
            \\  & \quad + 2 \left \langle \deV_t(x), \partial_t W_t (x) \right \rangle + \nabla W_t(x) \left(\deV_t(x),\deV_t(x)\right) \big] \diff \mu_t(x) \diff \mu_t(y) .
            \end{split}
        \end{align}
As a corollary, if $\epa \in \left(0,\frac{1}{10} \right)$, we have that
    \begin{align*}
        \partial_t I_t \geq - 3 I_t.
    \end{align*}
In particular, for any $t_2 \geq t_1 \geq 0$,
    \begin{align*}
        I_{t_2} \geq I_{t_1} e^{-3(t_2-t_1)} .
    \end{align*}
\end{lemma}
\begin{proof}
    \eqref{eqn:derivative I_t} follows from direct computations. Similar computations  also appear in the proof of \Cref{lem:second derivative of J}, so we omit the details here. We then apply \Cref{lem:Kuramoto perturbation} and obtain that $\partial_t I_t \geq - 3 I_t$.
\end{proof}

\section{Almost Kuramoto Model: Proof of \Cref{thm:exponential small outside positive cap}}\label{sec:proof of exponential small outside positive cap thm}

In this section, we analyze the dynamics \eqref{eqn:modified gradient flow} under the assumptions of \Cref{thm:main thm}. Before moving on to the proofs, we first explain some basic schemes. For the Kuramoto model, that is, $\phi'\equiv 1$, an important fact is that $R_t$ is nondecreasing. We can also deduce that $\partial_t R_t \geq 0$ in the Kuramoto
setting directly from \Cref{lemma:derivative M_t R_t}. On the other hand, the form of \Cref{lemma:derivative M_t R_t} cannot give $\partial_t R_t \geq 0$ for our more general transformer dynamics. In fact, one does not expect that $\partial_t R_t \geq 0$ in general, as illustrated in \Cref{Fig: Rt not monotone}. For this reason, we treat the case where $\partial_t R_t$ is large and the case where $\partial_t R_t$ is small (and even negative) separately. In \Cref{sec: Rt almost increasing}, we show that $R_t$ is almost increasing, in the sense that it cannot decrease by more than a factor of $R_0$; in \Cref{sec: antipodal unstable}, we show that when $\partial_t R_t$ is small, $U_t$ is almost static, and the density around the antipodal point $-U_t$  decreases exponentially fast.

\begin{figure}
\centering
\begin{subfigure}{.5\textwidth}
  \centering
  \includegraphics[width=50mm]{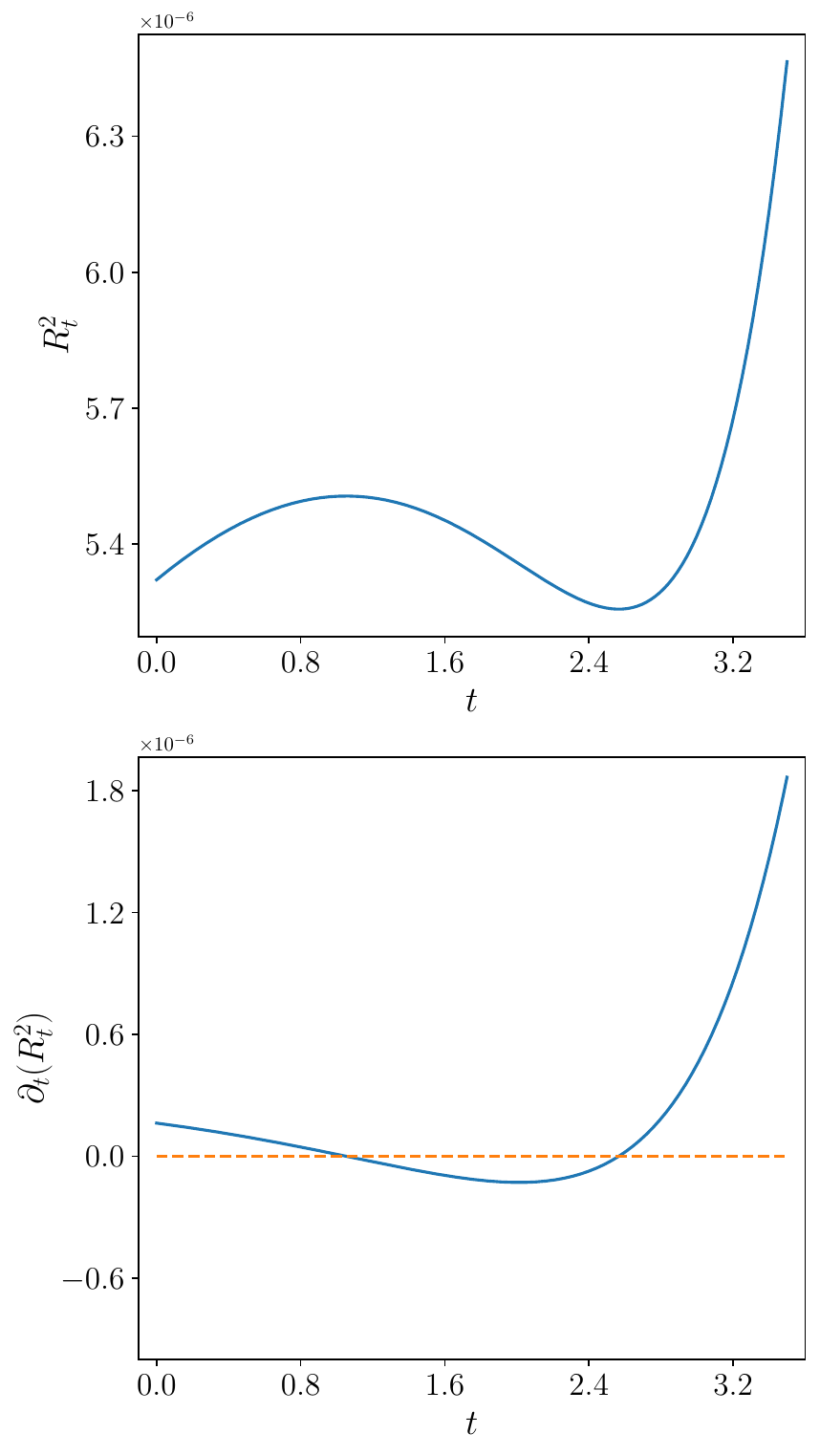}
\end{subfigure}%
\begin{subfigure}{.5\textwidth}
  \centering
  \includegraphics[width=50mm]{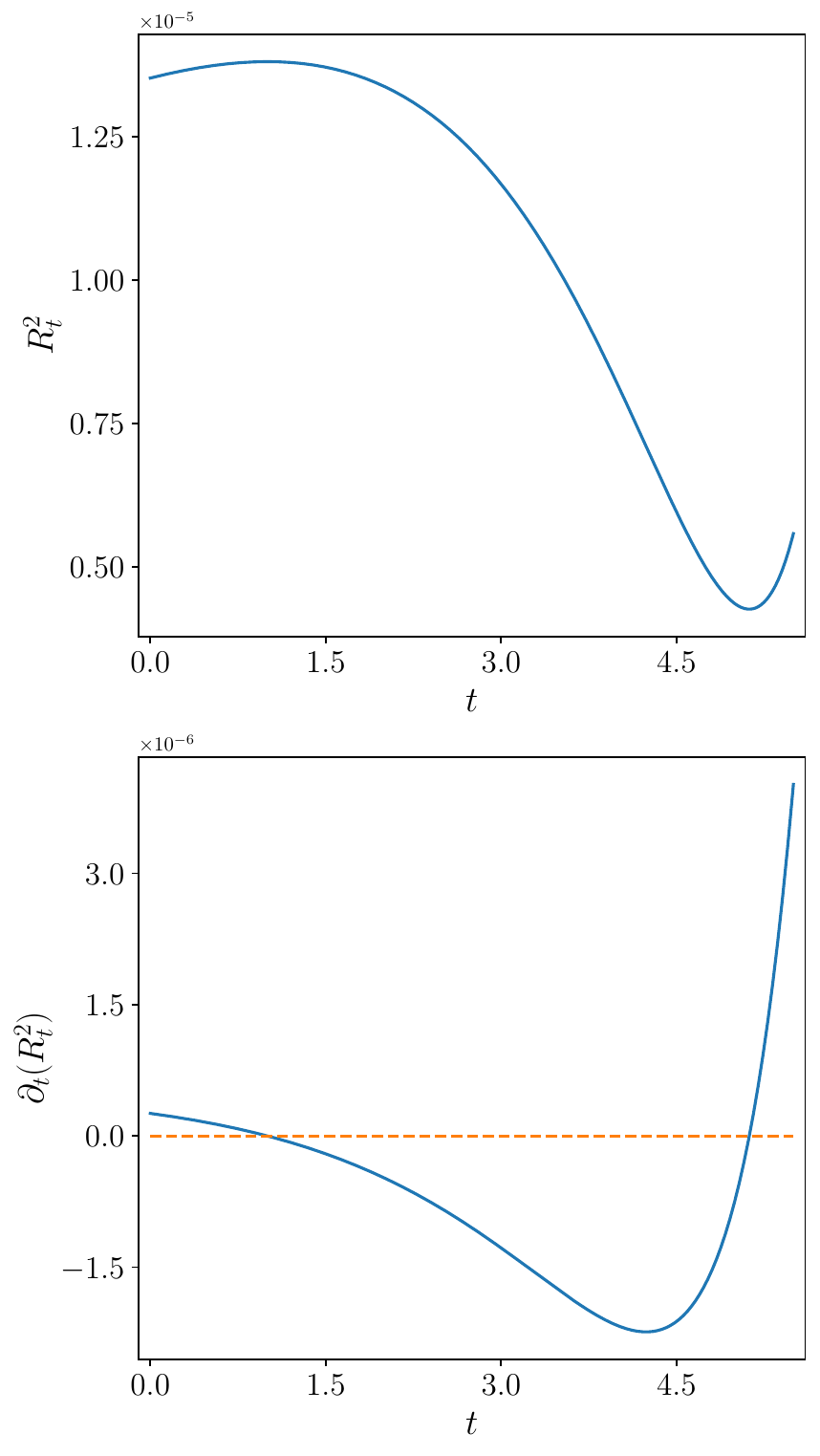}
\end{subfigure}
\caption{Examples of non-monotonic evolution of $R_t$. The top row shows $R_t^2$, and the bottom row shows $\partial_t(R_t^2)$, for two different initial profiles. Each column corresponds to a different initial profile. These plots illustrate that $R_t^2$ can exhibit non-monotonic behavior, with $\partial_t(R_t^2)$ taking both positive and negative values over time. $\phi(\langle A x, y\rangle ) = e^{0.1\langle x, y \rangle }$ in the plots.}
\label{Fig: Rt not monotone}
\end{figure}

\subsection{\texorpdfstring{$R_t$}{TEXT} is almost increasing}\label{sec: Rt almost increasing}

\begin{lemma}\label{lemma:increasing R_t}
    Fix a constant $\lambda \in (0,1)$ and an angle $\alpha \in (0, \pi/2)$. Assume that at some time $t_1 \geq 0$, 
        \begin{align*}
            R_{t_1} \geq \lambda R_0, \quad  \partial_t R_t \bigg|_{t= t_1} \geq \frac{1}{8} (\sin ^4 {\alpha})  \lambda^3 R_0 ^3 .
        \end{align*}
    If $\epa \leq 10^{-3} \sin^2 {\alpha} \lambda^2 R_0 ^2$, then $\partial_t (R_t ^2) > 0$ on $[t_1,t_1 +1]$, and
        \begin{align*}
            R_{t_1 + 1} ^2 - R_{t_1} ^2 \geq  \frac{\sin ^4 {\alpha}}{100} \lambda^4 R_0 ^4 .
        \end{align*}
\end{lemma}
\begin{proof}
    Combine \Cref{lemma:derivative M_t R_t} and \Cref{lemma:derivative I_t}, we see that for any $t \geq t_1$,
        \begin{align*}
            \begin{split}
                \partial_t (R_t ^2) \geq I_t - \epa ^2 \geq I_{t_1} e ^{-3(t-t_1)}- \epa^2 \geq \frac{1}{3} e ^{-3(t-t_1)}( \partial_t (R_t ^2)|_{t=t_1})-2\epa^2.
            \end{split}
        \end{align*}
    So, combining this inequality and the assumptions on $R_{t_1}, \partial_t R_t |_{t=t_1}, \epa$, we see that
    \begin{align*}
        \begin{split}
            \partial_t (R_t ^2) \geq \left(\frac{1}{12} e ^{-3(t-t_1)} -5 \cdot  10^{-5}\right)  (\sin ^4 {\alpha})  \lambda^4 R_0 ^4, 
        \end{split}
    \end{align*}
    which is positive when $t \in [t_1,t_1 +1]$. In particular, integrate the above inequality on $[t_1,t_1 +1]$, we see that
        \begin{align*}
            \begin{split}
                R_{t_1 + 1} ^2 - R_{t_1} ^2  \geq \frac{\sin ^4 {\alpha}}{100} \lambda^4 R_0 ^4 .
            \end{split}
        \end{align*}
\end{proof}

\begin{lemma}\label{lemma:derivative U_t}
    For the derivative of $U_t$, we have the following formula:
        \begin{align*}
            \partial_t U_t = \frac{1}{R_t} \proj_{U_t} \left[ \partial_t M_t \right] .
        \end{align*}
    As a corollary,
        \begin{align*}
            \|\partial_t U_t\|_2 \leq \frac{1}{R_t} I_t ^{\frac{1}{2}} \leq \frac{1}{R_t} \sqrt{\partial_t (R_t ^2) + \epa^2} .
        \end{align*}
\end{lemma}
\begin{proof}
    Direct computations.
\end{proof}

We then define a smooth auxiliary function $\xi_{\alpha_1,\alpha_2}(a)$ on $\mathbb{R}$, such that $\xi_{\alpha_1,\alpha_2}(a) = 1$ when $a \geq \cos \alpha_1$ and $\xi_{\alpha_1,\alpha_2}(a) = 0$ when $a \leq \cos \alpha_2$, where $0 \leq \alpha_1 < \alpha_2 \leq \pi$.  It is possible to construct such a cutoff function by mollifying indicator functions on $\R$. We denote the derivative of $\xi_{\alpha_1,\alpha_2}(a)$ with respect to $a$ as $\xi_{\alpha_1,\alpha_2} ' (a) $. A trivial fact is that we can also assume that $0 \leq \xi_{\alpha_1,\alpha_2} ' (a) \leq 2/(\cos \alpha_1 - \cos \alpha_2)$. 

\begin{lemma}\label{lemma:Negative Cap Measure}
    For the derivative of the measure on the negative spherical cap, we have the following formula:
    \begin{align}\label{eqn:derivative negative spherical cap 1}
        \begin{split}
            &\frac{\de}{\de t} \int_{\mathbb{S}^{d-1}} \xi_{\alpha_1,\alpha_2}\left(-\left\langle y, U_t \right\rangle \right) \diff \mu_t(y) 
            \\  &=  \int_{\mathbb{S}^{d-1}} \xi_{\alpha_1,\alpha_2} '   \left(-\left\langle y, U_t \right\rangle \right) \left(-\left\langle y, \partial_t U_t \right\rangle 
 - \left\langle U_t, \deV_t(y) \right\rangle \right) \diff \mu_t(y)  
            \\  &\leq \frac{2}{\cos \alpha_1 - \cos \alpha_2} \left( \|\partial_t U_t\|_2 - R_t \sin ^2 \alpha_1 + \epa \right)_+ 
            \\  &\leq \frac{2}{\cos \alpha_1 - \cos \alpha_2} \left(\frac{1}{R_t} \sqrt{\partial_t (R_t ^2) + \epa ^2} - R_t \sin ^2 \alpha_1 + \epa \right)_+ ,
        \end{split}
    \end{align}
where $u_+ \coloneq \max \{u,0 \}$ for $u \in \mathbb{R}$.
\end{lemma}
\begin{proof}
    The derivative equation in \eqref{eqn:derivative negative spherical cap 1} is by direct computations. For the first inequality, we notice that for those $y \in \S$ such that $-\cos \alpha_1 \leq \langle y,U_t\rangle \leq -\cos \alpha_2 $, we have that 
        \begin{align*}
            \begin{split}
                -\langle U_t , \deV_t(y) \rangle &= -\langle U_t,V_t(y)\rangle - \langle U_t,W_t(y)\rangle = -R_t \| \proj_{y} [U_t] \|_2 ^2 - \langle U_t,W_t(y)\rangle
                \\  &\leq - R_t \sin ^2 \alpha_1 + \epa .
            \end{split}
        \end{align*}
    The second inequality follows from \Cref{lemma:derivative U_t}.
\end{proof}

\begin{lemma}\label{lemma:almost increasing R_t derivative}
    For any $\beta \in (0, \pi/2)$, we have that
        \begin{align*}
            \partial_t (R_t ^2) \geq \frac{\sin ^2 \beta}{2} R_t ^2 \left(1- \frac{R_t + (1+ \cos \beta) \mu_{t}\left(S_{\beta} ^- (t) \right)}{\cos \beta}\right) - \epa^2 .
        \end{align*}
\end{lemma}
\begin{proof}
    According to \Cref{lemma:derivative M_t R_t}, we see that
        \begin{align*}
            \begin{split}
                \frac{1}{2} \partial_t R_t ^2 &= \int_{\S} \left( \| V_t(y) + W_t (y)\|_2 ^2 - \left\langle V_t(y) + W_t (y), W_t(y) \right\rangle \right) \diff \mu_t(y) 
                \\  &= \int_{\S} \left( \| V_t(y)\|_2 ^2 + \left\langle V_t(y), W_t(y) \right\rangle \right) \diff \mu_t(y)  
                \\  &\geq \frac{1}{2}\int_{\S} \left( \| V_t(y)\|_2 ^2 - \|W_t(y) \|_2 ^2 \right) \diff \mu_t(y)  
            \end{split}
        \end{align*}
    We notice that for $y \in \S \backslash \left(S_{\beta} ^+ (t) \cup S_{\beta} ^- (t) \right)$, we have that $\| V_t(y)\|_2 ^2 \geq R_t ^2 \sin^2 \beta$. Hence,
        \begin{align}\label{eqn:upper bound equatorial 0}
             \partial_t R_t ^2 \geq   R_t ^2 \sin^2 \beta \ \mu_{t}\left(\mathbb{S}^{d-1} \backslash (S_{\beta} ^+ (t) \cup S_{\beta} ^- (t)) \right) - \epa^2 .
        \end{align}
    On the other hand, by the definition of $R_t$, we see that
        \begin{align*}
            \begin{split}
                R_t &= \int_{\S} \langle y , U_t \rangle \diff \mu_t(y) 
                \\  
                &\geq \cos \beta \mu_t \left(S_{\beta} ^+ (t) \right) - \cos \beta \mu_{t}\left(\mathbb{S}^{d-1} \backslash (S_{\beta} ^+ (t) \cup S_{\beta} ^- (t)) \right) - \mu_t \left(S_{\beta} ^- (t) 
                \right)
                \\  &= \cos \beta  - 2\cos \beta \mu_{t}\left(\mathbb{S}^{d-1} \backslash (S_{\beta} ^+ (t) \cup S_{\beta} ^- (t)) \right) - (1+\cos \beta)\mu_t \left(S_{\beta} ^- (t) 
                \right),
            \end{split}
        \end{align*}
    and so, 
        \begin{align}\label{eqn:lower bound equatorial}
            \mu_{t}\left(\mathbb{S}^{d-1} \backslash (S_{\beta} ^+ (t) \cup S_{\beta} ^- (t)) \right) \geq \frac{1}{2} - \frac{R_t + (1+\cos \beta)\mu_t \left(S_{\beta} ^- (t) 
                \right)}{2 \cos \beta}.
        \end{align}
    Combine the above inequality and \eqref{eqn:upper bound equatorial 0}, we can obtain the formula in \Cref{lemma:almost increasing R_t derivative}. 
\end{proof}

\begin{lemma}\label{lemma:almost increasing R_t}
    Fix a constant $\lambda \in (1-10^{-3} ,1)$ and an angle $\alpha_1 \in [\pi/100, \pi/2)$. Assume that at time $t_1$, we have that 
    $R_{t_1} \geq R_0$, and there is a time window $[t_1,t_2]$, such that when $t \in [t_1,t_2]$,
        \begin{align*}
            \partial_t R_t \leq \frac{1}{8} (\sin ^4 {\alpha_1})  \lambda^3 R_0 ^3.
        \end{align*}
    Then, if
        \begin{align*}
            \epa \leq 10^{-3} (1-\lambda) \lambda^2 R_0 ^2,
        \end{align*}
    we have that for any $t \in [t_1,t_2]$,
        \begin{align}\label{eqn:R_t almost lower bound 0}
            \partial_t (R_t ^2) \geq -\frac{\sin ^2 {\beta}}{2 \cos {\beta}} \left( R_t - \lambda R_0\right) \left(R_t ^2 - \frac{3}{5}(1-\lambda)R_0(R_t+\lambda R_0)\right),
        \end{align}
    where $\beta$ is an angle in $(0,\pi/2)$ such that $\sin^2 \beta = \frac{1-\lambda}{5}R_0$. As a corollary, we have that for any $t \in [t_1,t_2]$, 
        \begin{align*}
            R_t \geq \lambda R_0 .
        \end{align*}
\end{lemma}
\begin{proof}
We use \Cref{lemma:almost increasing R_t derivative} to give a lower bound for $\partial_t R_t ^2$ on $[t_1,t_2]$, for which we actually need an upper bound for $\mu_{t}\left(S_{\beta} ^- (t) \right)$ when $t \in [t_1,t_2]$. 

First, take $\delta>0$ such that for $t \in [t_1 ,t_1 + \delta]$, we have that $R_t \geq \lambda R_0$. This is possible for some small $\delta>0$ first by the fact that $R_{t_1} \geq R_0$ and the continuity of the ODE solution $R_t$. We show that we can extend the interval $[t_1 ,t_1 + \delta]$ a little bit longer to an interval $[t_1 ,t_1 + \delta+\delta']$ for some $\delta'>0$ small, such that $R_t \geq \lambda R_0$ on $[t_1 ,t_1 + \delta+\delta']$. 

Because $\epa \leq 10^{-3} (1-\lambda) \lambda^2 R_0 ^2$ and $\partial_t R_t \leq \frac{1}{8} (\sin ^4 {\alpha_1})  \lambda^3 R_0 ^3$, we find that the right hand side of \eqref{eqn:derivative negative spherical cap 1} for any $t \in [t_1,t_1 +\delta]$ satisfies that,
    \begin{align*}
        \begin{split}
            &\frac{1}{R_t} \sqrt{\partial_t (R_{t} ^2)  + \epa ^2} - R_t \sin ^2 \alpha_1 + \epa \leq \sqrt{2 \frac{\partial_t R_t}{R_t}}  - R_t \sin ^2 \alpha_1 + 2\frac{\epa}{R_t}
            \\  &\leq -\frac{1}{2} \lambda R_0  \sin^2  \alpha_1 + \frac{2}{10^3}(1-\lambda) \lambda R_0
            \\  & \leq \lambda R_0 \left(-\frac{1}{5000} + \frac{2}{10^3} \frac{1}{10^3} \right) <0,
        \end{split}
    \end{align*}
where in the last inequality, we used the fact that $\sin \alpha_1 \geq \sin {\frac{\pi}{100}} \geq \frac{\pi}{100} \cdot \frac{2}{\pi}$ and $(1-\lambda) < 10^{-3}$. 
Hence, by the continuity of the ODE flow again, there is a small $\delta'>0$ such that for $t\in[t_1,t_1+ \delta+\delta']$, we have that the right hand side of \eqref{eqn:derivative negative spherical cap 1} is $0$, that is
    \begin{align}\label{eqn:negative cut off decreasing 0}
        \frac{\de}{\de t} \int_{\mathbb{S}^{d-1}} \xi_{\alpha_1,\alpha_2}\left(-\left\langle y, U_t \right\rangle \right) \diff \mu_t(y)  
        \leq 0
    \end{align}
for any $t \in [t_1,t_1+ \delta+\delta']$ and any $\alpha_2 \in (\alpha_1 , \pi/2)$.

Now, for any $\beta < \alpha_1 < \alpha_2 < \pi/2$ and $t \in [t_1,t_1+ \delta+\delta']$, we have that
    \begin{align*}
        \begin{split}
            \mu_{t}\left(S_{\beta} ^- (t) \right) 
            &\leq \int_{\mathbb{S}^{d-1}} \xi_{\alpha_1,\alpha_2}\left(-\left\langle y, U_t \right\rangle \right) \diff \mu_t(y) 
            \leq \int_{\mathbb{S}^{d-1}} \xi_{\alpha_1,\alpha_2}\left(-\left\langle y, U_{t_1} \right\rangle \right) \diff \mu_{t_1}(y)  
            \\  
            & \leq \mu_{t_1}\left(\mathbb{S}^{d-1} \backslash (S_{\beta} ^+ (t_1) ) \right) =  \mu_{t_1}\left(S_{\beta} ^- (t_1) \right) + \mu_{t_1}\left(\mathbb{S}^{d-1} \backslash (S_{\beta} ^+ (t_1) \cup S_{\beta} ^- (t_1)) \right) .
        \end{split}
    \end{align*}
On the other hand, we see that for any $s \geq 0$,
        \begin{align}\label{eqn:upper bound negative cap 0}
            \begin{split}
                R_s &= \int_{\S} \langle y , U_s \rangle \diff \mu_s(y) 
                \\  &\leq \mu_{s}\left(S_{\beta} ^+ (s) \right) + \cos {\beta} \mu_{s}\left(\mathbb{S}^{d-1} \backslash (S_{\beta} ^+ (s) \cup S_{\beta} ^- (s)) \right) - \cos {\beta} \mu_{s}\left(S_{\beta} ^- (s) \right)
                \\  &= \left(1- \mu_{s}\left(\mathbb{S}^{d-1} \backslash (S_{\beta} ^+ (s) \cup S_{\beta} ^- (s)) \right) - \mu_{t}\left(S_{\beta} ^- (s) \right) \right) 
                \\  & \quad + \cos {\beta} \mu_{s}\left(\mathbb{S}^{d-1} \backslash (S_{\beta} ^+ (s) \cup S_{\beta} ^- (s)) \right) - \cos {\beta} \mu_{s}\left(S_{\beta} ^- (s) \right)
                \\  &\leq 1- (1 + \cos {\beta}) \mu_{s}\left(S_{\beta} ^- (s) \right).
            \end{split}
        \end{align}
Combine \eqref{eqn:upper bound equatorial 0} and \eqref{eqn:upper bound negative cap 0} for $s=t_1$, we see that when $t\in[t_1,t_1+\delta+\delta']$,
    \begin{align*}
        \begin{split}
            &\mu_{t}\left(S_{\beta} ^- (t) \right)\leq \mu_{t_1}\left(S_{\beta} ^- (t_1) \right) + \mu_{t_1}\left(\mathbb{S}^{d-1} \backslash (S_{\beta} ^+ (t_1) \cup S_{\beta} ^- (t_1)) \right) 
            \\  &\leq \frac{1-R_{t_1}}{1 + \cos {\beta}} + \frac{ \partial_{t_1} R_{t_1} ^2 +\epa^2}{R_{t_1} ^2 \sin ^2 {\beta} } \leq \frac{1-R_{t_1}}{1 + \cos {\beta}} + \frac{\epa ^2 }{R_{t_1} ^2 \sin ^2 {\beta} } .
        \end{split}
    \end{align*}
Hence, \Cref{lemma:almost increasing R_t derivative} gives that when $t \in [t_1,t_1 + \delta+\delta']$, 
    \begin{align*}
        \partial_t (R_t ^2) \geq \frac{\sin ^2 {\beta}}{2 \cos {\beta}} \left( - R_t ^{3} + b(t_1) R_t ^2 + c(t_1)\right),
    \end{align*}
where 
    \begin{align*}
        b(t_1) = R_{t_1} + \cos \beta - 1 - \frac{1+ \cos \beta}{R_{t_1}^2 \sin ^2 \beta} \epa ^2, \quad c(t_1) = - \frac{2\cos \beta}{\sin ^2 \beta} \epa^2 .
    \end{align*}
To proceed, we first remark that if $\epa = 0$, we see that $b(t_1) = R_{t_1} + \cos \beta - 1<R_{t_1}$ and $c(t_1) = 0$. In this case, we see that $\partial_t (R_t ^2) \geq \frac{\sin ^2 {\beta}}{2 \cos {\beta}} R_t ^2 (b(t_1) - R_t)$. The right hand side is nonnegative once $R_t$ reaches $b(t_1)$. Hence, for $t \in [t_1,t_1 + \delta+\delta']$, we have that $R_t \geq b(t_1) \geq R_0 - \frac{1-\lambda}{5}R_0 > \lambda R_0$ . This strict inequality gives a little room when $\epa \neq 0$. As in the assumption, we have that $\sin^2 \beta = \frac{1-\lambda}{5}R_0$, $\epa \leq 10^{-3} (1-\lambda) \lambda^2 R_0 ^2$, and $R_{t_1} \geq R_0$. So, 
    \begin{align*}
        \begin{split}
            b(t_1) &\geq R_0 -  \frac{1-\lambda}{5}R_0 - \frac{10}{R_0^2 (1-\lambda)R_0} \left(10^{-3} (1-\lambda) \lambda^2 R_0 ^2\right) ^2
            \\  &\geq R_0 -  \frac{1-\lambda}{5}R_0 -\frac{1-\lambda}{5}R_0 = \lambda R_0 + \frac{3}{5}(1-\lambda) R_0 , 
        \end{split}
    \end{align*}
and
    \begin{align*}
        \begin{split}
            c(t_1) &\geq - \frac{10}{(1-\lambda) R_0} \left(10^{-3} (1-\lambda) \lambda^2 R_0 ^2 \right) ^2 \geq - \frac{3}{5} (1-\lambda) \lambda^2 R_0 ^3.
        \end{split}
    \end{align*}
Hence, when $t \in [t_1,t_1+\delta+\delta']$, 
    \begin{align*}
        \begin{split}
            \partial_t (R_t ^2) &\geq \frac{\sin ^2 {\beta}}{2 \cos {\beta}} \left( - R_t ^{3} + \left(\lambda R_0 + \frac{3}{5}(1-\lambda) R_0 \right) R_t ^2 - \frac{3}{5} (1-\lambda) \lambda^2 R_0 ^3 \right)
        \\  &= -\frac{\sin ^2 {\beta}}{2 \cos {\beta}} \left( R_t - \lambda R_0\right) \left(R_t ^2 - \frac{3}{5}(1-\lambda)R_0(R_t+\lambda R_0)\right),
        \end{split}
    \end{align*}
which is exactly \eqref{eqn:R_t almost lower bound 0}.
We also see that $\lambda R_0$ is strictly larger than the roots of the quadratic polynomial $x ^2 - \frac{3}{5}(1-\lambda)R_0(x+\lambda R_0)$ because $\lambda \in (\frac{2}{3},1)$. Because $R_{t_1} \geq R_0$, as in the argument for $\epa = 0$, we see that when $t \in [t_1 ,t_1 + \delta+\delta']$, $R_{t} \geq \lambda R_0$.

We finally remark that the only assumption we made in the proof is that for some $\delta>0$ and for $t \in [t_1 ,t_1+\delta]$ we have that $R_t \geq \lambda R_0$. Under this assumption, we obtained a $\delta'>0$ such that for $t \in [t_1 ,t_1+\delta+\delta']$, we have that \eqref{eqn:R_t almost lower bound 0} holds true and also $R_t \geq \lambda R_0$. By taking $[t_1,t_1+\delta]$ as the supremum interval on which $R_t \geq \lambda R_0$, we get that for any $t \in [t_1,t_2]$, \eqref{eqn:R_t almost lower bound 0} holds true, and $R_t \geq \lambda R_0$.

\end{proof}

\begin{lemma}\label{lemma: all time lower bound R_t}
    Fix a constant $\lambda \in (1-10^{-10} R_0 ^2 ,1)$. Then, if 
        \begin{align*}
            \epa \leq 10^{-3} (1-\lambda) \lambda^2 R_0 ^2,
        \end{align*}
    we have that
        \begin{align*}
            R_t \geq \lambda R_0,
        \end{align*}
    for any $t \geq 0$.
\end{lemma}
\begin{proof}
     Fix an angle $\alpha \in [\pi/100,\pi/2)$. We use \Cref{lemma:increasing R_t} and \Cref{lemma:almost increasing R_t}.  We first remark that the assumption on $\epa$ satisfies the assumptions we made in \Cref{lemma:increasing R_t} and \Cref{lemma:almost increasing R_t}. 
    
    If for any $t \geq 0$, we have that $\partial_t R_t \leq \frac{1}{8} (\sin ^4 {\alpha})  \lambda^3 R_0 ^3$, we see that $R_t \geq \lambda R_0$ for any $t \geq 0$ by \Cref{lemma:almost increasing R_t}. Otherwise, assume that $t_1 \geq 0$ is the first time such that $\partial_t R_t |_{t=t_1} > \frac{1}{8} (\sin ^4 {\alpha})  \lambda^3 R_0 ^3$. We have that $R_t \geq \lambda R_0$ on $[0,t_1]$, in particular, $R_{t_1} \geq \lambda R_0$. By \Cref{lemma:increasing R_t}, $R_t$ is increasing on $[t_1,t_1+1]$. By denoting $\eta = 1-\lambda^2 $, we have that
        \begin{align*}
            \begin{split}
                &R_{t_1 + 1} ^2 \geq  R_{t_1} ^2 +  \frac{\sin ^4 {\alpha}}{100} \lambda^4 R_0 ^4 \geq \lambda ^2 R_0 ^2 + \frac{16}{10^{10}} \lambda^4 R_0 ^4
                \\  &= R_0 ^2 \left(1-\eta + \frac{16}{10^{10}} (1-\eta)^2 R_0^2 \right) \geq R_0 ^2 \left(1+\frac{16 R_0 ^2}{10^{10}} -2\eta \right) \geq R_0 ^2 \left(1+\frac{12 R_0 ^2}{10^{10}} \right),
            \end{split}
        \end{align*}
    where we also used the fact that $\sin \alpha \geq \sin {\frac{\pi}{100}} \geq \frac{\pi}{100} \cdot \frac{2}{\pi}$, and the last inequality is because $\eta = 1- \lambda^2 \leq 2(1-\lambda) \leq \frac{2 R_0 ^2}{10^{10}}$, by the assumption of $\lambda$.
    Hence, we see that $R_t \geq \lambda R_0$ and $R_{t_1 +1} \geq R_0$ for $t \in [0,t_1 +1]$. We can run this argument again starting from $t=t_1 +1$ and extend the interval at least by $1$. So, $R_t \geq \lambda R_0$ for any $t \geq 0$.
\end{proof}

\subsection{Measures on Negative Caps Decrease Exponentially Fast}\label{sec: antipodal unstable}

In the following, we show an instability result for the negative spherical cap. Recall that $\mu_t$ has a density $f_t \in L^2 (\S)$.
\begin{lemma}\label{lemma:Negative Cap Square Derivative}
    For any $\alpha_1, \alpha_2 $ with $0\leq \alpha_1 < \alpha_2 \leq \pi$, we have the following formula:
        \begin{align}\label{eqn:derivative negative cap square}
        \begin{split}
            & \frac{\de}{\de t} \int_{\mathbb{S}^{d-1}} \xi_{\alpha_1,\alpha_2}\left(-\left\langle y, U_t \right\rangle \right) f_t ^2  (y) \diff y
            \\  &=  \int_{\mathbb{S}^{d-1}} \xi_{\alpha_1,\alpha_2} '   \left(-\left\langle y, U_t \right\rangle \right) \left(-\left\langle y, \partial_t U_t \right\rangle 
 - \left\langle U_t, \deV_t(y) \right\rangle \right) f_t ^2 (y) \diff y
    \\  & \quad + \int_{\mathbb{S}^{d-1}} \xi_{\alpha_1,\alpha_2}   \left(-\left\langle y, U_t \right\rangle \right) \left( (d-1) \left\langle M_t , y \right\rangle -  \dive_{\mathbb{S}^{d-1}} W_t(y) \right) f_t ^2 (y) \diff y.
        \end{split}
    \end{align}
Also,
    \begin{align}\label{eqn:L2_norm_gronwall}
        \frac{\de}{\de t} \|f_t\|_{L^2(\S)}^2 \leq (d-1)( R_t + \epa) \|f_t\|_{L^2(\S)}^2 
        .
    \end{align}
\end{lemma}
\begin{proof}
    \eqref{eqn:derivative negative cap square} is by direct computations. For the inequality, if we let $\alpha_1 \to \pi^{-}$ (or just replace $\xi_{\alpha_1,\alpha_2}$ with the constant function $1$),  we see that
        \begin{align*}
            \begin{split}
                \frac{\de}{\de t} \|f_t\|_{L^2(\S)}^2
                &= \int_{\mathbb{S}^{d-1}}  \left( (d-1) \left\langle M_t , y \right\rangle -  \dive W_t(y) \right) f_t ^2 (y) \diff y
                \\  &\leq (d-1)( R_t + \epa) \|f_t\|_{L^2(\S)}^2
                .
            \end{split}
        \end{align*}
\end{proof}

\begin{lemma}\label{lem:exponential decay negative cap square}
Fix a constant $\lambda \in (2/3,1)$ and an angle $\alpha_1 \in [\pi/100, \pi/2)$.
If there is a time window $[t_1,t_2]$, such that when $t \in [t_1,t_2]$,
    \begin{align*}
            R_t \geq \lambda R_0, \quad \partial_t R_t \leq \frac{1}{8} (\sin ^4 {\alpha_1})  \lambda^3 R_0 ^3,
    \end{align*}
and if
    \begin{align*}
        \epa \leq \frac{1}{10^4} \lambda^2 R_0 ^2 \cos{\alpha_1},
    \end{align*}
then, we have that for any $t \in [t_1,t_2]$,
    \begin{align*}
        \frac{\de}{\de t} f_t ^2 \left(S_{\alpha_1} ^- (t)\right) \leq -\frac{(d-1) \lambda R_0\cos{\alpha_1}}{2} f_t ^2 \left(S_{\alpha_1} ^- (t)\right).
    \end{align*}
    Here, we define $f_t^2(A) = \int_A f_t^2(x) \diff x $ for any measurable set $A\subseteq \S$.
\end{lemma}
\begin{proof}
    By \Cref{lemma:Negative Cap Square Derivative}, we see that for any $\alpha_2 \in (\alpha_1, \pi/2)$, we have that
    \begin{align*}
        \begin{split}
            & \frac{\de}{\de t} \int_{\mathbb{S}^{d-1}} \xi_{\alpha_1,\alpha_2}\left(-\left\langle y, U_t \right\rangle \right) f_t ^2  (y) \diff y
            \\  &\leq  \int_{\mathbb{S}^{d-1}} \xi_{\alpha_1,\alpha_2} '   \left(-\left\langle y, U_t \right\rangle \right) \left( \|\partial_t U_t\|_2 
 - R_t \sin ^2 {\alpha_1} + \epa \right) f_t ^2 (y) \diff y
    \\  & \quad + \int_{\mathbb{S}^{d-1}} \xi_{\alpha_1,\alpha_2}   \left(-\left\langle y, U_t \right\rangle \right) \left( -(d-1) R_t \cos {\alpha_2} +  (d-1)\epa \right) f_t ^2 (y) \diff y.
        \end{split}
    \end{align*}
Notice that we used the fact that $\left\langle U_t, V_t(y) \right\rangle = R_t \| \proj_{y} [U_t] \|_2 ^2 \geq R_t \sin ^2 {\alpha_1}$ when $ \cos \alpha_2 \leq -\left\langle y, U_t \right\rangle \leq \cos \alpha_1$, and the fact that $\left\langle M_t , y \right\rangle \leq - R_t \cos {\alpha_2}$ when $ \cos \alpha_2 \leq -\left\langle y, U_t \right\rangle \leq 1$.

Combine \Cref{lemma:derivative U_t}, we see that, similar to the proof of \Cref{lemma:almost increasing R_t}, by the assumptions on $R_t, \partial_t R_t, \epa$, 
    \begin{align*}
        \begin{split}
                &\|\partial_t U_t\|_2 
 - R_t \sin ^2 {\alpha_1} + \epa \leq \sqrt{2 \frac{\partial_t R_t}{R_t}}  - R_t \sin ^2 \alpha_1 + 2\frac{\epa}{R_t}
 \\ &\leq \lambda R_0 \left[-\frac{1}{2} \sin^2 \alpha_1 + \frac{\cos \alpha_1}{5000}\right] \leq \lambda R_0 \left[  -\frac{1}{2} \left(\frac{1}{50}\right)^2+ \frac{\cos \alpha_1}{5000}\right] <0,
        \end{split}
    \end{align*}
for any $t \in [t_1,t_2]$, where we used the fact that $\sin \left( \frac{\pi}{100} \right) \geq \frac{2}{\pi} \cdot \frac{\pi}{100}$. Also,
    \begin{align*}
        -(d-1) R_t \cos {\alpha_2} +  (d-1)\epa \leq (d-1)\lambda R_0 \left(- \cos \alpha_2+ \frac{1}{10} \cos \alpha_1 \right). 
    \end{align*}
Hence, combine the above two parts, we have that for any $t \in [t_1,t_2]$ and any $\alpha_2 \in (\alpha_1, \pi/2)$,
    \begin{align*}
        \begin{split}
            &\frac{\de}{\de t} \int_{\mathbb{S}^{d-1}} \xi_{\alpha_1,\alpha_2}\left(-\left\langle y, U_t \right\rangle \right) f_t ^2  (y) \diff y
            \\  &\leq (d-1)\lambda R_0 \left(- \cos \alpha_2+ \frac{1}{10}\cos \alpha_1 \right) \int_{\mathbb{S}^{d-1}} \xi_{\alpha_1,\alpha_2}\left(-\left\langle y, U_t \right\rangle \right) f_t ^2  (y) \diff y.
        \end{split}
    \end{align*}
We can then obtain the conclusion by sending $\alpha_2 \to \alpha_1 ^+$.
\end{proof}

Next, for any $t_1 \geq 0$, we define the diffeomorphisms $\{\phi_{t_1\to t}(x)\}_{t \geq t_1}$ on $\mathbb{S}^{d-1}$ by solving the ODE
    \begin{align*}
        \partial_t \phi_{t_1\to t}(x) = \deV_{t}(\phi_{t_1\to t}(x)), \text{ with }  \phi_{t_1 \to t_1}(x) = x, \ \forall x \in \mathbb{S}^{d-1} .
    \end{align*}
\begin{lemma}\label{lemma:shrink from negative cap to positive cap}
    Fix a constant $\lambda \in (2/3,1)$ and an angle $\alpha \in [\pi/100, \pi/2)$.
If there is a time window $[t_1,t_2]$, such that when $t \in [t_1,t_2]$,
    \begin{align*}
        R_t \geq \lambda R_0, \quad \partial_t R_t \leq \frac{1}{8} (\sin ^4 {\alpha})  \lambda^3 R_0 ^3,
    \end{align*}
and if
    \begin{align*}
        \epa \leq \frac{1}{10^3} \lambda^2 R_0 ^2 \sin \alpha,
    \end{align*}
then, for any $t_3 ,t_4 \in [t_1, t_2]$, $t_3 \leq t_4$, and $x \in \S$ such that 
    \begin{align*}
        \phi_{t_3\to t_4}(x) \in \S  \backslash \left( S_{\alpha} ^- (t_4) \cup S_{\alpha} ^+ (t_4) \right), 
    \end{align*}
we have that
    \begin{align}\label{eqn:derivative shrink to positive cap}
        \frac{\de}{\de t} \bigg|_{t=t_4} \left\langle \phi_{t_3\to t}(x) , U_t \right\rangle \geq \lambda R_0 \frac{\sin ^2{\alpha}}{4}.
    \end{align}
As a corollary, if we define,
    \begin{align}\label{eqn:small time}
        \delta=\delta(\lambda,R_0, \alpha) \coloneq \frac{4}{\lambda R_0 \sin^2 \alpha},
    \end{align}
then if $t_2 - t_1 \geq \delta$, we have that
    \begin{align*}
        \phi_{t_1\to t_2} \left(\mathbb{S}^{d-1} \backslash S_{\alpha} ^- (t_1) \right) \subseteq S_{\alpha} ^+ (t_2), \quad  \mathbb{S}^{d-1} \backslash S_{\alpha} ^+ (t_2) \subseteq  \phi_{t_1\to t_2} \left( S_{\alpha} ^- (t_1) \right) .
    \end{align*}
\end{lemma}
\begin{proof}
    By \Cref{lemma:derivative U_t}, and the assumptions on $R_t, \partial_t R_t, \epa$, we see that
        \begin{align*}
            \begin{split}
                &\frac{\de}{\de t} \bigg|_{t=t_4} \left\langle \phi_{t_3\to t}(x) , U_t \right\rangle = \left \langle \deV_{t_4} \left(\phi_{t_3\to t_4}(x) \right) , U_t \right \rangle + \left \langle \phi_{t_3\to t_4}(x), \partial_t U_t \right \rangle
                \\  &= \left \langle V_{t_4} \left(\phi_{t_3\to t_4}(x) \right) , U_t \right \rangle + \left \langle W_{t_4} \left(\phi_{t_3\to t_4}(x) \right) , U_t \right \rangle + \left \langle \phi_{t_3\to t_4}(x), \partial_t U_t \right \rangle
                \\  &\geq R_{t_4} \| \proj_{\phi_{t_3\to t_4}(x)} [U_t]\|_2 ^2 - \epa - \| \proj_{\phi_{t_3\to t_4}(x)} [U_t]\|_2 \frac{1}{R_{t_4}} \left( \sqrt{\partial_t (R_t ^2) |_{t=t_4}} + \epa\right)
                \\  &\geq \lambda R_0 \| \proj_{\phi_{t_3\to t_4}(x)} [U_t]\|_2 \left ( \| \proj_{\phi_{t_3\to t_4}(x)} [U_t]\|_2 - \frac{1}{2} \sin^2{\alpha} \right )  - \frac{2}{\lambda R_0}\epa 
                \\  &\geq \lambda R_0 (\sin{\alpha}) \left ( \sin{\alpha} - \frac{1}{2} \sin^2{\alpha} \right )  - \frac{2}{\lambda R_0}\epa 
                \\  &\geq \lambda R_0 \frac{\sin ^2{\alpha}}{2} - \lambda R_0 \frac{\sin \alpha}{500} \geq \lambda R_0 \frac{\sin ^2{\alpha}}{4}, 
            \end{split}
        \end{align*}
    where in the first inequality, we used the fact that $\partial_t U_t$ is in the tangent plane of $U_t$ and $\| \proj_{U_t} [\phi_{t_3\to t_4}(x)] \|_2 = \| \proj_{\phi_{t_3\to t_4}(x)} [U_t]\|_2$, and in the last inequality, we used that fact that $\frac{\pi}{100} \cdot \frac{2}{\pi} \leq \sin {\frac{\pi}{100}} \leq \sin \alpha$.
\end{proof}

\begin{theorem}\label{thm:length of small dR_t intervals}
    Fix a constant $\lambda \in (1-10^{-10} R_0 ^2 ,1)$ and the angle $\alpha 
 = \pi/100$.
If there is a time window $[t_1,t_2]$, such that when $t \in [t_1,t_2]$,
    \begin{align*}
            R_t \geq \lambda R_0, \quad \partial_t R_t \leq \frac{1}{8} (\sin ^4 {\alpha})  \lambda^3 R_0 ^3,
    \end{align*}
and if
    \begin{align*}
        \epa \leq \frac{1}{10^4} \lambda^2 R_0 ^2 \sin{\alpha}\cos{\alpha},
    \end{align*}
then there is a $T$ of the form
    \begin{align}\label{eqn:T dependence on t1}
        T = C_u t_1 + C_0
    \end{align}
where $C_u$ is a universal constant and $C_0$ is a constant depending on $R_0$ and $\|f_0\|_{L^2(\S)}$ 
, such that either $t_2 - t_1 \leq T$ or $t_2 = +\infty$.
\end{theorem}
\begin{proof}
    Assume that $t_2 - t_1 > T > \delta(\lambda,R_0,\alpha)$, where $\delta(\lambda,R_0,\alpha)$ is defined in \eqref{eqn:small time}. For any $r \in [t_1+\delta,t_2]$, by \Cref{lemma:shrink from negative cap to positive cap}, we see that
        \begin{align*}
            \mu_{r}\left(\S \backslash S_{\alpha} ^+ (r)  \right) \leq \mu_{r}\left( \phi_{r - \delta,r} \left( S_{\alpha} ^- (r - \delta) \right) \right) = \mu_{r - \delta} \left( S_{\alpha} ^- (r - \delta) \right).
        \end{align*}
    By \Cref{lem:exponential decay negative cap square}, we see that
        \begin{align*}
            f_{r - \delta} ^2 \left( S_{\alpha} ^- (r - \delta) \right) \leq e^{-\frac{(d-1) \lambda R_0\cos{\alpha}}{2} (r-\delta-t_1)} f_{t_1} ^2 \left( S_{\alpha} ^- (t_1) \right).
        \end{align*}
    Hence, by H\"older's inequality, we have that
        \begin{align}\label{eqn:exponential decay on long interval 2}
            \mu_{r}\left(\S \backslash S_{\alpha} ^+ (r)  \right) \leq C_d \cdot e^{-\frac{(d-1) \lambda R_0\cos{\alpha}}{4} (r-\delta-t_1)} \left[ f_{t_1} ^2 \left( S_{\alpha} ^- (t_1) \right) \right]^{\frac{1}{2}},
        \end{align}
    where $C_d>0$ is a constant depending on $d$ (actually the square root of the surface measure of $\S$, which is less than $10$ for any $d$). By \Cref{lemma:Negative Cap Square Derivative}, we see that 
        \begin{align*}
            \left[ f_{t_1} ^2 \left( S_{\alpha} ^- (t_1) \right) \right]^{\frac{1}{2}} 
            \leq \|f_{t_1}\|_{L^2(\S)} 
            \leq e^{(d-1)t_1} \|f_{0}\|_{L^2(\S)} 
        \end{align*}
Also, by the choice of $\delta$ in \eqref{eqn:small time}, and $\alpha \geq \frac{\pi}{100}$ we see that 
    \begin{align*}
        \frac{(d-1) \lambda R_0\cos{\alpha}}{4} \delta \leq 10^4 (d-1).
    \end{align*}
Hence,
    \begin{align}\label{eqn:exponential decay outside positive cap}
        \mu_{r}\left(\S \backslash S_{\alpha} ^+ (r)  \right) \leq 10  \cdot e^{-\frac{(d-1) \lambda R_0\cos{\alpha}}{4} r} \cdot e^{(d-1)(2t_1 + 10 ^4)} \cdot \|f_0\|_{L^2(\S)} 
        .
    \end{align}
On the other hand, by \Cref{thm:almost exponential decay I_t}, we have that $I_t+\partial_t I_t \leq 10^2 \mu_t \left( \S \backslash S_{\alpha} ^+ (t) \right)$. Multiply $e^t$ on both sides and integrate from $t_1$ to $r$, we obtain that
    \begin{align*}
        I_r \leq I_{t_1}e^{-r+t_1}+ \frac{10^3  \cdot e^{(d-1)(3t_1 + 10 ^4)} \cdot \|f_0\|_{L^2(\S)}}{1-\frac{(d-1) \lambda R_0\cos{\alpha}}{4}} e^{-\frac{(d-1) \lambda R_0\cos{\alpha}}{4} r}.
    \end{align*}
We notice that, if $\frac{\lambda R_0\cos{\alpha}}{4} r$ is much larger than $3t_1 + 10 ^4$, the right hand side of the above inequality can be very small, and hence $I_r$ is quantitatively small. Also, notice that, by \Cref{lemma:derivative M_t R_t} and the assumption that $R_t \geq \lambda R_0$, we have that for any $r \in [t_1 + \delta,t_2]$,
    \begin{align*}
        \partial_t R_t \big|_{t=r}\leq \frac{3I_r}{2\lambda R_0}+ \frac{\epa^2}{2\lambda R_0} \leq \frac{3I_r}{2\lambda R_0}+ 10^{-8} \lambda^3 R_0 ^3 \sin^2 {\alpha}.
    \end{align*}
Combine the above two inequalities, and the fact that $I_t \leq 2$ by its definition, we can simplify the above expression by writing
    \begin{align*}
        \partial_t R_t \big|_{t=r} \leq \frac{C_0}{\lambda R_0} e^{-\frac{(d-1)}{8}\left(r-30 t_1 - 10^5 \right)} + 10^{-8} \lambda^3 R_0 ^3 \sin^2 {\alpha},
    \end{align*}
where $C_0$ is a constant depending on $R_0$ and $\|f_0\|_{L^2(\S)}$ 
. Hence, for 
    \begin{align*}
        T= \frac{8}{d-1}\log \left( \frac{C_1}{R_0 ^4}  \right) + 30 t_1 + 10 ^5,
    \end{align*}
where $C_1$ is a constant depending on $R_0$ and $\|f_0\|_{L^2(\S)}$, we see that if $t_2 - t_1 \geq T$, then  $\partial_t R_t$ is upper-bounded by a positive function which is smaller than $ \frac{1}{8} (\sin ^4 {\alpha})  \lambda^3 R_0 ^3$ for any $t \geq t_1+T$, and we can repeat the above arguments until $t = +\infty$.
\end{proof}

\begin{proof}[Proof of \Cref{thm:exponential small outside positive cap}]

    Recall that $\alpha = \frac{\pi}{100}$ now by our assumption. Fix the $\lambda = 1-10^{-10} R_0 ^4 \geq 1-10^{-10} R_0 ^2$. We divide $\R_{\geq 0}$ into pieces $0=s_{-1} \leq t_0 < s_0 \leq t_1 < s_1 \leq t_2 < s_2 \cdots $, where for any $k \geq 0$
    \begin{align*}
        t_k \coloneq \inf \left\{ t \geq s_{k-1} \ \bigg| \  \partial_t R_t \geq \frac{1}{8} (\sin ^4 {\alpha})  \lambda^3 R_0 ^3 \right\} , \quad s_k \coloneq t_k +1 .
    \end{align*} 
    We first show that this construction must stop at some $k_{\ast}$-th step. By \Cref{lemma: all time lower bound R_t}, we have that $R_t \geq \lambda R_0$ for all $t \geq 0$. Actually, by \Cref{lemma:increasing R_t} and the proof of \Cref{lemma: all time lower bound R_t}, we see that for any $k >0$, 
        \begin{align*}
            R_{s_k} ^2 \geq R_{t_k}^2 + \frac{\sin ^4 {\alpha}}{100} \lambda^4 R_0 ^4 \geq \lambda^2 R_0 ^2 + \frac{\sin ^4 {\alpha}}{100} \lambda^4 R_0 ^4 \geq R_0 ^2 \left(1+\frac{12 R_0 ^2}{10^{10}} \right) .
        \end{align*}
    Hence, in \Cref{lemma:almost increasing R_t}, if we replace $R_0$ with $R_{s_{k-1}}$, we see that on $[s_{k-1} , t_k]$, $ \partial_t R_t \leq \frac{1}{8} (\sin ^4 {\alpha})  \lambda^3 R_0 ^3 \leq \frac{1}{8} (\sin ^4 {\alpha})  \lambda^3 R_{s_{k-1}} ^3$, and the assumption of \Cref{lemma:almost increasing R_t} is satisfied, and hence $R_{t_k} \geq \lambda R_{s_{k-1}}$. Hence, we can use \Cref{lemma:increasing R_t} and the proof of \Cref{lemma: all time lower bound R_t} again, and see that for any $k>0$, 
        \begin{align*}
            \begin{split}
                R_{s_k} ^2 &\geq R_{t_k}^2 + \frac{\sin ^4 {\alpha}}{100} \lambda^4 R_0 ^4 \geq \lambda^2 R_{s_{k-1}} ^2 + \frac{16}{10^{10}}\lambda^4 R_0 ^4
                \\  &\geq R_{s_{k-1}} ^2 + \lambda^2 -1 + \frac{16}{10^{10}}\lambda^4 R_0 ^4 \geq R_{s_{k-1}} ^2+ \frac{12}{10^{10}}R_0 ^{6}.
            \end{split}
        \end{align*}
Because $R_{s_k} ^2 \leq 1$, we must have that $k_{\ast} \leq 10^{9} R_0 ^{-6}$. Together with \Cref{thm:length of small dR_t intervals}, we see that for any $k \in \llbracket 0, k_{\ast}\rrbracket$, $t_k - s_{k-1} \leq C_u s_{k-1}+C_0$, with $C_u$ and $C_0$ obtained in \Cref{thm:length of small dR_t intervals}, and $t_{k_{\ast}+1}= +\infty$. 
Hence, we may set $\widetilde{T}_0=s_{k_{\ast}}$, which depends on $R_0$ and $\|f_0\|_{L^2(\S)}$. By \eqref{eqn:exponential decay outside positive cap} and its proof, the following inequality holds true
    \begin{equation}\label{eqn:R0_dependent_rate}
        \mu_t \left( \S \backslash S_{\alpha} ^+ (U_t) \right) \leq \widetilde{C}_0 e^{-(d-1)\widetilde{c}_1 R_0 t}, \quad \forall t \geq \widetilde{T}_0 \,,
    \end{equation}
    where $\widetilde{C}_0$ is a constant depending on $R_0$ and $\|f_0\|_{L^2(\S)}$, and $\widetilde{c}_1$ is a universal constant. If $R_0>\frac{1}{2}$, the desired form of result in \Cref{thm:exponential small outside positive cap} follows directly by identifying $T_0=\widetilde{T}_0$, $C_0=\widetilde{C}_0$ and $c_1=\widetilde{c}_1$. Otherwise, we define $\widetilde{T}_1=\widetilde{T}_0\vee\frac{\log(10\widetilde{C}_0)}{(d-1)\widetilde{c}_1 R_0}$ such that the mass outside the cap $S_{\alpha}^+(U_{\widetilde{T}_1})$ is small: $\mu_{\widetilde{T}_1}(\S \backslash S_{\alpha}^+(U_{\widetilde{T}_1})) \leq \frac{1}{10}$. At time $\widetilde{T}_1$, we estimate the lower bound for  $R_{\widetilde{T}_1}$ as follows:
    \begin{align}\label{eqn:large R0 time}
        \begin{split}
            R_{\widetilde{T}_1} &= \int_{\S} \langle y , U_{\widetilde{T}_1} \rangle \diff \mu_{\widetilde{T}_1}(y) 
            \\  
            &\geq \cos \alpha \ \mu_{\widetilde{T}_1} \left(S_{\alpha} ^+ (\widetilde{T}_1) \right) - \mu_{\widetilde{T}_1}\left(\mathbb{S}^{d-1} \backslash S_{\alpha} ^+ (\widetilde{T}_1)  \right)
            \\  &= \cos \alpha  - (1+\cos \alpha) \ \mu_{\widetilde{T}_1}\left(\mathbb{S}^{d-1} \backslash S_{\alpha} ^+ (\widetilde{T}_1) \right) 
            \\  &\geq \frac{9}{10} \cos\alpha - \frac{1}{10} \geq \frac{1}{2} .
        \end{split}
    \end{align}
    Thus, we can reset the starting time of the estimate \eqref{eqn:R0_dependent_rate} to $\widetilde{T}_1$. There exists $\widetilde{T}_2\geq 0$, depending on $\|f_{\widetilde{T}_1}\|_{L^2(\S)}$, and hence on $\|f_0\|_{L^2(\S)}$ and $R_0$ via \eqref{eqn:L2_norm_gronwall}, such that
    \begin{equation*}
        \mu_{t} \left( \S \backslash S_{\alpha} ^+ (U_{t}) \right) \leq \widetilde{C}_0 e^{-\frac{1}{2}(d-1)\widetilde{c}_1 (t-\widetilde{T}_1)}, \quad \forall t \geq \widetilde{T}_2 + \widetilde{T}_1 \,.
    \end{equation*}
    The desired result in \Cref{thm:exponential small outside positive cap} then follows by identifying $T_0=\widetilde{T}_2 + \widetilde{T}_1$, $C_0=\widetilde{C}_0 e^{\frac{1}{2}(d-1)\widetilde{c}_1 \widetilde{T}_1}$ and $c_1=\widetilde{c}_1/2$.
\end{proof}

\section{More accurate \texorpdfstring{$C_0,T_0$}{TEXT} in \Cref{thm:main thm}: Proof of \Cref{thm:main thm improved constants}}
The assumptions in this section are the same as in \Cref{thm:main thm}, that is, we have a family of $L^2(\S)$ probability densities, $\{f_t(x)\}_{t \geq 0}$ satisfying \eqref{eqn:modified gradient flow}, and $\epa \leq c_u R_0 ^6$. We notice that $C_0,T_0$ in \Cref{thm:main thm} (and \Cref{thm:exponential small outside positive cap}) comes from the proof for \Cref{thm:length of small dR_t intervals}. In particular, the term $T$ in \eqref{eqn:T dependence on t1} of \Cref{thm:length of small dR_t intervals} depends on $t_1$. Because of this dependence, when $t$ is in the interval where $\partial_t R_t \geq \frac{1}{8} (\sin ^4 {\alpha})  \lambda^3 R_0 ^3$, we lose control of the growth of $ f_{t} ^2\left(\S \backslash S_{\alpha} ^+ (t) \right)$. The following arguments are mainly for fixing this issue. For this purpose, we carefully investigate the characteristic flow associated with the dynamics \eqref{eqn:modified gradient flow}. Our analysis is inspired by the problems for the Kuramoto model considered in \cite{ha2020emergence, morales2022trend} where $d=2$ and $\beta=0$, but our more general dynamics \eqref{eqn:modified gradient flow} and the geometry of $\S$ make the arguments more involved than the circle case. 

We adopt the diffeomorphism  notation we used in proving \Cref{lemma:shrink from negative cap to positive cap}. That is, for any $t_1 \geq 0$, we define the diffeomorphisms $\{\phi_{t_1\to t}(x)\}_{t \geq t_1}$ on $\mathbb{S}^{d-1}$ by solving the ODE
    \begin{align*}
        \partial_t \phi_{t_1\to t}(x) = \deV_{t}(\phi_{t_1\to t}(x)), \text{ with }  \phi_{t_1\to t_1}(x) = x, \ \forall x \in \mathbb{S}^{d-1} .
    \end{align*}
After exploiting more properties of $\phi_{t_1\to t}$, we will be able to modify our \Cref{thm:length of small dR_t intervals}.

\begin{lemma}\label{lem:sliding norm}
    For any $t\geq t_1\geq 0$, and any measurable set $B \subseteq \S$, we have that
        \begin{align*}
            \frac{\de}{\de t} f_t ^2 \left( \phi_{t_1\to t}(B)\right)\leq 2(d-1)\cdot  f_t ^2 \left( \phi_{t_1\to t}(B)\right).
        \end{align*}
\end{lemma}
\begin{proof}
    For simplicity, we prove the case where $\phi_{t_1\to t}(B)$ has a smooth topological boundary $\partial \phi_{t_1\to t}(B)$ in $\S$. The general cases for $B$ can be done using the area formula (change of variables), and similar computations were used to prove Lemma A.1 in \cite{huang2024convergence}.

    When $\partial \phi_{t_1\to t}(B)$ is smooth, we notice that 
        \begin{align*}
            \begin{split}
                &\frac{\de}{\de t} f_t ^2 \left( \phi_{t_1\to t}(B)\right) = \frac{\de}{\de t} \int_{\phi_{t_1\to t}(B)} f_t ^2 (x) \diff x 
                \\  &= \int_{\partial \phi_{t_1\to t}(B)} f_t ^2 (x) \langle \nu(x) , \partial_t \phi_{t_1\to t} (\phi_{t_1\to t}^{-1}(x))\rangle \diff \mathcal{H}^{n-2}(x) + \int_{\phi_{t_1\to t}(B)} \partial_t (f_t ^2 (x) ) \diff x
                \\  &= \int_{ \phi_{t_1\to t}(B)} \dive \left(f_t ^2 (x) \deV_t(x)\right) \diff x + \int_{\phi_{t_1\to t}(B)} \partial_t (f_t ^2 (x) ) \diff x.
            \end{split}
        \end{align*}
    where $n(x)$ is the outer unit normal vector of $\partial \phi_{t_1\to t}(B)$ in $\S$, and $\mathcal{H}^{n-2}(x)$ is the Hausdorff measure. In the last line, we used the divergence theorem and $\partial_t \phi_{t_1\to t} (\phi_{t_1\to t}^{-1}(x)) = \deV_t(x)$. Because $f_t$ satisfies \eqref{eqn:modified gradient flow}, we have that 
        \begin{align*}
            \frac{\de}{\de t} f_t ^2 \left( \phi_{t_1\to t}(B)\right) = - \int_{ \phi_{t_1\to t}(B)} f_t ^2 (x) \dive \left(\deV_t(x) \right) \diff x.
        \end{align*}
    By \eqref{eqn:modified wss gradient}, we see that
        \begin{align*}
            \begin{split}
                -\dive\left(\deV_t(x) \right) &= \int_{\S} \left[ (d-1)\langle x , y \rangle  - \dive \left(W_t(x) \right) \right] f_t(y)\diff y
                \\  &\leq (d-1)(\langle x, M_t\rangle + \epa).
            \end{split}
        \end{align*}
    Combine the above inequalities and the fact that $\langle x, M_t\rangle \leq 1$, we get \Cref{lem:sliding norm}.
\end{proof}

In the following, we use $\Cx[S]$ to denote the geodesically convex hull of a set $S \subseteq \S$, that is, the intersection of all those closed geodesically convex subsets of $\S$ that contain $S$. The definition implies that $\Cx(S)$ is unique and closed. We first need the following geometric fact:
\begin{lemma}\label{lem:fact of geodesically convex sets}
    Let $S \subseteq \S$ be a closed subset. If $\inf_{x,y \in S} \langle x ,y \rangle > \frac{\sqrt{2}}{2}$, then, $\inf_{x,y \in S} \langle x ,y \rangle = \inf_{x,y \in \Cx [S]} \langle x ,y \rangle$.
\end{lemma}
\begin{proof}
    Because $S \subseteq \Cx[S]$, we have that $\inf_{x,y \in S} \langle x ,y \rangle \geq \inf_{x,y \in \Cx [S]} \langle x ,y \rangle$. Assume that $\inf_{x,y \in S} \langle x ,y \rangle > \inf_{x,y \in \Cx [S]} \langle x ,y \rangle$, and $\inf_{x,y \in \Cx [S]} \langle x ,y \rangle$ is achieved by $Z_1,Z_2 \in \Cx [S]$, and we use $\theta$ to denote the angle between $Z_1,Z_2$, that is, $\cos \theta = \langle Z_1, Z_2 \rangle$ and $\theta \in [0, \frac{\pi}{2})$. Notice that we can get $\theta < \frac{\pi}{2} $, because the maximal angle of points in $S$ does not exceed $\frac{\pi}{4}$. Assume that $Z_1 \notin S$, then we consider the spherical cap $ S_{\theta} ^+ (Z_2)$, where we used the definition \eqref{eqn:definition of positive cap}. Because $\cos \theta = \inf_{x,y \in \Cx [S]} \langle x ,y \rangle$, we have that $\Cx [S] \subseteq S_{\theta} ^+ (Z_2)$. Hence, $S \subseteq S_{\theta} ^+ (Z_2)$. We extend the geodesic from $Z_1$ to $Z_2$ further to a point $Z_3$, such that $\langle Z_1 ,Z_3 \rangle =0$. Recall the triangle inequality on sphere: for any $X_1,X_2,X_3$ such that $X_1,X_3 \in S_{\frac{\pi}{2}} ^+ (X_2)$, we have that $\theta_{21}+\theta_{23} \geq \theta_{13}$, where $\theta_{21}$ is the angle between $X_2$ and $X_1$, $\theta_{23}$ is the angle between $X_2$ and $X_3$, and $\theta_{13}$ is the angle between $X_1$ and $X_3$. Hence, $S_{\theta} ^+ (Z_2) \subseteq S_{\frac{\pi}{2}} ^+ (Z_3)$, and the boundaries of these two sets are only tangent at $Z_1$. Apparently, $S \subseteq \Cx [S] \subseteq S_{\theta} ^+ (Z_2) \subseteq S_{\frac{\pi}{2}} ^+ (Z_3)$. Because $Z_1 \notin S$, and the boundaries of $S_{\theta} ^+ (Z_2)$ and $S_{\frac{\pi}{2}} ^+ (Z_3)$ only intersect at $Z_1$, we see that the boundary of $S_{\frac{\pi}{2}} ^+ (Z_3)$ does not contain any point in $S$. Hence, we can take a very $\epsilon >0$, such that $S \subseteq S_{\frac{\pi}{2}-\epsilon} ^+ (Z_3)$. Because $S_{\frac{\pi}{2}-\epsilon} ^+ (Z_3)$ is also a closed geodesically convex set, by the definition of $\Cx[S]$, we must have that $\Cx[S] \subseteq S_{\frac{\pi}{2}-\epsilon} ^+ (Z_3)$. This is a contradiction, because $Z_1 \in \Cx[S]$ but $Z_1 \notin S_{\frac{\pi}{2}-\epsilon} ^+ (Z_3)$.
\end{proof}

\begin{lemma}\label{lem:Emergence of attractor}
    Fix a time $t_1 \geq 0$ and a closed subset $B \subseteq \S$ which is properly contained in a hemisphere of $\S$. Define 
        \begin{align*}
            D_t(B) \coloneq \inf_{x,y \in \Cx[\phi_{t_1\to t}(B)] } \langle x , y \rangle,
        \end{align*}
    and $\Gamma(B) \coloneq \mu_{t_1}(B) (1+D_{t_1}(B))-1$. If $\Gamma(B)>0$, $D_{t_1}(B) >0$, and if we have that $\epa^2 \leq \frac{1}{4} (1-D_{t_1}(B))  \Gamma(B) ^2$, then for any $t \geq t_1$, we have that $D_t(B) \geq D_{t_1}(B)$, and 
        \begin{align}\label{eqn:first in emergence of attractor}
            \inf_{x \in \Cx [\phi_{t_1\to t}(B)]} \langle x ,M_t \rangle \geq \mu_{t_1}(B)  \left( 1+D_t(B)\right)-1 \geq \Gamma(B), 
        \end{align}
    and
    \begin{align}\label{eqn:second in emergence of attractor}
            1- D_t(B) \leq \max \left \{ (1-D_{t_1}(B)) e^{\frac{-\Gamma(B)}{4}(t-t_1)} , \ \frac{4}{\Gamma(B)^2} \epa^2 \right\}.
        \end{align}
\end{lemma}
\begin{proof}
    The first inequality \eqref{eqn:first in emergence of attractor} basically follows from the fact that for any set $A\subset\S$, $\mu_t(\phi_{t_1\to t}(A))$ is a constant because $f_t$ satisfies \eqref{eqn:modified gradient flow} and $\phi_{t_1\to t}$ is its characteristic flow. More precisely,
        \begin{align*}
            \begin{split}
                &\inf_{x \in \Cx[\phi_{t_1\to t}(B)]} \langle x ,M_t \rangle = \inf_{x \in \Cx[\phi_{t_1\to t}(B)]} \int_{\S} \langle x, y  \rangle f_t(y) \diff y
                \\  &\geq  \inf_{x \in \Cx[\phi_{t_1\to t}(B)]} \int_{\Cx[\phi_{t_1\to t}(B)]} \langle x, y  \rangle f_t(y) \diff y - \int_{\S \backslash \Cx[\phi_{t_1,t}(B)] } f_t(y) \diff y
                \\  &\geq D_t(B) \mu_t\left(\Cx[\phi_{t_1\to t}(B)] \right) - (1-f_t\left(\Cx[\phi_{t_1\to t}(B)] \right))
                \\  &=f_t\left(\Cx[\phi_{t_1\to t}(B)] \right)  \left( 1+D_t(B)\right)-1 
                \\  & \geq \mu_t\left(\phi_{t_1\to t}(B) \right)  \left( 1+D_t(B)\right)-1 = \mu_{t_1}\left( B\right)  \left( 1+D_t(B)\right)-1.
            \end{split}
        \end{align*}

    To prove the second inequality \eqref{eqn:second in emergence of attractor}, we need to compute the derivatives of $D_t(B)$ in $t$. Let $\phi_{t_1\to t}(x),\phi_{t_1\to t}(y)$ be two points in $\Cx[\phi_{t_1\to t}(B)]$, we have that
        \begin{align}\label{eqn:derivative of two char lines}
            \begin{split}
                &\frac{\de}{\de t} \left \langle \phi_{t_1\to t}(x), \phi_{t_1\to t}(y) \right \rangle = \left \langle \deV_t(\phi_{t_1\to t}(x)), \phi_{t_1\to t}(y) \right \rangle + \left \langle  \phi_{t_1\to t}(x) , \deV_t(\phi_{t_1\to t}(y)) \right \rangle
                \\ &= \left \langle V_t(\phi_{t_1\to t}(x)), \phi_{t_1\to t}(y) \right \rangle + \left \langle  \phi_{t_1\to t}(x) , V_t(\phi_{t_1\to t}(y)) \right \rangle 
                \\  & \quad + \left \langle W_t(\phi_{t_1\to t}(x)), \phi_{t_1\to t}(y) \right \rangle + \left \langle  \phi_{t_1\to t}(x) , W_t(\phi_{t_1\to t}(y)) \right \rangle
                \\ &= \left \langle M_t, \proj_{\phi_{t_1\to t}(x)} [\phi_{t_1\to t}(y)] \right \rangle + \left \langle M_t, \proj_{\phi_{t_1\to t}(y)} [\phi_{t_1\to t}(x)] \right \rangle 
                \\  & \quad + \left \langle W_t(\phi_{t_1\to t}(x)), \proj_{\phi_{t_1\to t}(x)} [\phi_{t_1\to t}(y)] \right \rangle + \left \langle  \proj_{\phi_{t_1\to t}(y)} [\phi_{t_1\to t}(x)], W_t(\phi_{t_1\to t}(y)) \right \rangle.
            \end{split}
        \end{align}
If we let $\theta \in [0,\frac{\pi}{2}]$ such that $\cos (2\theta) = \langle \phi_{t_1\to t}(x) , \phi_{t_1\to t}(y) \rangle$, then we see that $\|\proj_{\phi_{t_1\to t}(x)} [\phi_{t_1\to t}(y)]\|_2 = \sin (2\theta)$. Also, there is a $Z \in \Cx[\phi_{t_1\to t}(B)] \subseteq \S$ such that $\proj_{\phi_{t_1\to t}(x)} [\phi_{t_1\to t}(y)] + \proj_{\phi_{t_1\to t}(y)} [\phi_{t_1\to t}(x)] = 2 \sin(\theta) \sin(2\theta) \cdot Z $. This $Z$ is actually the middle point on the shortest great circle connecting $\phi_{t_1\to t}(x),\phi_{t_1\to t}(y)$. Hence,
    \begin{align*}
        \frac{\de}{\de t} \left \langle \phi_{t_1\to t}(x), \phi_{t_1\to t}(y) \right \rangle \geq   \inf_{z \in \Cx[\phi_{t_1\to t}(B)]} \langle z ,M_t \rangle - 2\epa \sin(2\theta),
    \end{align*}
where we also used \Cref{lem:Kuramoto perturbation}. Because $\phi_{t_1\to t}(x),\phi_{t_1\to t}(y)$ were chosen arbitrarily, by writing $\sin (2 \theta) = \sqrt{ 1- \cos^2 (2\theta)}$ and $\sin (\theta) = \sqrt{\frac{1-\cos (2\theta)}{2}}$, we obtain that
    \begin{align}\label{eqn:third in emergence of attractor}
        \begin{split}
            &\frac{\de}{\de t} D_t(B) \geq 2\sqrt{1- D_t(B) ^2 } \left( \sqrt{\frac{1-D_t(B)}{2}} \inf_{z \in \Cx[\phi_{t_1\to t}(B)]} \langle z ,M_t \rangle - \epa \right)
            \\  &\geq 2\sqrt{1- D_t(B) ^2 } \left( \sqrt{\frac{1-D_t(B)}{2}} \left[ \mu_{t_1}\left( B\right)  \left( 1+D_t(B)\right)-1 \right] - \epa \right),
        \end{split}
    \end{align}
where in the last step, we used the first inequality \eqref{eqn:first in emergence of attractor} which we just proved. Assume that $[t_1,t_2]$ is the maximal interval such that \eqref{eqn:second in emergence of attractor} holds true for any $t \in [t_1,t_2]$, then we want to show that $t_2 = +\infty$. First, because we have \eqref{eqn:second in emergence of attractor} on $[t_1,t_2]$, we obtain that $D_t(B) \geq D_{t_1}(B)>0$ for any $t \in [t_1,t_2]$, where we also used the assumption that $\epa^2 \leq \frac{1}{4} (1-D_{t_1}(B))  \Gamma(B) ^2$. If $t_2$ is a finite number, then we let $t=t_2$ in \eqref{eqn:third in emergence of attractor}, and obtain that
    \begin{align*}
        \begin{split}
            &\frac{\de}{\de t} \bigg|_{t=t_2} (1- D_t(B)) \leq -2\sqrt{1- D_{t_2}(B) ^2 } \left( \sqrt{\frac{1-D_{t_2}(B)}{2}} \left[ \mu_{t_1}\left( B\right)  \left( 1+D_{t_2}(B)\right)-1 \right] - \epa \right)
            \\  &\leq -2\sqrt{1- D_{t_2}(B) ^2 } \left( \sqrt{\frac{1-D_{t_2}(B)}{2}} \left[ \mu_{t_1}\left( B\right)  \left( 1+D_{t_1}(B)\right)-1 \right] - \epa \right)
            \\  &= -2\sqrt{1- D_{t_2}(B) ^2 } \left( \sqrt{\frac{1-D_{t_2}(B)}{2}} \Gamma(B)  - \epa \right).
        \end{split}
    \end{align*}
Now, by the assumption of $t_2$, we also have that
    \begin{align*}
        1- D_{t_2}(B) = \max \left \{ (1-D_{t_1}(B)) e^{\frac{-\Gamma(B)}{4}(t_2-t_1)} , \ \frac{4}{\Gamma(B)^2} \epa^2 \right\} \geq \frac{4}{\Gamma(B)^2} \epa^2.
    \end{align*}
Hence,
    \begin{align*}
        \begin{split}
            &\frac{\de}{\de t} \bigg|_{t=t_2} (1- D_t(B)) \leq -\sqrt{1+ D_{t_2}(B) } \cdot [1-D_{t_2}(B) ] \Gamma(B)   \left(\sqrt{2}-1\right)
            \\  &\leq -\frac{2}{5} [1-D_{t_2}(B) ] \Gamma(B).
        \end{split}
    \end{align*}
Because $D_{t_2}(B) <1$ as $B$ is an open set in $\S$, by the continuity of the solution, there is a small time interval $[t_2,t_2+\delta]$ for some $\delta>0$, such that for $t \in [t_2 ,t_2+\delta]$, we have that
    \begin{align*}
        \frac{\de}{\de t} (1- D_t(B)) < -\frac{1}{4} [1-D_t(B) ] \Gamma(B).
    \end{align*}
Hence, for any $t \in [t_2,t_2 + \delta]$, \eqref{eqn:second in emergence of attractor} also holds true, because 
    \begin{align*}
        \begin{split}
            &1- D_t(B) < (1- D_{t_2}(B)) e^{-\frac{\Gamma(B)}{4}(t-t_2)} 
            \\  &\leq 
        \begin{cases}
            (1-D_{t_1}(B)) e^{\frac{-\Gamma(B)}{4}(t-t_1)}, \text{ if } 1- D_{t_2}(B) = (1-D_{t_1}(B)) e^{\frac{-\Gamma(B)}{4}(t_2-t_1)} , 
            \\  \frac{4}{\Gamma(B)^2} \epa^2 , \text{ if } 1- D_{t_2}(B) = \frac{4}{\Gamma(B)^2} \epa^2 ,
        \end{cases}
        \\  &\leq \max \left \{ (1-D_{t_1}(B)) e^{\frac{-\Gamma(B)}{4}(t-t_1)} , \ \frac{4}{\Gamma(B)^2} \epa^2 \right\},
        \end{split}
    \end{align*}
which contradicts to the assumption that $[t_1,t_2]$ is the maximal interval on which \eqref{eqn:second in emergence of attractor} holds true for any $t \in [t_1,t_2]$.
\end{proof}

\begin{lemma}\label{lem:first attractor}
    There is a $T_{\ast} \leq 10^4 R_0 ^{-3}$, such that if we let $\alpha_{\ast} \in (0,\frac{\pi}{2})$ satisfy $\sin^2 (\alpha_{\ast}) = 10^{-2} R_0$, then the set $B_{\ast} \coloneq S_{\alpha_{\ast}} ^+ (T_{\ast})$ satisfies the assumptions for the set $B$ for $t_1 = T_{\ast}$ in \Cref{lem:Emergence of attractor}. Furthermore, if $\epa \leq 10^{-2} R_0 ^2$, then for any $t \geq T_{\ast}$,
        \begin{align*}
            \mu_t \left( \Cx [\phi_{T_{\ast}\to t}(B_{\ast})]\right) \geq \mu_{{T_\ast}}\left(B_{\ast} \right) \geq \frac{1}{2} \left(1+\frac{9}{10 }R_0\right),
        \end{align*}
    and
        \begin{align}\label{eqn:lock the attractor together}
            \inf_{x,y \in \Cx[\phi_{T_{\ast}\to t}(B_{\ast})] } \langle x , y \rangle = D_t(B_{\ast}) \geq D_{T_{\ast}}(B_{\ast}) = 1- \frac{1}{50}R_0,
        \end{align}
    and
        \begin{align}\label{eqn:lock the attractor in hemisphere}
            \inf_{x \in \Cx [\phi_{T_{\ast}\to t}(B_{\ast})]} \langle x ,M_t \rangle \geq \mu_{T_{\ast}}(B_{\ast})  \left( 1+D_t(B_{\ast})\right)-1 \geq \frac{4}{5}R_0 .
        \end{align}        
\end{lemma}
\begin{proof}
    We define
        \begin{align*}
            T_{\ast} \coloneq \inf  \left \{ t \geq 0 \ | \ \partial_t R_t \leq 10^{-4}R_0 ^3 \right\}.
        \end{align*}
    Because, $1 \geq R_{T_{\ast}} - R_0 \geq T_{\ast} 10^{-4} R_0 ^3$, we have that $T_{\ast} \leq 10^4 R_0 ^{-3}$. Also, by the definition of $T_{\ast}$, we have that $R_t$ is strictly increasing on $[0,T_{\ast}]$, in particular, $R_t \geq R_0$ for $t \in [0,T_{\ast}]$.

    In order to verify the assumptions in \Cref{lem:Emergence of attractor} for $B_{\ast} = S_{\alpha_{\ast}} ^+ (T_{\ast})$, we need to get the corresponding $D_{T_{\ast}}(B_{\ast})$ and $f_{T_{\ast}}(B_{\ast})$. First, it is easy to see that
        \begin{align*}
            D_{T_{\ast}}(B_{\ast}) = \cos (2\alpha_{\ast}) = 1-2 \sin^2 (\alpha_{\ast}) = 1- \frac{1}{50}R_0.
        \end{align*}
    Then, we are going to estimate $\mu_{T_{\ast}}(B_{\ast})$. By the same reason as in proving the first inequality in $\eqref{eqn:upper bound negative cap 0}$, that is, divide the integral $R_{T_{\ast}} = \int_{\S} \langle y , U_{T_\ast} \rangle f_{T_\ast}(y) \diff y$ into integrals over $S_{\alpha_{\ast}} ^+ ({T_\ast})$, $\mathbb{S}^{d-1} \backslash (S_{\alpha_{\ast}} ^+ ({T_\ast}) \cup S_{\alpha_{\ast}} ^- ({T_\ast}))$, and $S_{\alpha_{\ast}} ^- ({T_\ast})$, we have that
        \begin{align*}
            \begin{split}
                R_0 \leq R_{T_{\ast}} &\leq \mu_{{T_\ast}}\left(S_{\alpha_{\ast}} ^+ ({T_\ast}) \right) + \cos ({\alpha_{\ast}}) \mu_{{T_\ast}}\left(\mathbb{S}^{d-1} \backslash (S_{\alpha_{\ast}} ^+ ({T_\ast}) \cup S_{\alpha_{\ast}} ^- ({T_\ast})) \right) 
                \\  & \quad  - \cos ({\alpha_{\ast}} ) \mu_{{T_\ast}}\left(S_{\alpha_{\ast}} ^- ({T_\ast}) \right)
                \\  &= (1+ \cos (\alpha_{\ast}))  \mu_{{T_\ast}}\left(S_{\alpha_{\ast}} ^+ ({T_\ast}) \right) + 2\cos ({\alpha_{\ast}}) \mu_{{T_\ast}}\left(\mathbb{S}^{d-1} \backslash (S_{\alpha_{\ast}} ^+ ({T_\ast}) \cup S_{\alpha_{\ast}} ^- ({T_\ast})) \right)
                \\  & \quad  - \cos ({\alpha_{\ast}} )
                \\  & \leq 2\mu_{{T_\ast}}\left(S_{\alpha_{\ast}} ^+ ({T_\ast}) \right) + 2 \mu_{{T_\ast}}\left(\mathbb{S}^{d-1} \backslash (S_{\alpha_{\ast}} ^+ ({T_\ast}) \cup S_{\alpha_{\ast}} ^- ({T_\ast})) \right) - \cos ({\alpha_{\ast}} ).
            \end{split}
        \end{align*}
Next, by the same reason as in proving \eqref{eqn:upper bound equatorial 0}, we have that
    \begin{align*}
        \begin{split}
            &\mu_{{T_\ast}}\left(\mathbb{S}^{d-1} \backslash (S_{\alpha_{\ast}} ^+ ({T_\ast}) \cup S_{\alpha_{\ast}} ^- ({T_\ast})) \right) \leq \frac{\partial_t R_t ^2 |_{t={T_\ast}} + \epa^2 }{ R_{T_\ast} ^2 (\sin^2 ({\alpha_{\ast}} ) )} 
            \\  &= \frac{2\partial_t R_t  |_{t={T_\ast}}}{ R_{T_\ast} (\sin^2 ({\alpha_{\ast}} ) )} +  \frac{ \epa^2 }{ R_{T_\ast} ^2 (\sin^2 ({\alpha_{\ast}} ) )} \leq \frac{2 \cdot 10^{-4}R_0 ^3}{R_0 \cdot 10^{-2} R_0} + \frac{10^{-4} R_0 ^2}{10^{-2} R_0} = 3 \cdot 10^{-2} R_0.
        \end{split}
    \end{align*}
Combine the above two inequalities, and $\cos ({\alpha_{\ast}} ) \geq \cos ^2 ({\alpha_{\ast}} ) = 1-  \sin ^2 ({\alpha_{\ast}} ) = 1- 10^{-2} R_0$, we see that
    \begin{align*}
        \mu_{{T_\ast}}\left(S_{\alpha_{\ast}} ^+ ({T_\ast}) \right) \geq \frac{1}{2} \left(1+R_0 - 7 \cdot  10^{-2} R_0\right) \geq \frac{1}{2} \left(1+\frac{9}{10 }R_0\right).
    \end{align*}
Hence,
    \begin{align*}
        \begin{split}
            &\Gamma \left( S_{\alpha_{\ast}} ^+ (T_{\ast}) \right) = \mu_{T_{\ast}}\left( S_{\alpha_{\ast}} ^+ (T_{\ast}) \right)(1+D_{T_{\ast}}(S_{\alpha_{\ast}} ^+ (T_{\ast}))-1
            \\  &\geq \frac{1}{2} \left(1+\frac{9}{10 }R_0\right) \left( 2- \frac{1}{50}R_0 \right)-1
            \\  &= \frac{9}{10} R_0 - \frac{1}{100} R_0 - \frac{9}{1000} R_0^2 \geq \frac{4}{5}R_0.
        \end{split}
    \end{align*}
Then, our assumption on $\epa$ in \Cref{lem:first attractor} also implies the assumption on $\epa$ in \Cref{lem:Emergence of attractor}, because 
    \begin{align*}
        \frac{1}{4} \left( 1-D_{T_{\ast}}(S_{\alpha_{\ast}} ^+ (T_{\ast}) \right)  \Gamma \left( S_{\alpha_{\ast}} ^+ (T_{\ast}) \right) ^2 \geq \frac{16}{5000} R_0^3 \geq 10^{-4} R_0^4 \geq \epa^2 .
    \end{align*}
By our \Cref{lem:Emergence of attractor}, we can finish the proof for \Cref{lem:first attractor}.
\end{proof}

Before we proceed, we need to give some further definitions. Let $T_{\ast}$, $\alpha_{\ast}$ be the time and the angle obtained in \Cref{lem:first attractor}, and recall that $B_{\ast} \coloneq S_{\alpha_{\ast}} ^+ (T_{\ast})$. Now we define the set $B_t \coloneq \phi_{T_{\ast}\to t}(B_{\ast})$, and its $\frac{R_0 ^2}{10^4}$-neighborhood set $\deB_t$
    \begin{align*}
        \deB_t \coloneq \left\{ x \in \S \ \bigg| \ \sup_{y \in B_t} \langle x, y \rangle \geq 1- \frac{R_0 ^2}{10^4} \right\}.
    \end{align*}

The following lemma is a further step after \Cref{lemma:shrink from negative cap to positive cap}.

\begin{lemma}\label{lemma:shrink from negative cap to the attractor}
    Fix a constant $\lambda \in (2/3,1)$ and an angle $\alpha \in [\pi/100, \pi/2)$. Let $T_{\ast}$, $\alpha_{\ast}$ be the time and the angle obtained in \Cref{lem:first attractor}. 
If there is a time window $[t_1,t_2]$, such that when $t \in [t_1,t_2]$,
    \begin{align*}
        R_t \geq \lambda R_0, \quad \partial_t R_t \leq \frac{1}{8} (\sin ^4 {\alpha})  \lambda^3 R_0 ^3,
    \end{align*}
and if
    \begin{align*}
        \epa \leq \frac{1}{10^4} \lambda^2 R_0 ^2 \sin \alpha \cos \alpha,
    \end{align*}
then, for any $t_3 \to t_4 \in [t_1, t_2]$, $t_3 \leq t_4$, and any $x,y \in \S$ such that 
    \begin{align*}
        \phi_{t_3\to t_4}(x) \in S_{\alpha} ^+ (t_4), \quad y \in B_{\ast}, \text{ and } \left\langle \phi_{t_3\to t_4}(x) , \phi_{T_{\ast}\to t_4}(y) \right\rangle \leq 1- \frac{R_0 ^2}{10^4}, 
    \end{align*}
we have that
    \begin{align}\label{eqn:derivative shrink to attractor}
        \frac{\de}{\de t} \bigg|_{t=t_4} \left\langle \phi_{t_3\to t}(x) , \phi_{T_{\ast}\to t}(y) \right\rangle \geq  \frac{\lambda R_0 ^{\frac{3}{2}} (\cos^2 \alpha )}{4} \left(1- \left\langle \phi_{t_3\to t}(x) , \phi_{T_{\ast}\to t}(y) \right\rangle \right).
    \end{align}
As a corollary, if we define,
    \begin{align}\label{eqn:small time 2}
        \widetilde\delta=\widetilde\delta(\lambda,R_0, \alpha) \coloneq \frac{4\log{(10^4 R_0 ^{-2})}}{\lambda R_0 ^{\frac{3}{2}} (\cos^2 \alpha )},
    \end{align}
then if $t_2 - t_1 \geq \widetilde \delta$, we have that
    \begin{align*}
        \phi_{t_1\to t_2} \left(S_{\alpha} ^+ (t_1)\right) \subseteq  \deB_{t_2}.
    \end{align*}
\end{lemma}
\begin{proof}
    Take any $x \in \S$ and any $y \in B_{\ast}$ such that $\left\langle \phi_{t_3\to t_4}(x) , \phi_{T_{\ast}\to t_4}(y) \right\rangle \leq 1- \frac{R_0 ^2}{10^4}$ and $\phi_{t_3\to t_4}(x) \in S_{\alpha} ^+ (t_4)$. Notice that $\phi_{t_3\to t_4}(x)$ and $\phi_{T_{\ast}\to t_4}(y)$ are in the same hemisphere, because \eqref{eqn:lock the attractor in hemisphere} means that $\phi_{T_{\ast}\to t_4}(y) \in S_{\frac{\pi}{2}} ^+ (t_4)$. By the same computation as \eqref{eqn:derivative of two char lines}, we see that
         \begin{align*}
            \begin{split}
                &\frac{\de}{\de t} \bigg|_{t=t_4} \left \langle \phi_{t_3\to t}(x), \phi_{T_{\ast}\to t}(y) \right \rangle = 2 \sin(\theta) \sin(2\theta) \langle Z ,M_t \rangle
                \\  & \quad + \left \langle W_{t_4}(\phi_{t_3\to t_4}(x)), \proj_{\phi_{t_3\to t_4}(x)} [\phi_{T_{\ast}\to t_4}(y)] \right \rangle + \left \langle  \proj_{\phi_{T_{\ast}\to t_4}(y)} [\phi_{t_3\to t_4}(x)], W_t(\phi_{T_{\ast}\to t_4}(y)) \right \rangle
                \\  &\geq 2 \sin(\theta) \sin(2\theta) \langle Z ,M_t \rangle - 2\epa \sin(2\theta) ,
            \end{split}
        \end{align*}
    where $\theta \in [0,\frac{\pi}{2}]$ such that $\cos (2\theta) = \langle \phi_{t_3\to t_4}(x) ,\phi_{T_{\ast}\to t_4}(y) \rangle$, and $Z \in \S$ is the middle point on the shortest great circle connecting $\phi_{t_3\to t_4}(x),\phi_{T_{\ast}\to t_4}(y)$. Notice that because \eqref{eqn:lock the attractor in hemisphere} implies $\langle \phi_{T_{\ast}\to t_4}(y), M_t \rangle >0$, we have that
        \begin{align*}
            \langle Z ,M_t \rangle = \frac{\langle \phi_{t_3\to t_4}(x) +\phi_{T_{\ast}\to t_4}(y), M_t \rangle}{\|\phi_{t_3\to t_4}(x) +\phi_{T_{\ast}\to t_4}(y)\|_2} \geq \frac{\langle \phi_{t_3\to t_4}(x), M_t \rangle}{2} \geq \frac{\cos \alpha}{2}R_t \geq\frac{\cos \alpha}{2}\lambda R_0 .
        \end{align*}
    Also, we can write $\sin (2 \theta) = \sqrt{ 1- \cos^2 (2\theta)}$ and $\sin (\theta) = \sqrt{\frac{1-\cos (2\theta)}{2}}$. By the assumption, $\cos(2 \theta)=\left\langle \phi_{t_3\to t_4}(x) , \phi_{T_{\ast}\to t_4}(y) \right\rangle \leq 1- \frac{R_0 ^2}{10^4}$, which implies that
        \begin{align*}
            \sin (\theta) \langle Z ,M_t \rangle \geq \sqrt{\frac{R_0 ^2}{2 \cdot 10^4}} \frac{\cos \alpha}{2} \lambda R_0 \geq \frac{\lambda R_0 ^2 \cos \alpha}{300} \geq 10 \epa .
        \end{align*}
    So, 
        \begin{align*}
            \begin{split}
                \frac{\de}{\de t} \bigg|_{t=t_4} \left \langle \phi_{t_3\to t}(x), \phi_{T_{\ast}\to t}(y) \right \rangle &\geq \frac{\lambda R_0 \cos \alpha }{2} \sin{(\theta)}\sin{(2\theta)}
                \\  &= \frac{\lambda R_0 
 \cos \alpha}{2 \sqrt{2}} (1-\cos (2\theta)) \sqrt{1+\cos(2\theta)}.
            \end{split}
        \end{align*}
    Because we cannot rule out the case when $\cos(2\theta) < 0$, we need to get a  lower bound for $1+\cos(2\theta)$. \eqref{eqn:lock the attractor in hemisphere} implies $\langle \phi_{T_{\ast}\to t_4}(y), M_{t_4} \rangle >\frac{4}{5} R_0$. Because $M_{t_4} = R_{t_4} U_{t_4}$ and $R_{t_4} \leq 1$, we see that $\langle \phi_{T_{\ast}\to t_4}(y), U_{t_4} \rangle >\frac{4}{5} R_0$. Because $\phi_{t_3\to t_4}(x) \in S_{\alpha} ^+ (t_4) $, we have that $\langle \phi_{t_3\to t_4}(x), U_{t_4} \rangle \geq \cos \alpha$.  We use the following fact: for any $Z_1,Z_2,Z_3 \in \S$,
        \begin{align}\label{eqn:spherical triangle ineql}
            \begin{split}
                \langle Z_1 , Z_2 \rangle  &= \langle Z_1 , Z_3 \rangle  \langle Z_2 , Z_3  \rangle + \langle \proj_{Z_3} [Z_1] , \proj_{Z_3} [Z_2] \rangle  
                \\  &\geq \langle Z_1 , Z_3 \rangle  \langle Z_2 , Z_3  \rangle - \| \proj_{Z_3} [Z_1] \|_2 \| \proj_{Z_3} [Z_2]\|_2. 
            \end{split}
        \end{align}
    Hence, 
        \begin{align*}
            \begin{split}
                \cos (2\theta) &= \langle \phi_{t_3\to t_4}(x) ,\phi_{T_{\ast}\to t_4}(y) \rangle 
                \\  &\geq \langle \phi_{T_{\ast}\to t_4}(y), U_{t_4} \rangle\langle \phi_{t_3\to t_4}(x), U_{t_4} \rangle -1 \geq\frac{4}{5}R_0 \cos{\alpha} -1 .
            \end{split}
        \end{align*}
    Combine the above arguments, we obtain \eqref{eqn:derivative shrink to attractor}:
        \begin{align*}
            \frac{\de}{\de t} \bigg|_{t=t_4} \left \langle \phi_{t_3\to t}(x), \phi_{T_{\ast},t}(y) \right \rangle \geq  \frac{\lambda R_0 ^{\frac{3}{2}} (\cos^2 \alpha )}{4} (1-\cos (2\theta)) .
        \end{align*}

    Next, we show that if $t_2 - t_1 \geq \widetilde \delta$ for the $\widetilde \delta$ defined in \eqref{eqn:small time 2}, we have $\phi_{t_1\to t_2} \left(S_{\alpha} ^+ (t_1)\right) \subseteq  \deB_{t_2}$. Because the assumptions on $R_t$, $\partial_t R_t$, and $\epa$ in \Cref{lemma:shrink from negative cap to positive cap} are also satisfied here,  by \eqref{eqn:derivative shrink to positive cap} in \Cref{lemma:shrink from negative cap to positive cap}, we first know that $\phi_{t_1\to t} \left(S_{\alpha} ^+ (t_1)\right) \subseteq S_{\alpha} ^+ (t)$ for any time $t \in [t_1,t_2]$, because \eqref{eqn:derivative shrink to positive cap} means that for points already in $S_{\alpha} ^+ (t_1)$, those points along the characteristic flow, that is, $\phi_{t_1\to t}$, cannot escape the cap $S_{\alpha} ^+ (t)$ for any time $t \in [t_1,t_2]$. By \eqref{eqn:derivative shrink to attractor}, we have that for any $x \in S_{\alpha} ^+ (t_1)$, $y \in B_{\ast}$, 
        \begin{align*}
            \frac{\de}{\de t} \left( 1-\left\langle \phi_{t_1 \to t}(x) , \phi_{T_{\ast}\to t}(y) \right\rangle \right) \leq  -\frac{\lambda R_0 ^{\frac{3}{2}} (\cos^2 \alpha )}{4} \left( 1-\left\langle \phi_{t_1\to t}(x) , \phi_{T_{\ast}\to t}(y) \right\rangle \right),
        \end{align*}
    as long as $1-\left\langle \phi_{t_1\to t}(x) , \phi_{T_{\ast}\to t}(y) \right\rangle \geq \frac{R_0 ^2}{10^4}$. Hence, after at most $\widetilde \delta$ time, we have that $1-\left\langle \phi_{t_1\to t}(x) , \phi_{T_{\ast}\to t}(y) \right\rangle \leq \frac{R_0 ^2}{10^4}$, which implies that for any $t \geq t_1 + \widetilde \delta$, $\phi_{t_1\to t_2} \left(S_{\alpha} ^+ (t_1)\right)$ is contained in $\deB_t$, the $\frac{R_0 ^2}{10^4}$-neighborhood of $B_t$.
\end{proof}

\begin{lemma}\label{lem:exponentially decay out attractor}
    Fix a constant $\lambda \in (2/3,1)$ and an angle $\alpha \in [\pi/100, \pi/2)$. Let $T_{\ast}$, $\alpha_{\ast}$ be the time and the angle obtained in \Cref{lem:first attractor}. 
Assume that there is a time window $[t_1,t_2]$, such that when $t \in [t_1,t_2]$,
    \begin{align*}
        R_t \geq \lambda R_0, \quad \partial_t R_t \leq \frac{1}{8} (\sin ^4 {\alpha})  \lambda^3 R_0 ^3.
    \end{align*}
If
    \begin{align*}
        \epa \leq \frac{1}{10^4} \lambda^2 R_0 ^2 \sin \alpha \cos \alpha,
    \end{align*}
and if $t_2-t_1 \geq \delta+\widetilde \delta$ for $\delta$ defined in \eqref{eqn:small time} in \Cref{lemma:shrink from negative cap to positive cap}, and $\widetilde \delta$ defined in \eqref{eqn:small time 2} in \Cref{lemma:shrink from negative cap to the attractor}, we have that
    \begin{align*}
        f_{t_2} ^2 \left( \S \backslash \deB_{t_2} \right) \leq f_{t_1} ^2 \left( S_{\alpha} ^- (t_1) \right) \cdot e^{3(d-1) (\widetilde \delta + \delta)} \cdot e^{-\frac{(d-1) \lambda R_0\cos{\alpha}}{2} (t_2-t_1)}  .
    \end{align*}
\end{lemma}
\begin{proof}
    By \Cref{lemma:shrink from negative cap to the attractor}, we have that $\phi_{t_2-\widetilde \delta\to t_2} \left(S_{\alpha} ^+ (t_2-\widetilde \delta)\right) \subseteq  \deB_{t_2}$. Hence, because $\phi_{t_2-\widetilde \delta\to t_2}$ is a diffeomorphism on $\S$, we see that
        \begin{align*}
            \begin{split}
                &f_{t_2} ^2 \left( \S \backslash \deB_{t_2} \right) \leq f_{t_2} ^2 \left( \S \backslash \phi_{t_2-\widetilde \delta\to t_2} \left(S_{\alpha} ^+ (t_2-\widetilde \delta)\right)  \right) = f_{t_2} ^2 \left(  \phi_{t_2-\widetilde \delta\to t_2} \left( \S \backslash S_{\alpha} ^+ (t_2-\widetilde \delta)\right)  \right)
                \\  &\leq e^{2(d-1) \widetilde \delta} \cdot f_{t_2 - \widetilde \delta} ^2 \left(   \S \backslash S_{\alpha} ^+ (t_2-\widetilde \delta)\right),
            \end{split}
        \end{align*}
    where in the second inequality, we used \Cref{lem:sliding norm}. Next, by \Cref{lemma:shrink from negative cap to positive cap}, we have that $\S \backslash S_{\alpha} ^+ (t_2 - \widetilde \delta) \subseteq  \phi_{t_2-\widetilde \delta - \delta\to t_2-\widetilde \delta} \left( S_{\alpha} ^- (t_2-\widetilde \delta - \delta) \right)$. Using \Cref{lem:sliding norm} again, we have that
        \begin{align*}
            \begin{split}
                &f_{t_2} ^2 \left( \S \backslash \deB_{t_2} \right) \leq e^{2(d-1) \widetilde \delta} \cdot f_{t_2 - \widetilde \delta} ^2 \left(   \phi_{t_2-\widetilde \delta - \delta\to t_2-\widetilde \delta} \left( S_{\alpha} ^- (t_2-\widetilde \delta - \delta) \right) \right)
                \\  &\leq e^{2(d-1) (\widetilde \delta + \delta)} \cdot f_{t_2 - \widetilde \delta -\delta} ^2 \left(    S_{\alpha} ^- (t_2-\widetilde \delta - \delta) \right).
            \end{split}
        \end{align*}
    By \Cref{lem:exponential decay negative cap square}, we finally obtain that
        \begin{align*}
            f_{t_2} ^2 \left( \S \backslash \deB_{t_2} \right) \leq f_{t_1} ^2 \left( S_{\alpha} ^- (t_1) \right) \cdot e^{2(d-1) (\widetilde \delta + \delta)} \cdot e^{-\frac{(d-1) \lambda R_0\cos{\alpha}}{2} (t_2-t_1-\delta-\widetilde \delta)} .
        \end{align*}
\end{proof}

Before we proceed, we need another auxiliary lemma similar to \Cref{lem:first attractor}.

\begin{lemma}\label{lem:moving attractor}
    Let $T_{\ast}$ be the time obtained in \Cref{lem:first attractor} and take two times $t_1,t$ such that $T_{\ast} \leq t_1 \leq t$. If $\epa \leq 10^{-3} R_0 ^2$, then
        \begin{align*}
            \mu_t\left(\phi_{t_1\to t}\left(\deB_{t_1}  \right)\right) \geq \frac{1}{2} \left(1+\frac{9}{10 }R_0\right),
        \end{align*}
    and
        \begin{align*}
            \inf_{x,y \in \Cx\left[\phi_{t_1\to t}\left(\deB_{t_1}  \right)\right] } \langle x , y \rangle  \geq D_{t_1} ( \deB_{t_1} )= 1- \frac{1}{10}R_0,
        \end{align*}
    and
        \begin{align}\label{eqn:attractor nbhd in positive cap}
            \inf_{x \in \Cx\left[\phi_{t_1\to t}\left(\deB_{t_1}  \right)\right] } \langle x ,M_t \rangle \geq \mu_{t_1}\left(\deB_{t_1} \right) \left(1+D_{t_1} ( \deB_{t_1} ) \right)-1 \geq \frac{4}{5}R_0 .
        \end{align}
\end{lemma}
\begin{proof}
    The proof is similar to the proof for \Cref{lem:first attractor}. We need to verify the assumptions in \Cref{lem:Emergence of attractor} for the set $\deB_{t_1}$, for which we need to get the corresponding $D_{t_1} ( \deB_{t_1} )$ and $f_{t_1}(\deB_{t_1})$. First, take any  $x,y \in \deB_{t_1}$, then by the definition of $\deB_{t_1}$, there are corresponding $x',y' \in B_{t_1} = \phi_{T_{\ast},t_1} (B_{\ast})$, such that $\langle x, x' \rangle \geq 1- \frac{R_0 ^2}{10^4}$ and $\langle y, y' \rangle \geq 1- \frac{R_0 ^2}{10^4}$. Using the inequality  \eqref{eqn:spherical triangle ineql}, we have that 
        \begin{align*}
            \begin{split}
                 \langle x ,y \rangle &\geq \langle x ,y' \rangle \langle y ,y' \rangle - \sqrt{1-\langle y ,y' \rangle^2} 
            \\  &\geq \langle x' ,y' \rangle \langle x ,x' \rangle \langle y ,y' \rangle - \sqrt{1-\langle y ,y' \rangle^2} - \sqrt{1-\langle x ,x' \rangle^2}.
            \end{split}
        \end{align*}
    By \eqref{eqn:lock the attractor together} in \Cref{lem:first attractor}, we have that $\langle x' ,y' \rangle \geq 1-\frac{R_0}{50}$. Hence, we have that
        \begin{align*}
            \begin{split}
                \langle x ,y \rangle &\geq \left(1- \frac{R_0}{50} \right)\left(1- \frac{R_0 ^2}{10^4}\right)\left(1- \frac{R_0 ^2}{10^4}\right) -2 \sqrt{\frac{2R_0 ^2}{10^4}-\frac{R_0 ^4}{10^8}}
                \\  &\geq 1- \frac{3 R_0}{100}-\frac{3R_0}{100}\geq 1- \frac{ R_0}{10} > \frac{\sqrt{2}}{2}.
            \end{split}
        \end{align*}
    Because $x,y \in \deB_{t_1}$ are arbitrary, we have that $\inf_{x,y \in \deB_{t_1}} \langle x ,y \rangle \geq 1- \frac{ R_0}{10}> \frac{\sqrt{2}}{2}$. By \Cref{lem:fact of geodesically convex sets}, we have that $D_{t_1} ( \deB_{t_1} )= \inf_{x,y \in \Cx [\deB_{t_1}]} \langle x ,y \rangle \geq 1- \frac{ R_0}{10} $. Also, by \Cref{lem:first attractor}, we have that 
        \begin{align*}
            \mu_{t_1}\left(\deB_{t_1} \right) \geq \mu_{t_1}(B_{t_1}) = \mu_{T_{\ast}}(B_{\ast}) \geq \frac{1}{2} \left(1+\frac{9}{10 }R_0\right).
        \end{align*}
    Hence, to check the assumptions in \Cref{lem:Emergence of attractor}, we see that
        \begin{align*}
            \begin{split}
            &\Gamma\left(\deB_{t_1}\right) = \mu_{t_1}\left(\deB_{t_1} \right) \left(1+D_{t_1} ( \deB_{t_1} ) \right)-1
            \\  &\geq \frac{1}{2} \left(1+\frac{9}{10 }R_0\right) \left(2-\frac{R_0}{10 }\right) - 1
            \\  &= \frac{18}{20 }R_0 - \frac{1}{20 }R_0 - \frac{9}{200 } R_0 ^2 \geq \frac{4}{5}R_0.
            \end{split}
        \end{align*}
    Also, by the definition of $\deB_{t_1}$, we see that $D_{t_1} ( \deB_{t_1} )= \inf_{x,y \in \Cx [\deB_{t_1}]} \langle x ,y \rangle \leq 1- \frac{ R_0 ^2}{10 ^4} $. We get that
        \begin{align*}
            \frac{1}{4} \left(1- D_{t_1} ( \deB_{t_1} )\right) \Gamma\left(\deB_{t_1}\right) ^2 \geq \frac{1}{4} \cdot  \frac{ R_0 ^2}{10 ^4} \cdot \frac{16 R_0^2}{25} \geq \epa^2.
        \end{align*}
    We then finish the proof by \Cref{lem:Emergence of attractor}.
\end{proof}

\begin{proof}[Proof of \Cref{thm:main thm improved constants}]
    Similar to the proof for \Cref{thm:exponential small outside positive cap}, we fix the $\lambda = 1-10^{-10} R_0 ^4 \geq 1-10^{-10} R_0 ^2$, $\alpha = \frac{\pi}{100}$, and divide $\R_{\geq 0}$ into pieces $0=s_{-1} \leq t_0 < s_0 \leq t_1 < s_1 \leq t_2 < s_2 \cdots $, where for any $k \geq 0$
    \begin{align*}
        t_k \coloneq \inf \left\{ t \geq s_{k-1} \ \bigg| \  \partial_t R_t \geq \frac{1}{8} (\sin ^4 {\alpha})  \lambda^3 R_0 ^3 \right\} , \quad s_k \coloneq t_k +1 .
    \end{align*} 
    As in the proof for \Cref{thm:exponential small outside positive cap}, we showed that this construction must stop at some $k_{\ast}$-th step and $k_{\ast} \leq 10^{9} R_0 ^{-6}$. After the time $s_{k_{\ast}}$, we already saw in the proof of \Cref{thm:exponential small outside positive cap} that $f_t \left( \S \backslash S_{\alpha} ^+ (t) \right)$ starts to decay exponentially fast. To go further, we estimate how large $s_{k_{\ast}}$ can be without using \Cref{thm:length of small dR_t intervals} directly.

    We first prove that there is an upper bound $C_{\ast}$ depending on $\|f_0\|_{L^2(\S)}$ and $R_0$, such that for any $k \in [-1,k_{\ast}]$, $f_{s_k} ^2 \left( S_{\alpha} ^- (s_k) \right) \leq C_{\ast}$. We fix a $k \in [-1,k_{\ast}]$.  First, if $s_k \leq T_{\ast}$ for the $T_{\ast}$ obtained in \Cref{lem:first attractor}, by \Cref{lem:sliding norm}, we have that
        \begin{align*}
            f_{s_k} ^2 \left( S_{\alpha} ^- (s_k) \right) 
            \leq f_{s_k} ^2 \left( \S  \right) 
            \leq \|f_0\|_{L^2(\S)}^2 \cdot e^{2(d-1)s_k} 
            \leq \|f_0\|_{L^2(\S)}^2 \cdot e^{2(d-1)T_{\ast}}.
        \end{align*}
    Now, if $s_k >T_{\ast}$, take the $l$ such that the following two conditions are satisfied:
        \begin{itemize}
            \item [(1)] For any $p \in [l+1,k]$ ($\emptyset$ if $l=k$), $t_p -s_{p-1} \leq \delta + \widetilde \delta$ for $\delta$ defined in \eqref{eqn:small time} in \Cref{lemma:shrink from negative cap to positive cap}, and $\widetilde \delta$ defined in \eqref{eqn:small time 2} in \Cref{lemma:shrink from negative cap to the attractor}.
        \item[(2)] $t_l - (\delta + \widetilde \delta) >s_{l-1} > T_{\ast} $ or $s_{l-1} \leq T_{\ast} \leq s_l$.
        \end{itemize}
If the first case in the condition (2) above holds true, we apply the \eqref{eqn:attractor nbhd in positive cap} in \Cref{lem:moving attractor}, which implies that $\phi_{t_l\to s_k}\left(\deB_{t_l} \right) \subseteq S_{\frac{\pi}{2}} ^+ (s_k)$. Hence, because in this case, $s_k - t_l \leq (k-l)(\delta+\widetilde \delta +1)+1$, by \Cref{lem:sliding norm}, we get that
    \begin{align*}
        f_{s_k} ^2 \left( S_{\alpha} ^- (s_k) \right) \leq  f_{s_k} ^2 \left(\S \backslash \phi_{t_l\to s_k}\left(\deB_{t_l} \right) \right) \leq e^{2(d-1)[(k-l)(\delta+\widetilde \delta +1)+1]} f_{t_l} ^2 \left( \S \backslash \deB_{t_l} \right). 
    \end{align*}
    Now, combine this inequality with  \Cref{lem:exponentially decay out attractor} for the interval $[s_{l-1},t_l]$, we get that
        \begin{align*}
            f_{s_k} ^2 \left( S_{\alpha} ^- (s_k) \right) \leq e^{3(d-1)[(k-l+1)(\delta+\widetilde \delta +1)]} f_{s_{l-1}} ^2 \left( S_{\alpha} ^- (s_{l-1}) \right).
        \end{align*}
    Using this inequality, we can iteratively pull $s_k$ back to the time when the second case in the condition (2) happens, and this iteration does not exceed $k_{\ast}$-times. In this case, we have that $s_{k-1} \leq T_{\ast} \leq s_k$. If $T_{\ast} < s_k - \delta-\widetilde \delta -1$, we have that
    \begin{align*}
        \begin{split}
            &f_{s_k} ^2 \left( S_{\alpha} ^- (s_k) \right) 
            \leq  f_{s_k} ^2 \left(\S \backslash \phi_{t_k\to s_k}\left(\deB_{t_k} \right) \right) 
            \leq e^{2(d-1)} f_{t_k} ^2 \left( \S \backslash \deB_{t_k} \right) 
            \\  
            &\leq e^{3(d-1)(\delta+\widetilde \delta+1)} f_{T_{\ast}} ^2 \left( S_{\alpha} ^- (T_{\ast}) \right) 
            \leq e^{3(d-1)(\delta+\widetilde \delta+1+T_{\ast})} \|f_0\|_{L^2(\S)}^2
            ,
        \end{split}
    \end{align*}
where the last inequality follows from \Cref{lem:sliding norm} on $[0,T_{\ast}]$.
If $T_{\ast} \geq s_k - \delta-\widetilde \delta -1$, then $s_k - T_{\ast} \leq \delta+\widetilde \delta+1$. Using \Cref{lem:sliding norm} again, we have that
    \begin{align*}
        \begin{split}
            &f_{s_k} ^2 \left( S_{\alpha} ^- (s_k) \right) \leq f_{s_k} ^2 \left( \S \backslash B_{s_k} \right) \leq e^{2(d-1)(s_k-T_{\ast})} f_{T_{\ast}} ^2 \left( \S \backslash B_{\ast} \right)
            \\  &\leq  e^{2(d-1)(\delta+\widetilde \delta+1+T_{\ast})} \|f_0\|_{L^2(\S)}^2
            .
        \end{split}
    \end{align*}
Combine the above arguments in all possibilities, we obtain that for any $k \in [-1,k_{\ast}]$, we have that
    \begin{align*}
        f_{s_k} ^2 \left( S_{\alpha} ^- (s_k) \right) \leq e^{4(d-1)[k_{\ast}(\delta+\widetilde \delta)+T_{\ast}]} \|f_0\|_{L^2(\S)}^2 
        .
    \end{align*}

    Next, we need to estimate each $t_k - s_{k-1}$ for $k \in [0,k_{\ast}]$. Assume that $t_k - s_{k-1}> \delta$ for $\delta$ defined in \eqref{eqn:small time} in \Cref{lemma:shrink from negative cap to positive cap}. Because for $t \in [s_{k-1} ,t_k]$, by definition of $t_k$, we have that $\partial_t R_t \leq \frac{1}{8} (\sin ^4 {\alpha})  \lambda^3 R_0 ^3$, we can then apply \Cref{lemma:shrink from negative cap to positive cap} to obtain that for any $r \in [s_{k-1}+\delta,t_k]$, $\S \backslash S_{\alpha} ^+ (r) \subseteq \phi_{r - \delta\to r} \left( S_{\alpha} ^- (r - \delta) \right)$. By the same reason we obtained \eqref{eqn:exponential decay on long interval 2} in \Cref{thm:length of small dR_t intervals}, we can obtain that
        \begin{align*}
            \mu_{r}\left(\S \backslash S_{\alpha} ^+ (r)  \right) \leq 10 \cdot e^{-\frac{(d-1) \lambda R_0\cos{\alpha}}{4} (r-\delta-s_{k-1})} \left[ f_{s_{k-1}} ^2 \left( S_{\alpha} ^- (s_{k-1}) \right) \right]^{\frac{1}{2}}.
        \end{align*}
    Applying the upper bound for $f_{s_{k-1}} ^2 \left( S_{\alpha} ^- (s_{k-1}) \right) $ we got earlier, we see that
        \begin{align}\label{eqn:exponential decay with uniform parameter}
            \mu_{r}\left(\S \backslash S_{\alpha} ^+ (r)  \right) \leq  e^{-\frac{(d-1) \lambda R_0\cos{\alpha}}{4} (r-s_{k-1})} \cdot e^{3(d-1)[k_{\ast}(\delta+\widetilde \delta)+T_{\ast}]} \|f_0\|_{L^2(\S)}
            ,
        \end{align}
    for any $r \in [s_{k-1}+\delta,t_k]$. To simplify the notation in the proof, we let $\eta \coloneq \frac{(d-1) \lambda R_0\cos{\alpha}}{4}$ and $A \coloneq e^{3(d-1)[k_{\ast}(\delta+\widetilde \delta)+T_{\ast}]} \left[\|f_{0}\|_{L^2(\left(\S\right))}\right]^{\frac{1}{2}}$. By \Cref{thm:almost exponential decay I_t}, we have that $I_t+\partial_t I_t \leq 10^2 \mu_t \left( \S \backslash S_{\alpha} ^+ (t) \right) \leq 10^2 A e^{-\eta (t-s_{k-1})}$. Multiply $e^t$ on both sides and integrate from $s_{k-1}$ to $r$, we obtain that
    \begin{align*}
        I_r \leq I_{s_{k-1}}e^{-r+s_{k-1}}+ 10^3 A e^{-\xi (t-s_{k-1})},
    \end{align*}
    where $\xi = \eta$ if $\eta <1$ and $\xi=\frac{1}{2}$ if $\eta \geq 1$. Also, by the fact that $I_t \leq 2$ using its definition directly, we can simplify the above inequality and obtain that for any $r \in [s_{k-1}+\delta,t_k]$, 
        \begin{align*}
            I_r \leq  10^4 A e^{-\xi (r-s_{k-1})}.
        \end{align*}
      By \Cref{lemma: all time lower bound R_t}, we have that $R_t \geq \lambda R_0$. Using \Cref{lemma:derivative M_t R_t}, we have that for any $r \in [s_{k-1}+\delta,t_k]$,
    \begin{align*}
        \partial_t R_t \big|_{t=r}\leq \frac{3I_r}{2\lambda R_0}+ \frac{\epa^2}{2\lambda R_0} \leq \frac{3I_r}{2\lambda R_0}+ 10^{-8} \lambda^3 R_0 ^3 \sin^2 {\alpha}.
    \end{align*}
    In particular, we can pick $r= t_k$, and by the construction of $t_k$, we must have that 
        \begin{align*}
            \frac{1}{8} (\sin ^4 {\alpha})  \lambda^3 R_0 ^3 \leq \partial_t R_t \big|_{t=t_k}.
        \end{align*}
    Recall that we already fixed $\alpha = \frac{\pi}{100}$ in the assumption, and $\lambda$ is very close to $1$ by our choice at the beginning. So, combine the two inequalities for $\partial_t R_t \big|_{t=t_k}$ and $I_{t_k}$, we have that 
        \begin{align*}
            \frac{\sin ^4 \alpha}{24} \lambda^4 R_0 ^4 \leq I_{t_k} \leq 10^4 A e^{-\xi (t_k-s_{k-1})}.
        \end{align*}
    Hence,
        \begin{align}\label{eqn: length of each small der interval}
            \xi(t_k - s_{k-1}) \leq 3(d-1)[k_{\ast}(\delta+\widetilde \delta)+T_{\ast}] + \log{ \left[10^{17} R_0 ^{-4} \|f_{0}\|_{L^2(\left(\S\right))}
            \right]}.
        \end{align}

    Using \eqref{eqn: length of each small der interval}, we sum both sides from $k=1$ to $k=k_{\ast}$, and obtain that 
        \begin{align*}
            \begin{split}
                &\xi s_{\ast} \leq 1+3(d-1)k_{\ast}[k_{\ast}(\delta+\widetilde \delta)+T_{\ast}] + k_{\ast}\log{ \left[10^{17} R_0 ^{-4} \|f_{0}\|_{L^2(\left(\S\right))}
                \right]}.
                \\  &\leq 10^{23} (d-1)R_0 ^{-14} + 10^{9} R_0 ^{-6} \log {\left[ \|f_{0}\|_{L^2(\left(\S\right))}  \right]},
            \end{split}
        \end{align*}
    where we also used the fact that $k_{\ast} \leq 10^9 R_0^{-6}$, $T_{\ast} \leq 10^4 R_0 ^{-3}$, $\delta \leq 10^4 R_0 ^{-1}$, and  $\widetilde\delta \leq 10^2 R_0 ^{-2}$. Hence, apply \eqref{eqn:exponential decay with uniform parameter} for $r \in [s_{k_{\ast}}+\delta , +\infty]$ we see that 
        \begin{align*}
            \mu_{r}\left(\S \backslash S_{\alpha} ^+ (r)  \right) \leq  e^{-\frac{(d-1) \lambda R_0\cos{\alpha}}{4} (r-s_{k_{\ast}})} \cdot e^{3(d-1)[k_{\ast}(\delta+\widetilde \delta)+T_{\ast}]} \|f_{0}\|_{L^2(\left(\S\right))} 
            .
        \end{align*}
    Hence, if we set 
        \begin{align*}
            S_0 \coloneq \xi^{-1}\left[10^{24} (d-1)R_0 ^{-14} + 10^{9} R_0 ^{-6} \log {\left( \|f_{0}\|_{L^2(\S)}  \right)} \right], 
        \end{align*}
    we have that when $r \geq S_0$,
        \begin{align*}
            \mu_{r}\left(\S \backslash S_{\alpha} ^+ (r)  \right) \leq  e^{-\frac{(d-1)R_0}{8} (r-S_0)} \cdot \|f_0\|_{L^2(\S)} 
            .
        \end{align*}

 We now eliminate the dependence on the initial radius $ R_0 $ in the exponent by further evolving the flow. Specifically, we define 
    \begin{align*}
    S_1 = 
    \begin{cases}
    S_0, & \text{if } \mu_{S_0}\left(\mathbb{S} \setminus S_{\alpha}^+(S_0) \right) \leq 0.1, \\
    S_0 + \dfrac{8}{R_0} \log\left(10\|f_0\|_{L^2(\mathbb{S})}\right), & \text{otherwise}.
    \end{cases}       
    \end{align*}
    Then, as established in \eqref{eqn:large R0 time}, we have $ R_{S_1} \geq \frac{1}{2} $. Restart the flow at time $ S_1 $, and we define
    \begin{align*}
            S_2 \coloneq \xi_1^{-1}\left[10^{24}2^{14} (d-1) + 10^{9} 2^{6} \log {\left( \|f_{S_1}\|_{L^2(\S)}  \right)} \right], 
        \end{align*}
    with $\xi_1^{-1} = \frac{8}{(d-1)\lambda_1 \cos\alpha} \vee 1$ and $\lambda_1 = 1-10^{-10} 2^{-4}$. For all $ r \geq S_1 + S_2 $, we then have the estimate
    \begin{align*}
            &\mu_{r}\left(\S \backslash S_{\alpha} ^+ (r)  \right) \leq  e^{-\frac{(d-1)}{16} (r-S_1 - S_2)} \cdot \|f_{S_1+S_2}\|_{L^2(\S)} 
            \\
            &\leq e^{-\frac{(d-1)}{16} (r-33(S_1 + S_2))} \cdot \|f_{0}\|_{L^2(\S)} 
            \,,
        \end{align*}
    where we used the bound $\|f_{S_1+S_2}\|_{L^2(\S)} \leq e^{2(d-1)(S_1+S_2)}\|f_{0}\|_{L^2(\S)} $ from \Cref{lem:sliding norm}.
    To estimate $ S_1 + S_2 $, note from \Cref{lem:sliding norm} that $\log {\left( \|f_{S_1}\|_{L^2(\S)}  \right)} \leq 2(d-1)S_1 + \log {\left( \|f_0\|_{L^2(\S)}  \right)} $, which yields the upper bound
    \begin{align*}
        S_2 \leq 10^{31} (d-1) + 10^{14} (d-1)S_1 + 10^{13}\log {\left( \|f_0\|_{L^2(\S)}  \right)}
    \end{align*}
    Combining with the definition of $ S_1 $ and noting that $R_0\leq1$, we arrive at
    \begin{align*}
        S_1+S_2 \leq \left[ \frac{16}{(d-1)R_0} \vee 1\right](d-1)[10^{39}(d-1)R_0^{-14} + 10^{24}R_0^{-6}\log{\left( \|f_0\|_{L^2(\S)}  \right)}] \,.
    \end{align*}
    Finally, if we set
    \begin{align*}
    T_0 \coloneq \left[ \frac{16}{(d-1)R_0} \vee 1\right](d-1)[10^{41}(d-1)R_0^{-14} + 10^{26}R_0^{-6}\log{\left( \|f_0\|_{L^2(\S)}  \right)}] \,,
    \end{align*}
    we have that when $r\geq T_0$
    \begin{align*}
            &\mu_{r}\left(\S \backslash S_{\alpha} ^+ (r)  \right) \leq   e^{-\frac{(d-1)}{16} (r-T_0)} \cdot \|f_{0}\|_{L^2(\S)} 
            \,,
        \end{align*}
    and the result stated in \Cref{thm:main thm improved constants} then follows.
\end{proof}

\vspace{0.5in}
\noindent {\bf Acknowledgments} 
Y.P. is supported
in part by NSF under CCF-2131115. P.R. is supported by NSF grants DMS-2022448 and CCF2106377. The authors thank Tim Roith for carefully reading our preprint and pointing out a helpful example in Remark 3.5 of \cite{burger2025analysis}, which complements our \Cref{thm:global_max}.


\bibliographystyle{alpha}
\bibliography{references}

\end{document}